\documentclass{article} 
\usepackage{iclr2026_conference,times}

\usepackage{amsmath,amsfonts,bm}

\def\eqref#1{equation~\ref{#1}}

\def\1{\bm{1}}

\DeclareMathAlphabet{\mathsfit}{\encodingdefault}{\sfdefault}{m}{sl}
\SetMathAlphabet{\mathsfit}{bold}{\encodingdefault}{\sfdefault}{bx}{n}

\usepackage{hyperref}
\usepackage{amsmath}
\usepackage{amssymb}
\usepackage{booktabs}
\usepackage{tabularx}
\usepackage{url}
\usepackage{amsthm}
\usepackage{graphicx}
\usepackage{multicol}
\usepackage{multirow}
\usepackage{thmtools,thm-restate}
\newtheorem{theorem}{Theorem}[section]

\newtheorem{proposition}[theorem]{Proposition}

\usepackage{bm}
\usepackage{caption}
\usepackage{subcaption}
\usepackage{siunitx}
\usepackage{makecell}
\usepackage[table]{xcolor}
\usepackage[percent]{overpic}

\definecolor{color_1}{RGB}{255,0,128}
\definecolor{color_2}{RGB}{0,128,128}
\definecolor{color_3}{RGB}{0,128,0}
\definecolor{color_4}{RGB}{128,0,0}
\definecolor{color_5}{RGB}{128,0,128}

\newcommand{\ours}{B\'ezierFlow}

\title{B\'ezierFlow: Learning B\'ezier Stochastic Interpolant Schedulers for Few-Step Generation}

\author{Yunhong Min\textsuperscript{$\ast$} \hspace{0.5em} Juil Koo\textsuperscript{$\ast$} \hspace{0.5em} Seungwoo Yoo \hspace{0.5em} Minhyuk Sung \\
KAIST\\
\texttt{\{dbsghd363,63days,dreamy1534,mhsung\}@kaist.ac.kr} \\
}

\iclrfinalcopy 

\begin{document}

\vspace*{-0.3cm}
\maketitle
\begingroup
\renewcommand\thefootnote{}\footnotetext{\textsuperscript{$\ast$}Equal contribution.}
\endgroup



%


\begin{abstract}
\vspace{-0.5\baselineskip}
We introduce B\'ezierFlow, a lightweight training approach for few-step generation with pretrained diffusion and flow models. B\'ezierFlow achieves a 2–3× performance improvement for sampling with $\leq$ 10 NFEs while requiring only 15 minutes of training. Recent lightweight training approaches have shown promise by learning optimal timesteps, but their scope remains restricted to ODE discretizations. To broaden this scope, we propose learning the optimal transformation of the sampling trajectory by parameterizing stochastic interpolant (SI) schedulers. The main challenge lies in designing a parameterization that satisfies critical desiderata, including boundary conditions, differentiability, and monotonicity of the SNR. To effectively meet these requirements, we represent scheduler functions as Bézier functions, where control points naturally enforce these properties. This reduces the problem to learning an ordered set of points in the time range, while the interpretation of the points changes from ODE timesteps to Bézier control points. Across a range of pretrained diffusion and flow models, B\'ezierFlow consistently outperforms prior timestep-learning methods, demonstrating the effectiveness of expanding the search space from discrete timesteps to Bézier-based trajectory transformations. Project Page: \href{https://bezierflow.github.io}{https://bezierflow.github.io}.
\end{abstract}

\section{Introduction}
\vspace{-0.5\baselineskip}
Diffusion and flow models have achieved state-of-the-art performance, but at the cost of high computation due to their iterative generative processes. A large body of recent work~\citep{Lu:2022DPMSolver, Song:2023CM, Liu:2023RF, Tong:2025LD3} has aimed to accelerate generation to just a few steps. For diffusion models, many dedicated solvers~\citep{Lu:2022DPMSolver, Lu:2023DPMSolver++, Zhang:2023iPNDM, zhao:2023UniPC} tailored to their ODE formulations have been proposed. Although these methods substantially reduce the number of iterations from hundreds to tens, they are often insufficient to bring the steps down to only a few. More recent distillation techniques, such as consistency models~\citep{Song:2023CM} and variants~\citep{Kim:2023CTM, Berthelot:2023TRACT, Zhou:2025IMM} for diffusion and ReFlow~\citep{Liu:2023RF} for flow models, can reduce the number of steps to as few as one, but they require substantial fine-tuning, often hundreds to thousands of GPU hours, even for small datasets.

A notable line of recent work is the lightweight training approach, which learns only a few parameters with a pretrained model to improve output quality for a small number of function evaluations (NFEs). Compared to the distillation techniques, such approaches require only tens of GPU minutes for training while achieving considerable improvements. The key questions for lightweight training are: (1) what to optimize, and (2) how to parameterize the variables. For the former, most previous work~\citep{Tong:2025LD3, Chen:2024GITS, Xue:2024DMN} has focused on learning the optimal sequence of timesteps for ODE solves, treating a nondecreasing sequence of timesteps as learnable variables. Most notably, a recent teacher-forcing approach~\citep{Tong:2025LD3} that uses the outputs of a multistep adaptive solver as the teacher has demonstrated the effectiveness of learned ODE timesteps. 

Broadening our scope, we explore variables beyond ODE timesteps that can be learned with lightweight training. As our first key contribution, we propose optimizing the sampling trajectories themselves. 
The Stochastic Interpolant (SI) framework~\citep{Albergo:2023SI} provides a unified view of modern ODE-based generative models. In this framework, the state at any time is written as a linear interpolation between two endpoint samples: one drawn from the source (e.g., latent) distribution and the other from the target data distribution. The interpolation is governed by a pair of time-dependent coefficient functions, referred to as the SI scheduler. The scheduler fully specifies the geometry of the sampling trajectory. Different models adopt different schedulers, yet interchanging one scheduler for another at inference time does not change the endpoint marginal distributions. Inspired by this, we propose a lightweight training framework for learning an SI scheduler, which is equivalent to sampling path transformations that preserve the endpoints.

Our second key contribution lies in the parameterization of SI schedulers. The 1D continuous functions for SI schedulers must satisfy the following properties: (i) \emph{boundary conditions}, ensuring that the endpoints of the coefficients are fixed, (ii) \emph{monotonicity}, which guarantees a strictly nondecreasing signal-to-noise ratio (SNR) along the sampling path, and (iii) \emph{differentiability}, which ensures that a velocity field can be derived in the ODEs governed by the learned scheduler. To effectively parameterize the space of such functions, while restricting the scope to polynomials, we propose a \emph{B\'ezier}-based parameterization, termed the {B\'ezier SI Scheduler}, which forms the core of our overall lightweight training framework, {B\'ezierFlow}. A 1D B\'ezier function naturally satisfies all of these properties: the boundary conditions can be enforced by simply setting the two end control points to the time range boundaries; the function is smooth and differentiable by the definition of polynomial B\'ezier curves; and monotonicity can be achieved by enforcing a nondecreasing order of control points, making the learning process identical to learning the ODE timesteps in previous work~\citep{Tong:2025LD3}. 


In our experiments, we evaluate both diffusion and flow models across diverse datasets and ODE solvers. B\'ezierFlow consistently outperforms existing acceleration techniques while requiring only lightweight training, taking around 15 minutes on a single GPU. Extensive results further demonstrate the effectiveness of optimizing the sampling trajectories directly, rather than ODE timesteps, when coupled with our continuous B\'ezier-based parameterization.

\section{Related Work}
\vspace{-0.5\baselineskip}

There have been various attempts to improve few-step generation in ODE-based generative models. One major line of work focuses on designing dedicated ODE solvers tailored to the dynamics of the models~\citep{Lu:2022DPMSolver,Lu:2023DPMSolver++,zhao:2023UniPC,Zhang:2023iPNDM}. Although these methods require no additional training, they are unable to achieve high-fidelity generation with only a few steps. Another line is distillation-based approaches~\citep{Song:2023CM,Salimans:2022PD,Liu:2023RF}, which have demonstrated impressive gains in the very low-NFE regime, but incur substantial computational cost in training. Despite these diverse strategies, our lightweight training approach is most closely related to the methods listed below, which we now discuss in more detail.

\vspace{-0.5\baselineskip}
\paragraph{Learning ODE Solving Timesteps.}
Several methods aim to achieve high-fidelity generation with few NFEs by optimizing the ODE timesteps in a lightweight manner. ~\citet{Chen:2024GITS} frame ODE timestep learning as a selection problem under a fixed NFE budget. Based on statistics collected from multiple sampling trajectories, they allocate more steps to regions of high curvature and fewer to flatter regions. ~\citet{Xue:2024DMN} optimize timesteps from the perspective of numerical integration: given a specific ODE solver, they minimize the accumulated local integration error along the trajectory. ~\citet{Tong:2025LD3} learn optimal timesteps through a data-driven distillation framework, where a high-NFE sampler serves as the teacher and a low-NFE sampler as the student. The timesteps are optimized by minimizing the discrepancy between their outputs starting from the same initial noise. Compared to these methods that learn optimal ODE timesteps, we learn optimal stochastic interpolant schedulers and demonstrate superior performance over these approaches.


\vspace{-0.5\baselineskip}
\paragraph{Learning Sampling Trajectories.}
Several works~\citep{Karras:2022EDM, Lipman:2024FMGuide, Pokle:2023training, Kim:2025RBF} have explored changing sampling trajectories at inference time to improve generation quality and diversity, while selecting from a few predefined stochastic interpolant schedulers (e.g., linear, VP, VE), rather than parameterizing the function space and finding the best scheduler through optimization. To our knowledge, Bespoke Solver~\citep{Shaul:2023Bespoke} is the only approach that learns an optimal sampling trajectory. Unlike our method, however, it relies on a discrete parameterization, which prevents direct derivation of first-order derivatives. These derivatives must instead be represented through auxiliary variables, introducing redundancy that can lead to inconsistencies between zeroth-order and first-order representations, and thus often fail to capture a truly differentiable function.

\section{Background: Stochastic Interpolant Framework}
\label{sec:background}
\vspace{-0.5\baselineskip}


Stochastic Interpolant (SI)~\citep{Albergo:2023SI} is a unified framework for generative modeling, encompassing both ODE-based and SDE-based models~\citep{Song:2021ScoreSDE, Ho:2020DDPM, Song:2021DDIM, Lipman:2023FM}. Given two marginal probability densities $p_0, p_1: \mathbb{R}^d \rightarrow \mathbb{R}_{\geq 0}$, a stochastic interpolant $x(t)$ is defined by the following stochastic process:
\begin{equation}
    \label{eq:si}
    x(t) = \alpha(t) x_1 + \sigma(t) x_0 + \gamma(t) z, \quad t\in [0,1],
\end{equation}
where $\alpha(t)$ and $\sigma(t)$ are interpolation coefficients between $x_0 \sim p_0$ and $x_1 \sim p_1$, and $z\sim \mathcal{N}(0,I)$ is a latent variable introducing stochasticitiy. The process satisfies the boundary conditions $x(0)=x_0$ and $x(1)=x_1$ by enforcing $\alpha(0)=\sigma(1)=0$, $\alpha(1)=\sigma(0)=1$, and $\gamma(0)=\gamma(1)=0$.

While the general formulation written in Eq.~\ref{eq:si} is broad, many important generative models---including diffusion~\citep{Song:2021DDIM, Ho:2020DDPM, Karras:2022EDM}, flow~\citep{Lipman:2023FM, Lipman:2024FMGuide}, and score-based~\citep{Song:2021ScoreSDE} models---can be expressed in a more specific form, referred to as one-sided interpolants:
\begin{align}
\label{eq:si_ode}
    x(t) = \alpha(t)x_1 + \sigma(t) x_0,
\end{align}
where, as a common practice, $p_0=\mathcal{N}(0,I)$ and $p_1 = p_{\text{data}}$, thereby the latent variable $z$ is absorbed into the initial state $x_0$. By differentiating both sides of Eq.~\ref{eq:si_ode}, it can be expressed in the following ODE form:
\begin{align}
\label{eq:velocity}
    \frac{dx(t)}{dt} = \dot{\alpha}(t)x_1 + \dot{\sigma}(t)x_0,
\end{align}
where we denote a time derivative by the dot. For these dynamics to be well defined, $\alpha(t)$ and $\sigma(t)$ must be twice continuously differentiable ($C^2$) to ensure that the divergence terms in the associated Fokker-Planck equation are well-defined.

Based on Eq.~\ref{eq:si_ode}, within the SI framework, different ODE-based generative models, including diffusion, flow, score-based models, learn different but interchangeable quantities. Along the sampling path $(x,t)$, flow models $v_\phi(x,t)$ approximate the velocity field $u_t(x)=\mathbb{E}[\dot{x}_t \mid x_t = x]$, while diffusion models $\epsilon_\phi(x,t)$ approximate the expected initial random noise state $\eta_t(x) = \mathbb{E} [x_0 \mid x_t = x]$. Finally, score-based models $s_\phi(x,t)$ estimate the score function, which is equivalent to the scaled version of the expected initial state: $\nabla \log p_t(x) = -\sigma^{-1}(t) \eta_t(x)$. Thus, different types of generative models are mathematically linked, and under the SI framework, a pretrained model of one type can be reinterpreted as another at inference. For convenience, we collectively refer to these ODE-based generative models as the SI model, denoted by  $S_\phi(x,t)$, throughout the paper.




\section{B\'ezierFlow}
\vspace{-0.5\baselineskip}



\subsection{Problem Definition}
\label{subsec:problem_definition}
The objective of our work is to learn an optimal \emph{sampling trajectory} that enables high-quality generation with a few NFEs (e.g., $\leq10$), while using a pretrained diffusion or flow model. 

We consider two sampling trajectories that share the same endpoints $x_0$ and $x_1$. The source path refers to the trajectory used during model training, while the target path is a newly optimized trajectory for inference. Although both trajectories share the same endpoints, we assume their intermediate geometry matters: due to discretization error in ODE solving, output quality would depend on the path geometry. Given this assumption, we therefore aim to optimize the target path such that, even with only a few of NFEs, its geometry produces sampling results comparable to those obtained along the source path with many steps.

Formally, given a pretrained SI model $S_\phi$, let $\xi(x_0, \{t_i\}_{i=1}^N; S_\phi)$ denote a multistep ODE solver along the source path (the \emph{teacher}), and $\bar{\xi}_\theta(x_0, \{s_i\}_{i=1}^M; S_\phi)$ a few-step ODE solver along the target path (the \emph{student}), where $\{t_i\}_{i=1}^N$ and $\{s_i\}_{i=1}^M$ are the respective timestep sets with $M \ll N$. Although both solvers start from the same initial state $x_0 \sim p_0$, they differ in the number of NFEs and the sampling path. Let $q(x_1)$ denote the distribution induced by the teacher, and $\bar{p}_\theta(x_1)$ the student’s output distribution. Our objective is formulated as the following teacher-forcing KL minimization:
\vspace{0.1\baselineskip}
\begin{equation}
\label{eq:kl}
\min_{\theta} D_{\text{KL}}\left(q(x_1) \Vert \bar{p}_\theta(x_1)\right).
\end{equation}

In practice, we optimize Eq.~\ref{eq:kl} using the following tractable surrogate objective~\citep{Tong:2025LD3}, which enforces the outputs of the two solvers to align with each other:
\begin{align}
\label{eq:teacher_forcing}
\min_\theta \mathcal{L}(\theta) = \mathbb{E}_{x_0 \sim p_0}\big[ d\big(\xi(x_0, \{t_i\}_{i=1}^N; S_\phi), \bar{\xi}_\theta(x_0, \{s_i\}_{i=1}^M; S_\phi)\big)\big],
\end{align}
where $d(\cdot,\cdot)$ is a distance metric such as LPIPS~\citep{Zhang:2018LPIPS}. This lightweight optimization adjusts only the target scheduler coefficients while using the pretrained model, thereby improving few-step generation at minimal training cost. See the left of Fig.~\ref{fig:bezierflow_compare}, which compares sampling trajectories from the prior distribution to the target distribution: (i) the teacher’s trajectories with many steps (NFE=50), (ii) the student’s initial trajectories, and (iii) trajectories optimized by \ours{}. While the student's initial trajectories deviate from the target distribution at NFE=3, after training with~\ours{}, they closely follow those of the teacher despite the much smaller NFE.


\begin{figure}[t!]
    \centering
    \includegraphics[width=\linewidth]{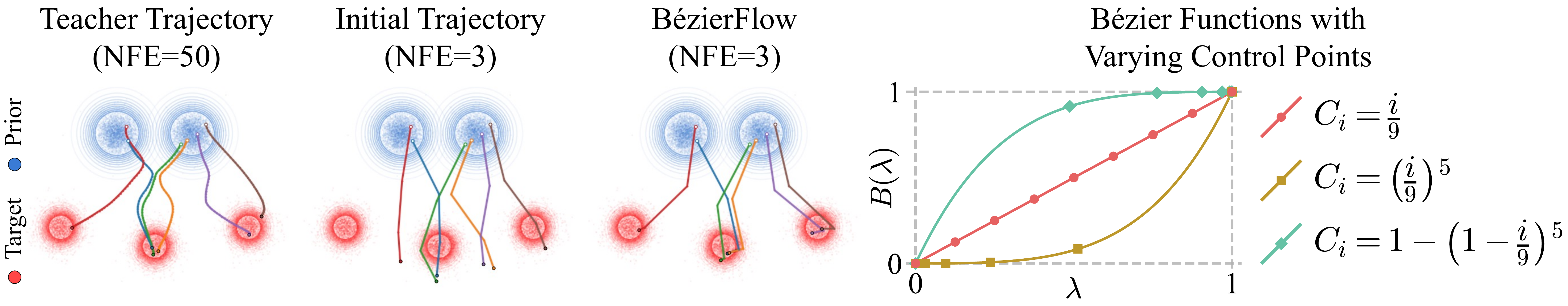}
    \label{fig:toy}
    \vspace{-\baselineskip}
    \caption{
    \textbf{Illustration of Sampling Trajectories and B\'ezier Functions.}
    On the left, we visualize different sampling trajectories. While initial trajectories deviate from the target distribution, B\'ezierFlow aligns them with those of the teacher using only NFE=3. On the right, we present examples of 8-degree B\'ezier functions with different arrangements of control points.
    }
    \vspace{-\baselineskip}
    \label{fig:bezierflow_compare}
\end{figure}

\subsection{Sampling Path Transformation}
\label{subsec:transformed-path}
\vspace{-0.5\baselineskip}
As discussed in Sec.~\ref{subsec:problem_definition}, in order to define two trajectories that share the same endpoints, we view this as a transformation of the source path via the path reparameterization following prior works~\citep{Karras:2022EDM, Shaul:2023Bespoke, Pokle:2023training, Kim:2025RBF}. 

In the SI framework, a sampling path is governed by a pair of interpolation coefficients in Eq.~\ref{eq:si_ode}, which we refer to as a \emph{scheduler}. We denote the coefficients of the source path as the source scheduler $(\alpha_t, \sigma_t)$, and those of the target path as the target scheduler $(\bar\alpha_s, \bar\sigma_s)$. Specifically, to relate the two schedulers, we adopt a scaling reparameterization trick~\citep{Karras:2022EDM}, where the target state $\bar{x}_s$ is defined from the source state $x_t$ as $\bar{x}_s = c_s x_{t_s}$. Here, $c_s$ can be arbitrary scalar functions and $t_s$ any invertible mapping from $s$, but we define them as
\begin{align}
\label{eq:scale_transform}
    c_s &= \begin{cases}
        \frac{\bar\sigma(s)}{\sigma(t_s)} = \frac{\bar{\alpha}(s)}{\alpha(t_s)}, &0 < s< 1, \\
        1,&\text{otherwise}, \\
    \end{cases} \\
    t_s &= t(s) = \rho^{-1}\big(\bar{\rho}(s)\big), \qquad \rho(t)=\frac{\alpha(t)}{\sigma(t)},\,\, \bar\rho(s)=\frac{\bar{\alpha}(s)}{\bar{\sigma}(s)},
\end{align}
where both $\rho$ and $\bar{\rho}$ are invertible as the signal-to-noise ratio (SNR) increases monotonically over time. 


Applying the change-of-variables to Eq.~\ref{eq:velocity}, the velocity in the target path $\bar{u}_s(\bar{x}_s)$ is expressed as
\begin{align}
\label{eq:ubar}
    \bar{u}_s(\bar{x}_s) = \frac{d\bar x_s}{ds}
    =\Big(\partial_s\log c_s\Big)\,\bar x_s
    + c_s\,\frac{dt_s}{ds}\;u_{t_s}\!\Big(\frac{\bar x_s}{c_s}\Big).
\end{align}

We note that replacing the source scheduler with the target scheduler based on Eq.~\ref{eq:scale_transform} at inference is valid for two reasons: (i) the endpoint marginal distributions are preserved, and (ii) the training objective of a SI model is invariant to the choice of schedule as long as the SNR endpoints (minimum and maximum values) are the same. Thus, learning $(\bar\alpha_s, \bar\sigma_s)$ only changes the geometry of the sampling path, and hence the few-step discretization behavior, without altering the underlying target distributions or requiring a different pretrained SI model. See App.~\ref{sec:validity_of_scheduler_reparameterization} for more details.

\subsection{B\'ezier Stochastic Interpolant Scheduler}
\label{subsec:opt-bezier}
\vspace{-0.5\baselineskip}
In Sec.~\ref{subsec:transformed-path}, we discussed $\emph{what}$ to optimize, namely the sampling path determined by the SI scheduler $(\bar\alpha_s, \bar\sigma_s)$. The next crucial question is $\emph{how}$ to parameterize these 1D continuous functions effectively. Since the space of arbitrary 1D functions is prohibitively large, we employ 1D B\'ezier parameterization, which offers strong expressiveness with a compact number of parameters---the control points. Moreover, B\'ezier functions naturally satisfy the three key requirements of the SI scheduler: (i) boundary conditions, as described in Sec.~\ref{sec:background}, (ii) monotonicity to ensure a strictly nondecreasing signal-to-noise ratio (SNR), and (iii) differentiability to compute the transformed velocity in Eq.~\ref{eq:ubar}. 


An $n$-degree B\'ezier curve is defined as a weighted linear combination of $n+1$ control points $\{C_i\}_{i=0}^n$, where the weights are given by Bernstein basis polynomials $b_{i,n}$:
\vspace{-0.1\baselineskip}
\begin{equation}
\label{eq:bezier}
    B(\lambda) = \sum_{i=0}^n b_{i,n}(\lambda) C_i, \quad b_{i,n}(\lambda)=\binom{n}{i}(1-\lambda)^{n-i}\lambda^i,\quad \lambda \in [0,1].
\end{equation}
\vspace{-0.1\baselineskip}
As shown in Eq.~\ref{eq:bezier}, with only $n$ control points, it can represent a wide range of trajectories. Unlike arbitrary 1D polynomial functions, B\'ezier functions always pass through their control points in order, making it straightforward to enforce boundary conditions and monotonicity. See the right of Fig.~\ref{fig:bezierflow_compare}, which illustrates how 1D B\'ezier functions $B(\lambda), \, \lambda \in [0,1]$ can represent diverse shapes under different control point arrangements while keeping the endpoints fixed.

Moreover, they are inherently smooth and infinitely differentiable $(C^\infty)$, with a closed-form derivative:
\begin{equation}
    \dot{B}(\lambda) = n \sum_{i=0}^{n-1}b_{i, n-1}(\lambda)(C_{i+1} - C_i),
\end{equation}
which allows us to directly compute the transformed velocity in Eq.~\ref{eq:ubar} at any time $s$. Specifically, we parameterize $\bar\alpha(s)$ and $\bar\sigma(s)$ as $n$-degree 1D B\'ezier functions, each defined by a set of control points:
\begin{equation}
    \bar\alpha^\theta(s) = (\alpha_1 - \alpha_0) \sum_{i=0}^n b_{i,n}(s)\,C_i^{(\alpha)} + \alpha_0,\qquad
    \bar\sigma^\theta(s)= (\sigma_1 - \sigma_0 )\sum_{i=0}^n b_{i,n}(s)\,C_i^{(\sigma)} + \sigma_0.
\end{equation}
For the boundary conditions, we fix the end control points ($C_0^{(\alpha)}=C_0^{(\sigma)}=0$, $C_n^{(\alpha)}=C_n^{(\sigma)}=1$) and treat only the $n-1$ interior control points as parameters. Concretely, with learnable parameters $\theta^{(\alpha)},\theta^{(\sigma)}\in\mathbb{R}^{n-1}$, the control points are given by
\begin{equation}
    C^{(\alpha)}=\big[\,0,\ \psi(\theta^{(\alpha)})_{1:n-1},\ 1\,\big],\quad
    C^{(\sigma)}=\big[\,0,\ \psi(\theta^{(\sigma)})_{1:n-1},\ 1\,\big],
\end{equation}
where $\phi(\theta)_i=\frac{e^{\theta_i}}{\sum_{j=1}^p e^{\theta_j}}$ is a softmax function, and
$\psi(\theta)_i=\sum_{j=1}^{i}\phi(\theta)_j$ is a cumulative softmax function that ensures monotonicity. This monotonic parameterization ensures that $\bar\rho(s)=\bar\alpha(s)/\bar\sigma(s)$ is strictly nondecreasing on $[0,1)$, resulting in $\bar\rho^{-1}$ exists. 




\subsection{Connection to Prior Work}
\label{subsec:connection_to_prior_Work}
\vspace{-0.3\baselineskip}
\paragraph{LD3~\citep{Tong:2025LD3}.}
From a parameterization perspective, both LD3 and ours optimize the same type of parameter: a nondecreasing sequence of timesteps. However, their interpretations differ: LD3’s parameters correspond directly to discrete ODE solver timesteps, whereas ours correspond to B\'ezier control points that form a continuous sampling path. Interpreting these parameters as an SI scheduler allows our approach to explore a much broader search space compared to LD3. See App.~\ref{sec:theoretical_analysis} for the proof.

\vspace{-0.3\baselineskip}
\paragraph{Bespoke Solver~\citep{Shaul:2023Bespoke}.}
Bespoke solver similarly optimizes a new, target sampling path, but the key difference from our approach lies in the parameterization. Their parameterization is discrete: they learn per-step variables $t_s$ and $c_s$ in Eq.~\ref{eq:scale_transform}. Such a discrete parameterization requires separately modeling the derivatives $\dot t_s, \dot c_s$, thereby breaking the intrinsic connection between the values and their derivatives. This can yield mismatches between the predicted next value from numerical integration and the actual learned value, ultimately causing unstable optimization. 

In contrast, our B\'ezier parameterization ensures that the resulting scheduler $(\bar\alpha^\theta_s, \bar\sigma^\theta_s)$ is smooth and, in particular, satisfies the $C^2$ condition required for the SI scheduler as discussed in Sec.~\ref{sec:background}. Consequently, the time-derivative terms in the transformed velocity can be computed directly rather than learned separately, and the learned ODE trajectories are well-defined, thereby leading to more stable optimization. 

Beyond the parameterization, Bespoke Solver also differs in its training objective: it relies on step-wise error minimization, whereas our method optimizes a global trajectory-level loss. See App.~\ref{sec:comparison_with_bespoke} for the results comparing with Bespoke Solver trained under our loss, which isolate and highlight the benefit of our B\'ezier-based continuous parameterization.

\section{Experiments}
\label{sec:experiment}
\vspace{-0.5\baselineskip}
\paragraph{Experiment Setup.}
We evaluate our \ours{} (BF) on both diffusion and flow models across diverse datasets for image generation. For diffusion models, we adopt EDM~\citep{Karras:2022EDM} with pretrained checkpoints on CIFAR-10 ($32\times32$)~\citep{Krizhevsky:2009CIFAR}, FFHQ ($64\times64$)~\citep{Karras:2019FFHQ}, and AFHQv2 ($64\times64$)~\citep{Choi:2020AFHQv2}. For flow models, we use pretrained ReFlow~\citep{Liu:2023RF} on CIFAR-10 ($32\times32$)~\citep{Krizhevsky:2009CIFAR}, FlowDCN~\citep{Wang:2024FlowDCN} on ImageNet ($256\times256$)~\citep{Deng2009ImageNet}, and Stable Diffusion v3.5 (SD)~\citep{Esser:2024SD3} on MS-COCO ($512\times512$)~\citep{Lin:2014MSCOCO}. All pretrained models are from their official implementations.

Each model is paired with its dedicated ODE solver: UniPC~\citep{zhao:2023UniPC} and iPNDM~\citep{Zhang:2023iPNDM} for diffusion models, and Runge–Kutta integrators, RK1 (Euler) and RK2 (Midpoint), for flow models. We consider the following learning-based acceleration methods as baselines for comparisons with ours: DMN~\citep{Xue:2024DMN}, GITS~\citep{Chen:2024GITS}, and LD3~\citep{Tong:2025LD3}. For flow models, we additionally include Bespoke Solver~\citep{Shaul:2023Bespoke} as a baseline, which was specifically designed for RK1 and RK2 solvers. The results of the base ODE solvers, without additional learning, are reported as reference. All baselines are evaluated using their official implementations, except for Bespoke Solver, whose official code is not publicly available. 



\vspace{-0.3\baselineskip}
\paragraph{Implementation Details.} 
Methods based on the teacher-forcing framework, including LD3, Bespoke Solver, and our~\ours{}, are trained and validated with the same number of samples for both training and validation: 200 each for CIFAR-10, 50 each for FFHQ, AFHQv2, and ImageNet, and 25 each for Stable Diffusion v3.5, following the experiment setup used in LD3~\citep{Tong:2025LD3}. Training is performed for 8 epochs on CIFAR-10, FFHQ, AFHQv2 and 5 epochs on the others. For GITS~\citep{Chen:2024GITS}, we precompute statistics using 256 sampling trajectories. For the training of~\ours{}, we use $32$ control points for the B\'ezier parameterization unless stated otherwise. For all models, we initialize the target scheduler as the linear SI scheduler, i.e., $\bar\alpha(s)=s$ and $\bar\sigma(s)=1-s$.
We set the timesteps uniformly in SNR $\rho(s)$ for diffusion models and uniformly in time $s$ for flow models.
Following LD3~\citep{Tong:2025LD3}, we also feed the learned decoupled timesteps~\citep{Li:2024TimeShift} to the neural network. See App.~\ref{sec:implementation_details} for more details.


\begin{table}[!t]
\centering
\scriptsize
\caption{\textbf{FID comparison of few-step generation with diffusion models}. Results of the base ODE solvers are reported on each top rows, highlighted in gray. \textbf{Bold} indicates the best results, and \underline{underline} marks the second best.} 
\begin{tabularx}{\textwidth}
  {>{\raggedright\arraybackslash}p{0.1\textwidth}
   |*{4}{>{\centering\arraybackslash}X}
   |>{\raggedright\arraybackslash}p{0.1\textwidth}
   |*{4}{>{\centering\arraybackslash}X}}
\toprule
 Method & NFE=4 & NFE=6 & NFE=8 & NFE=10 & Method & NFE=4 & NFE=6 & NFE=8 & NFE=10 \\
 \midrule
\multicolumn{10}{c}{CIFAR-10 $32\times 32$ with EDM~\citep{Karras:2022EDM} (Teacher FID: 2.08)} \\
 \midrule
 \cellcolor{gray!20}UniPC & 50.30 & 19.33 & 9.64 & 6.16 & \cellcolor{gray!20}iPNDM & 29.53 & 9.84 & 5.30 & 3.75 \\
 + DMN & 26.42 & 8.11  & 4.22 & 2.79 & + DMN & 28.29 & 9.33 & 4.82 & 3.52 \\
 + GITS & 24.83 & 11.02 & 6.68 & 5.02 & + GITS & 16.20 & 6.80 & 4.07 & 3.30 \\
 + LD3 & \underline{12.04} & \underline{3.56} & \underline{2.43} & \underline{2.62} & + LD3 & \underline{9.97} & \underline{4.42} & \underline{2.93} & \underline{2.44} \\
 + \ours{} & \textbf{9.55} & \textbf{3.13} & \textbf{2.40} & \textbf{2.09} & + \ours{} & \textbf{6.93} & \textbf{3.35} & \textbf{2.81} & \textbf{2.43} \\
 \midrule
 \multicolumn{10}{c}{FFHQ $64\times 64$ with EDM~\citep{Karras:2022EDM} (Teacher FID: 2.86)} \\
 \midrule
 \cellcolor{gray!20}UniPC & 47.62 & 14.96 & 7.76 & 8.93 & \cellcolor{gray!20}iPNDM & 28.75 & 11.15 & 6.68 & 4.80 \\
 + DMN & 25.87 & 9.44 & 5.06 & 4.06 & + DMN & 30.89 & 11.93 & 7.33 & 6.20 \\
 + GITS & 22.99 & 12.12 & 8.90 & 4.40 & + GITS & 18.51 & 9.21 & 5.58 & 4.37 \\
 + LD3 & \underline{22.48} & \textbf{6.16} & \underline{4.25} & \textbf{2.92} & + LD3 & \underline{15.55} & \textbf{5.89} & \textbf{3.74} & \textbf{3.03} \\
 + \ours{} & \textbf{17.05} & \underline{7.43} & \textbf{3.82} & \underline{3.13} & + \ours{} & \textbf{15.39} & \underline{7.84} & \underline{5.56} & \underline{3.75} \\
 \midrule
 \multicolumn{10}{c}{AFHQv2 $64\times 64$ with EDM~\citep{Karras:2022EDM} (Teacher FID: 2.04)} \\
 \midrule
 \cellcolor{gray!20}UniPC & 23.59 & 10.15 & 7.76 & 6.38 & \cellcolor{gray!20}iPNDM & 15.14 & 6.12 & 3.80 &  3.01 \\
 + DMN & 30.39 & 14.40 & 3.98 & 3.69 & + DMN & 33.21 & 15.95 & 5.99 & 5.29 \\
 + GITS & \underline{13.20} & 7.50 & 3.89 & 3.94 & + GITS & 14.31 & 5.81 & 3.88 & 3.57 \\
 + LD3 & 18.17 & \underline{4.95} & \textbf{2.68} & \underline{3.02} & + LD3 & \textbf{11.85} & \textbf{3.11} & \textbf{2.45} & \underline{2.18} \\
 + \ours{} & \textbf{12.27} & \textbf{4.46} & \underline{2.75} & \textbf{2.67} & + \ours{} & \underline{14.44} & \underline{4.69} & \underline{2.63} & \textbf{2.16} \\
 \bottomrule
\end{tabularx}
\label{tbl:main_fid_diffusion}
\end{table}
\begin{table}[!t]
\centering
\scriptsize
\caption{\textbf{FID comparison of few-step generation with flow-based models}. Results of the base ODE solvers are reported on each top rows, highlighted in gray. \textbf{Bold} indicates the best results, and \underline{underline} marks the second best.} 
\setlength{\tabcolsep}{5pt}
\begin{tabularx}{\textwidth}
  {>{\raggedright\arraybackslash}p{0.1\textwidth}
   |*{4}{>{\centering\arraybackslash}X}
   |>{\raggedright\arraybackslash}p{0.1\textwidth}
   |*{4}{>{\centering\arraybackslash}X}}
\toprule
 Method & NFE=4 & NFE=6 & NFE=8 & NFE=10 & Method & NFE=4 & NFE=6 & NFE=8 & NFE=10 \\
 \midrule
\multicolumn{10}{c}{CIFAR-10 $32\times 32$ with ReFlow~\citep{Liu:2023RF} (Teacher FID: 2.70)} \\
 \midrule
 \cellcolor{gray!20}RK1 & 52.78 & 26.30 & 17.40 & 13.30 & \cellcolor{gray!20}RK2 & 25.36 & 12.12 & 9.17 & 7.89 \\
 + DMN & 180.03 & 104.23 & 30.94 & 21.58 & + DMN & 82.41 & 51.99 & 21.43 & 18.62 \\
 + Bespoke & 45.31 & 18.08 & 11.88 & 9.25 & + Bespoke & 39.45 & 64.87 & 16.67 & 13.34 \\
 + GITS & 47.42 & 26.11 & 19.89 & 15.34 & + GITS & \underline{22.84} & \underline{11.84} & 8.77 & 6.58 \\
 + LD3 & \underline{38.95} & \underline{20.10} & \underline{12.54} & \underline{9.64} & + LD3 & 29.45 & 13.82 & \underline{6.26} & \underline{3.86} \\
 + \ours{} & \textbf{20.64} & \textbf{9.67} & \textbf{7.30} & \textbf{5.51} & + \ours{} & \textbf{13.18} & \textbf{6.00} & \textbf{4.31} & \textbf{3.74} \\
 \midrule
 \multicolumn{10}{c}{ImageNet $256\times 256$ with FlowDCN~\citep{Wang:2024FlowDCN} (Teacher FID: 15.89)} \\
 \midrule
 \cellcolor{gray!20}RK1 & 12.03 & 12.04 & 13.55 & 14.43 & \cellcolor{gray!20}RK2 & 7.91 & 10.54 & 12.97 & 14.08 \\
 + DMN & 142.79 & 28.56 & \underline{10.61} & \underline{11.69} & + DMN & 7.96 & 10.23 & \underline{9.42} & \underline{7.86} \\
 + Bespoke & \underline{11.85} & 11.81 & 13.39 & 14.31 & + Bespoke & \underline{7.66} & 10.05 & 13.02 & 14.23 \\
+ GITS & 13.20 & \underline{10.91} & 11.91 & 12.93 & + GITS & 8.18 & \underline{9.80} & 12.30 & 13.27 \\
 + LD3 & \textbf{11.62} & 11.94 & 13.36 & 14.12 & + LD3 & \textbf{7.59} & 10.17 & 12.75 & 14.04 \\
 + \ours{} & 15.60 & \textbf{6.85} & \textbf{7.77} & \textbf{8.11} & + \ours{} & 9.50 & \textbf{5.94} & \textbf{6.22} & \textbf{7.56} \\
 \midrule
 \multicolumn{10}{c}{MS-COCO $512\times 512$ with Stable Diffusion~\citep{Esser:2024SD3} (Teacher FID: 12.13)} \\
 \midrule
 \cellcolor{gray!20}RK1 & 57.93 & \textbf{30.96} & 21.50 & \underline{17.19} & \cellcolor{gray!20}RK2 & 34.95 & 17.89 & 13.33 & 11.61 \\
 + DMN & 113.24 & 46.02 & 31.58 & 24.41 & + DMN & 36.33 & \underline{16.45} & 27.09 & 17.36 \\
 + Bespoke & 134.21 & 52.51 & 23.70 & 20.69 & + Bespoke & 45.23 & 40.87 & 20.18 & 13.26 \\
 + GITS & 70.01 & 42.44 & 31.89 & 25.47 & + GITS & \textbf{31.09} & 21.21 & 15.58 & 14.65 \\
 + LD3 & \underline{55.31} & 36.85 & \underline{20.37} & 19.76 & + LD3 & 39.03 & 18.04 & \underline{12.30} & \underline{11.54} \\
 + \ours{} & \textbf{54.05} & \underline{33.43} & \textbf{19.69} & \textbf{16.52} & + \ours{} & \underline{33.94} & \textbf{16.41} & \textbf{12.20} & \textbf{11.02} \\
 \bottomrule
\end{tabularx}
\label{tbl:main_fid_flow}
\end{table}

\vspace{-0.3\baselineskip}
\paragraph{Quantitative Results.}
We report Fréchet Inception Distance (FID)~\citep{Heusel:2017FID} scores across diverse datasets for diffusion models in Tab.~\ref{tbl:main_fid_diffusion} and for flow models in Tab.~\ref{tbl:main_fid_flow}. FID is computed between the reference set and 50K generated samples, where the test set for each dataset serves as the reference. For SD, both the reference and generated sets are constructed from disjoint subsets of 30K text prompts from MS-COCO, following the setup used in LD3~\citep{Tong:2025LD3}. Refer to App.~\ref{subsec:sd_clip_pickscore} for a more comprehensive comparison of SD, including text-image alignment metrics.


As shown in Tab.~\ref{tbl:main_fid_diffusion}, for few-step generation with pretrained diffusion models, \ours{} consistently achieves the best FID on CIFAR-10 across different NFEs, with especially large margins over the second-best at small NFEs (e.g., at NFE=4, \ours: 9.55 vs. LD3: 12.04 with UniPC, and \ours: 6.93 vs. LD3: 9.97 with iPNDM). On FFHQ and AFHQv2, \ours{} outperforms or remains comparable to the baselines. The improvements are particularly strong at small NFEs, for example, \ours: 17.05 vs. LD3 (the second-best): 22.48 at NFE=4 with UniPC.

When it comes to flow models, as shown in Tab.~\ref{tbl:main_fid_flow}, \ours{} achieves state-of-the-art results on CIFAR-10 with both RK1 and RK2, outperforming the others by clear margins. For example, we surpass the second-best LD3 by 18.31 at NFE=4 with RK1, and the second-best GITS by 9.66 at NFE=4 with RK2. On ImageNet, we consistently obtain the best results across most NFEs, except at NFE=4, again by large margins. On MS-COCO evaluated with Stable Diffusion v3.5, \ours{} outperforms the baselines at most NFEs, demonstrating generalizability to large-scale models.

Overall, these results demonstrate that \ours{} attains the best or comparable performance to existing acceleration approaches across diverse experiment setups, including both diffusion and flow models, different NFEs, ODE solvers, and datasets.





\vspace{-0.75\baselineskip}
\paragraph{Qualitative Results.}
We present qualitative results for accelerated sampling of diffusion models in Fig.~\ref{fig:main_qualitative_diffusion} and flow models in Fig.~\ref{fig:main_qualitative_flow}. Across both model classes, \ours{} (BF) consistently produces sharper details and fewer artifacts at low NFEs. Notably, in the last row of Fig.~\ref{fig:main_qualitative_diffusion} (right), the baselines fail to generate a plausible animal face, whereas our method produces a clear and realistic cat face. See App.~\ref{sec:more_quali} for more qualitative results.


\begin{figure}[t!]
\centering
{\scriptsize
\setlength{\tabcolsep}{4pt} 
\begin{minipage}[t]{0.49\textwidth}
  \centering
  \begin{tabularx}{\linewidth}{@{} >{\centering\arraybackslash}m{1.3em} @{\hspace{6pt}}| *{5}{>{\centering\arraybackslash}m{4.18em}} @{}}
    \toprule
    NFE & \hspace{0.3em}UniPC & DMN & \hspace{-0.3em}GITS & \hspace{-0.7em}LD3 & \hspace{-1.8em}BF \\
    \midrule
    \multicolumn{6}{@{}c@{}}{FFHQ $64\times64$ with EDM~\citep{Karras:2022EDM}}\\
    \midrule
    \multirow{2}{*}{6}
    & \multicolumn{5}{@{}c@{}}{%
         \includegraphics[height=1.12cm]{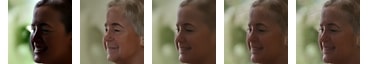}%
       }\\
    & \multicolumn{5}{@{}c@{}}{%
         \includegraphics[height=1.12cm]{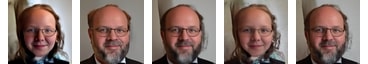}%
       }\\
    \midrule
    \multirow{2}{*}{8}
    & \multicolumn{5}{@{}c@{}}{%
         \includegraphics[height=1.12cm]{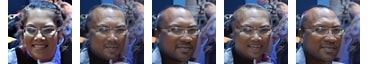}
       }\\
    & \multicolumn{5}{@{}c@{}}{%
         \includegraphics[height=1.12cm]{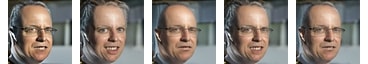}
       }\\
    \bottomrule
  \end{tabularx}
\end{minipage}\hfill
\begin{minipage}[t]{0.49\textwidth}
  \centering
  \begin{tabularx}{\linewidth}{@{} >{\centering\arraybackslash}m{1.3em} @{\hspace{6pt}}| *{5}{>{\centering\arraybackslash}m{4.18em}} @{}}
    \toprule
    NFE & \hspace{0.3em}UniPC & DMN & \hspace{-0.3em}GITS & \hspace{-0.7em}LD3 & \hspace{-1.8em}BF \\
    \midrule
    \multicolumn{6}{@{}c@{}}{AFHQv2 $64\times64$ with EDM~\citep{Karras:2022EDM}}\\
    \midrule
    \multirow{2}{*}{6}
    & \multicolumn{5}{@{}c@{}}{%
         \includegraphics[height=1.12cm]{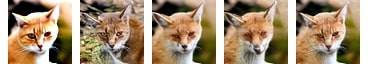}
       }\\
    & \multicolumn{5}{@{}c@{}}{%
         \includegraphics[height=1.12cm]{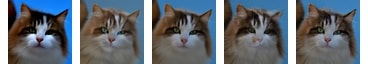}
       }\\
    \midrule
    \multirow{2}{*}{8}
    & \multicolumn{5}{@{}c@{}}{%
         \includegraphics[height=1.12cm]{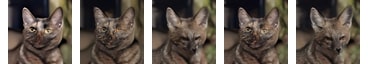}
       }\\
    & \multicolumn{5}{@{}c@{}}{%
         \includegraphics[height=1.12cm]{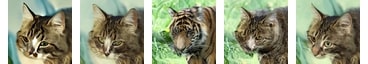}
       }\\
    \bottomrule
  \end{tabularx}
\end{minipage}
}
\caption{\textbf{Qualitative comparisons of samples generated using NFEs 6 and 8 on FFHQ and
AFHQv2 datasets.} We use UniPC solver as the base solver for both cases.}
\label{fig:main_qualitative_diffusion}
\end{figure}
\begin{figure}[t!]
\centering
{\scriptsize
\setlength{\tabcolsep}{4pt} 
\begin{minipage}[t]{0.485\textwidth}
  \centering
  \begin{tabularx}{\linewidth}{@{} >{\centering\arraybackslash}m{1.3em} @{\hspace{6pt}}| *{6}{>{\centering\arraybackslash}m{3.3em}} @{}}
    \toprule
    NFE & RK2 & \hspace{-0.2em}DMN & \hspace{-0.65em}GITS & \hspace{-1em}Bespoke & \hspace{-1.3em}LD3 & \hspace{-2.1em}BF \\
    \midrule
    \multicolumn{7}{@{}c@{}}{CIFAR-10 $32\times 32$ with ReFlow~\citep{Liu:2023RF}} \\
    \midrule
    \multirow{2}{*}{6}
    & \multicolumn{6}{@{}c@{}}{%
         \includegraphics[height=1.01cm]{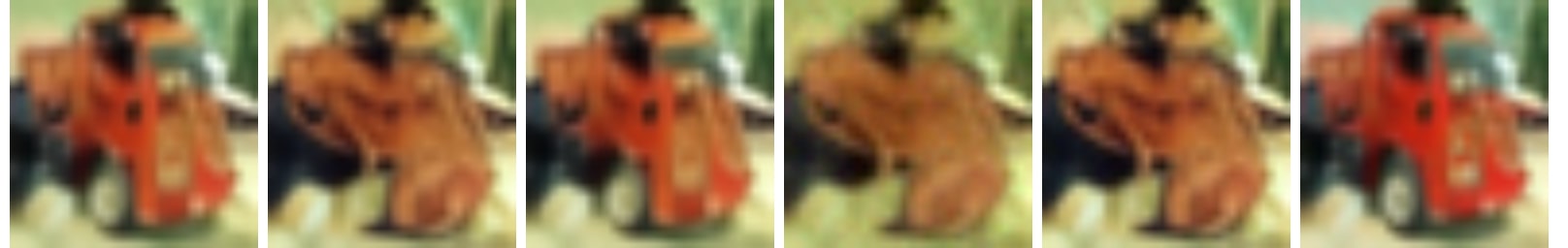}%
       }\\
    & \multicolumn{6}{@{}c@{}}{%
         \includegraphics[height=1.01cm]{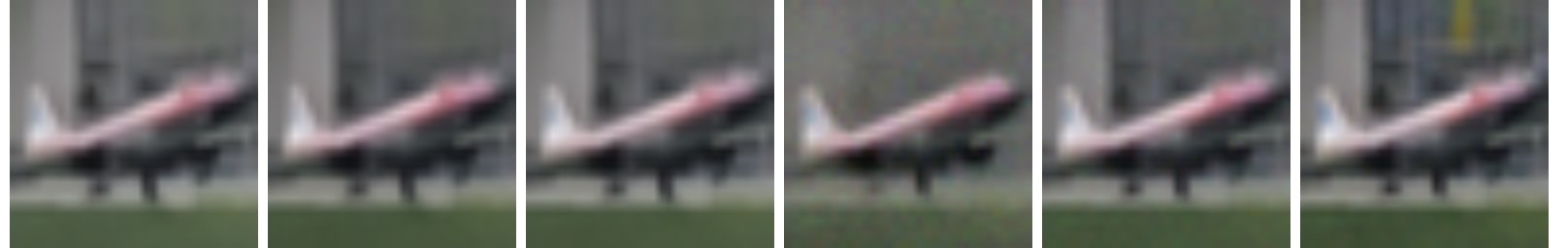}%
       }\\
    \midrule
    \multirow{2}{*}{8}
    & \multicolumn{6}{@{}c@{}}{%
         \includegraphics[height=1.01cm]{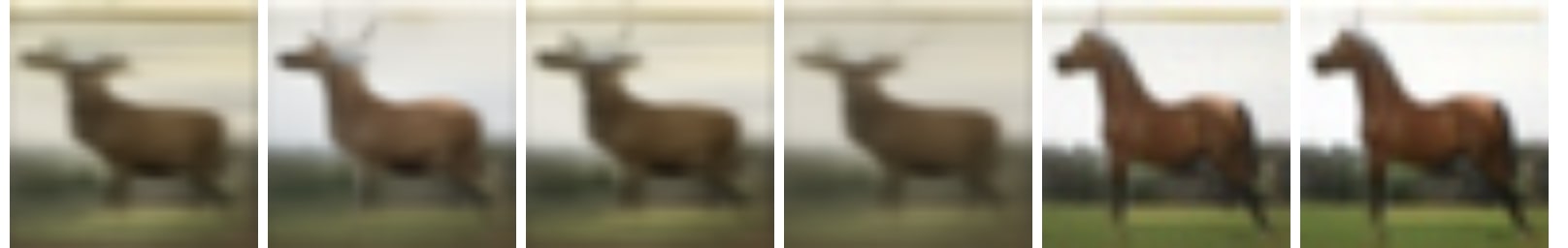}
       }\\
    & \multicolumn{6}{@{}c@{}}{%
         \includegraphics[height=1.01cm]{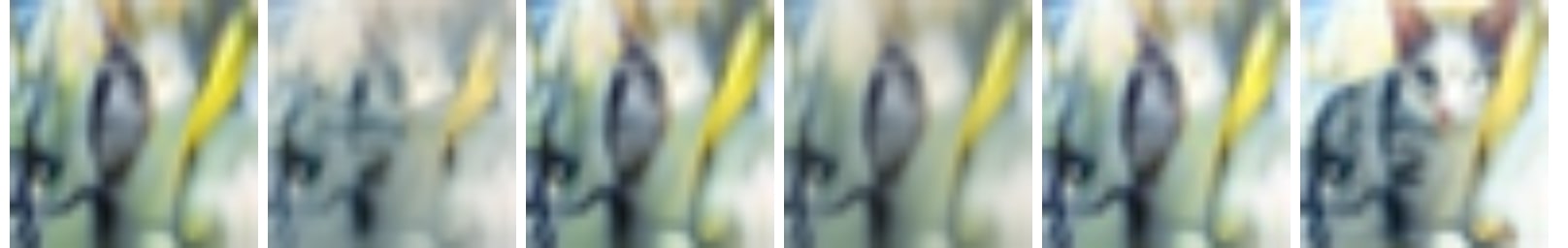}
       }\\
    \bottomrule
  \end{tabularx}
\end{minipage}\hfill
\begin{minipage}[t]{0.485\textwidth}
  \centering
  \begin{tabularx}{\linewidth}{@{} >{\centering\arraybackslash}m{1.3em} @{\hspace{6pt}}| *{6}{>{\centering\arraybackslash}m{3.3em}} @{}}
    \toprule
    NFE & RK2 & \hspace{-0.2em}DMN & \hspace{-0.65em}GITS & \hspace{-1em}Bespoke & \hspace{-1.3em}LD3 & \hspace{-2.1em}BF \\
    \midrule
    \multicolumn{7}{@{}c@{}}{ImageNet $256\times 256$ with FlowDCN~\citep{Wang:2024FlowDCN}}\\
    \midrule
    \multirow{2}{*}{6}
    & \multicolumn{6}{@{}c@{}}{%
         \includegraphics[height=1.01cm]{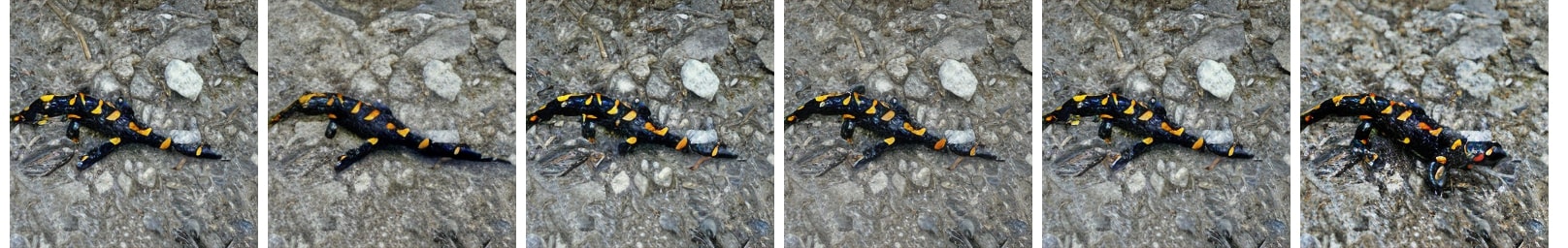}
       }\\
    & \multicolumn{6}{@{}c@{}}{%
         \includegraphics[height=1.01cm]{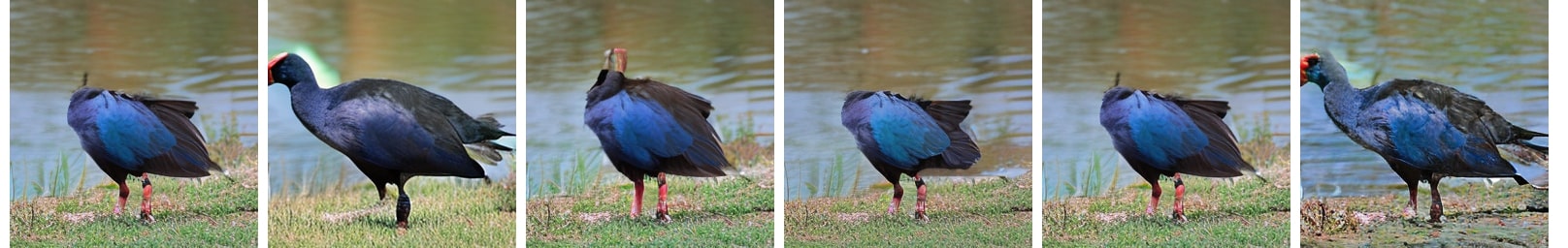}
       }\\
    \midrule
    \multirow{2}{*}{8}
     & \multicolumn{6}{@{}c@{}}{%
         \includegraphics[height=1.01cm]{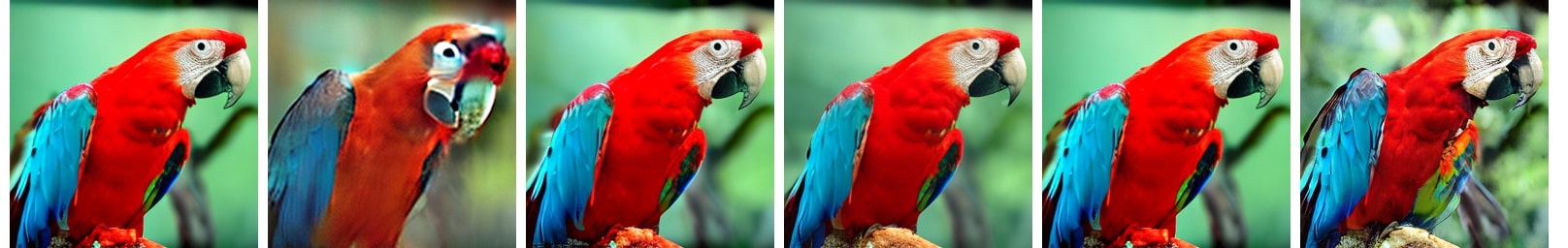}
       }\\
     & \multicolumn{6}{@{}c@{}}{%
         \includegraphics[height=1.01cm]{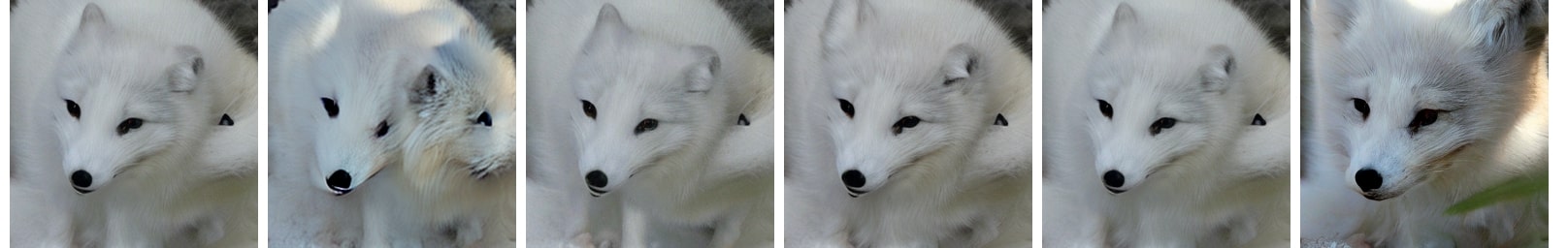}
       }\\
    \bottomrule
  \end{tabularx}
\end{minipage}
}
\caption{\textbf{Qualitative comparisons of samples generated using NFEs 6 and 8 on CIFAR-10 and ImageNet datasets.} We use RK2 solver as the base solver for both cases.}
\vspace{-\baselineskip}
\label{fig:main_qualitative_flow}
\end{figure}


\begin{figure}[t]
  \centering
  {\small\makebox[0.52\textwidth][c]{AFHQv2}\hfill
  \makebox[0.48\textwidth][c]{CIFAR-10}}

  \par\vspace{0.4em}
  \begin{subfigure}[t]{0.25\textwidth}
    \centering
    \includegraphics[width=\linewidth]{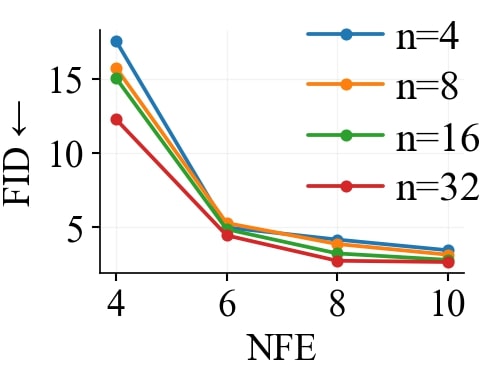}
    \caption{UniPC}
    \label{fig:fid-unipc}
  \end{subfigure}\hfill
  \begin{subfigure}[t]{0.25\textwidth}
    \centering
    \includegraphics[width=\linewidth]{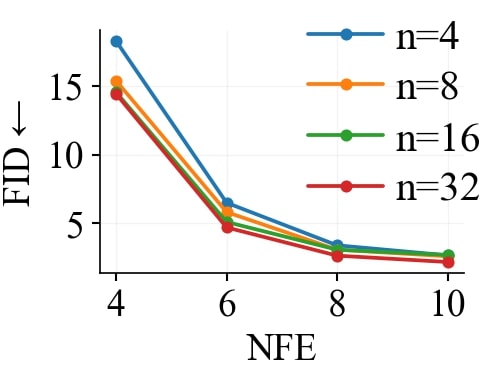}
    \caption{iPNDM}
    \label{fig:fid-ipndm}
  \end{subfigure}\hfill
  \begin{subfigure}[t]{0.25\textwidth}
    \centering
    \includegraphics[width=\linewidth]{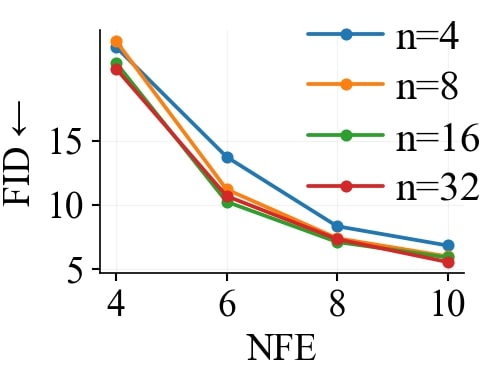}
    \caption{RK1}
    \label{fig:fid-euler}
  \end{subfigure}\hfill
  \begin{subfigure}[t]{0.25\textwidth}
    \centering
    \includegraphics[width=\linewidth]{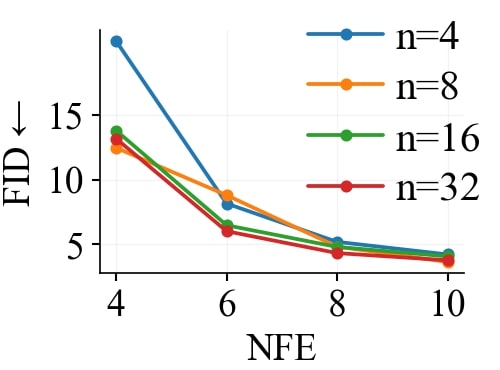}
    \caption{RK2}
    \label{fig:fid-midpoint}
  \end{subfigure}

  \caption{\textbf{Effect of the degree of B\'ezier functions.} Each line reports FID scores on CIFAR-10 across Bézier function degrees $n$ from 4 to 32. Higher degrees yield lower FID, with convergence around $n=32$.}
  \label{fig:bezier_order}
\end{figure}

\begin{center}
\setlength{\tabcolsep}{6pt}            
\renewcommand{\arraystretch}{1.15}     

\begin{minipage}[t]{0.54\linewidth}
  \centering
  \captionof{table}{\textbf{Generalizability of BézierFlow to unseen NFEs.} Each column corresponds to the inference NFE. Baselines are trained with the same NFE used at inference, whereas ours is trained once on NFE=10 without decoupled timesteps and directly applied to unseen NFEs. Reported results are FID scores on CIFAR-10 (lower is better; best in \textbf{bold}).}
  \label{tab:scheduler-transfer}
  {\scriptsize
  \begin{tabular*}{\linewidth}{@{\extracolsep{\fill}}lcccccccc}
    \toprule
    \multirow{2}{*}{Method} &
      \multicolumn{3}{c}{RK1} &
      \multicolumn{3}{c}{RK2} \\
    \cmidrule(lr){2-4}\cmidrule(lr){5-7}
     & 6 & 8 & 10
     & 6 & 8 & 10 \\
    \midrule
    
    GITS & 26.11 & 19.89 & 15.34 & 11.84 & 8.77 & 6.58 \\
    Bespoke & 18.08 & 11.88 & 9.25 & 64.87 & 16.67 & 13.34\\
    LD3 & 20.10 & 12.54 & 9.64 & 13.82 & 6.26 & 3.86 \\
    \midrule
    BF (NFE=10) & \textbf{18.50} & \textbf{9.02} & \textbf{6.01} & \textbf{9.57} & \textbf{5.32} & \textbf{3.71} \\
    \bottomrule
  \end{tabular*}}
\end{minipage}
\hfill
\begin{minipage}[t]{0.45\linewidth}
  \centering
  \captionof{table}{\textbf{Comparison with distillation methods on CIFAR-10.} The middle column reports FID and NFE. Our training time is measured on an NVIDIA A6000, while distillation times are directly from the original papers on NVIDIA A100s. Best in \textbf{bold}.}
  \label{tab:cifar10_ours_vs_baselines}
  {\scriptsize
  \begin{tabular*}{\linewidth}{@{\extracolsep{\fill}}lcc}
    \toprule
    Method & FID ($\downarrow$) / NFE & Training Time \\
    \midrule
    \multicolumn{3}{c}{CIFAR-10 $32{\times}32$ — Diffusion} \\
    \midrule
    CD   & 2.93 / NFE=2 & 8 days \\
    BF w/ UniPC                    & \textbf{2.09} / NFE=10 & 15 minutes \\
    \midrule
    \multicolumn{3}{c}{CIFAR-10 $32{\times}32$ — Flow}
    \\
    \midrule
    2-RF & 3.85 / NFE=2 & 8 days \\
    BF w/ RK2                         & \textbf{3.74} / NFE=10 & 15 minutes \\
    \bottomrule
  \end{tabular*}}
\end{minipage}
\end{center}

\vspace{-1.0\baselineskip}
\paragraph{Effect of the Degree of B\'ezier Functions.}
Fig.~\ref{fig:bezier_order} shows the effect of the degree of B\'ezier functions by varying the number of control points from $4$ to $32$. Across different datasets, NFEs, and base ODE solvers, increasing the number of control points consistently improves FID, indicating that higher-degree B\'ezier functions provide greater expressiveness for learning optimal sampling trajectories. Empirically, we observe that the additional gains between $n=16$ and $n=32$ become marginal, so we adopt $n=32$ as our default. Note that the training time overhead from increasing $n$ is negligible, with only a few seconds difference between $n=4$ and $n=32$.

\vspace{-0.5\baselineskip}
\paragraph{Generalizability to Unseen NFEs.}
As discussed in Sec.~\ref{subsec:connection_to_prior_Work}, unlike prior works~\citep{Xue:2024DMN, Chen:2024GITS, Tong:2025LD3, Shaul:2023Bespoke} that learn \emph{discrete per-step variables}, \ours{} learns the sampling trajectory with continuous functions, enabling generalization to NFEs unseen during training. As shown in Tab.~\ref{tab:scheduler-transfer}, \ours{} trained with NFE=10 also performs well at NFE=6 and NFE=8, achieving outperforming FID scores against the baselines trained directly at those NFEs.


\vspace{-0.5\baselineskip}
\paragraph{Training Efficiency.} 
In Tab.~\ref{tab:cifar10_ours_vs_baselines}, we compare \ours{} on CIFAR-10~\citep{Krizhevsky:2009CIFAR} against representative distillation-based approaches: Consistency Distillation (CD)~\citep{Song:2023CM} for diffusion models and 2-Rectified Flow (2-RF)~\citep{Liu:2023RF} for flow models. While \ours{} achieves better FID scores with a larger inference NFE, its \emph{training} cost is significantly lower, requiring only \textbf{15 minutes} compared to \textbf{8 days} for distillation, which corresponds to approximately $\mathbf{0.13\%}$ of the training time. A more comprehensive comparison of training time is provided in App.~\ref{subsec:comparison_with_few_step}.

\vspace{-0.5\baselineskip}
\paragraph{Combination with LD3.}
Although the search space of LD3 is a subset of that of \ours{}, LD3 optimizes discrete timesteps whereas \ours{} learns continuous sampling paths, making it natural to ask whether jointly optimizing the two could offer complementary improvements. We therefore optimize both the target timesteps and the scheduler in a unified framework. As shown in Tab.~\ref{tbl:ours_ld3} of the Appendix, however, the combination does not offer clear advantages over using \ours{} alone. This indicates that the advantages provided by LD3 are already captured within the learned scheduler, and thus do not translate into further gains when optimized jointly.



\section{Conclusion}
We introduce \ours{}, a lightweight training framework for few-step generation. By combining the optimization of sampling trajectories, rather than discrete ODE timesteps, with a B\'ezier-based continuous parameterization, \ours{} achieves consistent improvements across diffusion and flow models with only minutes of training, surpassing existing lightweight training approaches. For future work, we plan to explore alternative basis functions for B\'ezier functions, which may enable richer expressiveness with fewer control points.  


\paragraph*{Ethics Statement.}
We affirm adherence to the ICLR Code of Ethics. This work relies only on publicly available models and datasets and does not involve human subjects, user data, or personally identifiable information. We acknowledge the potential for misuse of generative AI and encourage responsible deployment and use of our method.

\section*{Acknowledgements}
This work was supported by IITP grants (RS-2024-00399817, RS-2025-25441313, RS-2025-25443318), funded by the Korean government (MSIT); the Industrial Technology Innovation Program (RS-2025-02317326), funded by the Korean government (MOTIE); the National Supercomputing Center (KSC-2025-CRE-0475); and the DRB-KAIST SketchTheFuture Research Center.

{\small
\bibliography{main}

@String(CVPR= {IEEE Conf. Comput. Vis. Pattern Recog.})

@String(ICCV= {Int. Conf. Comput. Vis.})

@String(ECCV= {Eur. Conf. Comput. Vis.})

@String(NIPS= {Adv. Neural Inform. Process. Syst.})

@String(ICLR = {Int. Conf. Learn. Represent.})

@String(ICML  = {ICML})

@String(TMLR = {TMLR})

@String{EMNLP = {EMNLP}}

@String(CVPR  = {CVPR})

@String(ICCV  = {ICCV})

@String(ECCV  = {ECCV})

@String(NIPS  = {NeurIPS})

@String(ICLR  = {ICLR})

@article{Krizhevsky:2009CIFAR,
    title={{Learning multiple layers of features from tiny images}},
    author={Krizhevsky, Alex},
    year={2009},
}

@article{Albergo:2023SI,
  title={{Stochastic Interpolants: A Unifying Framework for Flows and Diffusions}},
  author={Albergo, Michael S. and Boffi, Nicholas M. and Vanden-Eijnden, Eric},
  journal={JMLR},
  year={2023},
}

@inproceedings{Deng2009ImageNet,
  title={{ImageNet: A large-scale hierarchical image database}}, 
  booktitle=CVPR, 
  author={Deng, Jia and Dong, Wei and Socher, Richard and Li, Li-Jia and Kai Li and Li Fei-Fei},
  year={2009},
}

@inproceedings{Esser:2024SD3,
    title = {{Scaling rectified flow transformers for high-resolution image synthesis}},
    author = {Esser, Patrick and Kulal, Sumith and Blattmann, Andreas and Entezari, Rahim and M\"{u}ller, Jonas and Saini, Harry and Levi, Yam and Lorenz, Dominik and Sauer, Axel and Boesel, Frederic and Podell, Dustin and Dockhorn, Tim and English, Zion and Rombach, Robin},
    booktitle=ICML,
    year = {2024},
}

@inproceedings{Shaul:2023Bespoke,
  title={Bespoke Solvers for Generative Flow Models},
  author={Neta Shaul and Juan Perez and Ricky T. Q. Chen and Ali Thabet and Albert Pumarola and Yaron Lipman},
  booktitle=ICLR,
  year={2024},
}

@inproceedings{Xue:2024DMN,
    title={Accelerating Diffusion Sampling with Optimized Time Steps}, 
    author={Xue, Shuchen and Liu, Zhaoqiang and Chen, Fei and Zhang, Shifeng and Hu, Tianyang and Xie, Enze and Li, Zhenguo},
    booktitle=CVPR,
    year={2024},
}

@inproceedings{Chen:2024GITS,
  title={On the Trajectory Regularity of ODE-based Diffusion Sampling},
  author={Chen, Defang and Zhou, Zhenyu and Wang, Can and Shen, Chunhua and Lyu, Siwei},
  booktitle=ICML,
  year={2024}
}

@inproceedings{Tong:2025LD3,
  title={Learning to Discretize Denoising Diffusion ODEs},
  author={Tong, Vinh and Hoang, Trung-Dung and Liu, Anji and Van den Broeck, Guy and Niepert, Mathias},
  booktitle=ICLR,
  year={2025},
}

@inproceedings{Kingma:2021VDM,
      title={Variational Diffusion Models}, 
      author={Diederik P. Kingma and Tim Salimans and Ben Poole and Jonathan Ho},
      booktitle=NIPS,
      year={2023},
}

@inproceedings{Song:2021ScoreSDE,
      title={Score-Based Generative Modeling through Stochastic Differential Equations}, 
      author={Yang Song and Jascha Sohl-Dickstein and Diederik P. Kingma and Abhishek Kumar and Stefano Ermon and Ben Poole},
      booktitle=ICLR,
      year={2021},
}

@inproceedings{Song:2023CM,
      title={Consistency Models}, 
      author={Yang Song and Prafulla Dhariwal and Mark Chen and Ilya Sutskever},
      booktitle=ICML,
      year={2023},
}

@inproceedings{Salimans:2022PD,
      title={Progressive Distillation for Fast Sampling of Diffusion Models}, 
      author={Tim Salimans and Jonathan Ho},
      booktitle=ICLR,
      year={2022},
}

@inproceedings{Liu:2023RF,
      title={Flow Straight and Fast: Learning to Generate and Transfer Data with Rectified Flow}, 
      author={Xingchao Liu and Chengyue Gong and Qiang Liu},
      booktitle=ICLR,
      year={2023},
}

@inproceedings{Lu:2022DPMSolver,
      title={DPM-Solver: A Fast ODE Solver for Diffusion Probabilistic Model Sampling in Around 10 Steps}, 
      author={Cheng Lu and Yuhao Zhou and Fan Bao and Jianfei Chen and Chongxuan Li and Jun Zhu},
      booktitle=NIPS,
      year={2022},
}

@inproceedings{Lu:2023DPMSolver++,
   title={DPM-Solver++: Fast Solver for Guided Sampling of Diffusion Probabilistic Models},
   author={Lu, Cheng and Zhou, Yuhao and Bao, Fan and Chen, Jianfei and Li, Chongxuan and Zhu, Jun},
   booktitle=ICLR,
   year={2023},
}

@inproceedings{zhao:2023UniPC,
      title={UniPC: A Unified Predictor-Corrector Framework for Fast Sampling of Diffusion Models}, 
      author={Wenliang Zhao and Lujia Bai and Yongming Rao and Jie Zhou and Jiwen Lu},
      booktitle=NIPS,
      year={2023},
}

@inproceedings{Zhang:2023iPNDM,
      title={Fast Sampling of Diffusion Models with Exponential Integrator}, 
      author={Qinsheng Zhang and Yongxin Chen},
      booktitle=ICLR,
      year={2023},
}

@article{BUTCHER:1996RK,
title = {A history of Runge-Kutta methods},
journal = {Applied Numerical Mathematics},
year = {1996},
author = {J.C. Butcher},
}

@inproceedings{Karras:2019FFHQ,
      title={A Style-Based Generator Architecture for Generative Adversarial Networks}, 
      author={Tero Karras and Samuli Laine and Timo Aila},
      booktitle=CVPR,
      year={2019},
}

@inproceedings{Choi:2020AFHQv2,
      title={StarGAN v2: Diverse Image Synthesis for Multiple Domains}, 
      author={Yunjey Choi and Youngjung Uh and Jaejun Yoo and Jung-Woo Ha},
      booktitle=CVPR,
      year={2020},
}

@inproceedings{Li:2024TimeShift,
      title={Alleviating Exposure Bias in Diffusion Models through Sampling with Shifted Time Steps}, 
      author={Mingxiao Li and Tingyu Qu and Ruicong Yao and Wei Sun and Marie-Francine Moens},
      booktitle=ICLR,
      year={2024},
}

@article{Pokle:2023training,
  title={Training-free linear image inverses via flows},
  author={Pokle, Ashwini and Muckley, Matthew J and Chen, Ricky TQ and Karrer, Brian},
  journal=TMLR,
  year={2024}
}

@inproceedings{Ho:2020DDPM,
      title={Denoising Diffusion Probabilistic Models}, 
      author={Jonathan Ho and Ajay Jain and Pieter Abbeel},
      booktitle=NIPS,
      year={2020},
}

@inproceedings{Karras:2022EDM,
      title={Elucidating the Design Space of Diffusion-Based Generative Models}, 
      author={Tero Karras and Miika Aittala and Timo Aila and Samuli Laine},
      year={2022},
      booktitle=NIPS
}

@article{Lipman:2024FMGuide,
      title={Flow Matching Guide and Code}, 
      author={Yaron Lipman and Marton Havasi and Peter Holderrieth and Neta Shaul and Matt Le and Brian Karrer and Ricky T. Q. Chen and David Lopez-Paz and Heli Ben-Hamu and Itai Gat},
      year={2024},
      journal={arXiv},
}

@inproceedings{Kim:2025RBF,
  title={Inference-time scaling for flow models via stochastic generation and rollover budget forcing},
  author={Kim, Jaihoon and Yoon, Taehoon and Hwang, Jisung and Sung, Minhyuk},
  booktitle=NIPS,
  year={2025}
}

@inproceedings{Lin:2014MSCOCO,
  title={Microsoft coco: Common objects in context},
  author={Lin, Tsung-Yi and Maire, Michael and Belongie, Serge and Hays, James and Perona, Pietro and Ramanan, Deva and Doll{\'a}r, Piotr and Zitnick, C Lawrence},
  booktitle=ECCV,
  year={2014},
}

@inproceedings{Heusel:2017FID,
  title={Gans trained by a two time-scale update rule converge to a local nash equilibrium},
  author={Heusel, Martin and Ramsauer, Hubert and Unterthiner, Thomas and Nessler, Bernhard and Hochreiter, Sepp},
  booktitle=NIPS,
  year={2017}
}

@inproceedings{Zhang:2018LPIPS,
  title={The unreasonable effectiveness of deep features as a perceptual metric},
  author={Zhang, Richard and Isola, Phillip and Efros, Alexei A and Shechtman, Eli and Wang, Oliver},
  booktitle=CVPR,
  year={2018}
}

@inproceedings{Kim:2023CTM,
  title={Consistency trajectory models: Learning probability flow ode trajectory of diffusion},
  author={Kim, Dongjun and Lai, Chieh-Hsin and Liao, Wei-Hsiang and Murata, Naoki and Takida, Yuhta and Uesaka, Toshimitsu and He, Yutong and Mitsufuji, Yuki and Ermon, Stefano},
  booktitle=ICLR,
  year={2024}
}

@inproceedings{Song:2021DDIM,
      title={Denoising Diffusion Implicit Models}, 
      author={Jiaming Song and Chenlin Meng and Stefano Ermon},
      booktitle=ICLR,
      year={2021},
}

@inproceedings{Lipman:2023FM,
      title={Flow Matching for Generative Modeling}, 
      author={Yaron Lipman and Ricky T. Q. Chen and Heli Ben-Hamu and Maximilian Nickel and Matt Le},
      booktitle=ICLR,
      year={2023},
}

@article{Berthelot:2023TRACT,
  title={Tract: Denoising diffusion models with transitive closure time-distillation},
  author={Berthelot, David and Autef, Arnaud and Lin, Jierui and Yap, Dian Ang and Zhai, Shuangfei and Hu, Siyuan and Zheng, Daniel and Talbott, Walter and Gu, Eric},
  journal={arXiv},
  year={2023}
}

@inproceedings{Zhou:2025IMM,
title={Inductive Moment Matching},
author={Linqi Zhou and Stefano Ermon and Jiaming Song},
booktitle=ICML,
year={2025},
}

@inproceedings{Wang:2024FlowDCN,
      title={FlowDCN: Exploring DCN-like Architectures for Fast Image Generation with Arbitrary Resolution}, 
      author={Shuai Wang and Zexian Li and Tianhui Song and Xubin Li and Tiezheng Ge and Bo Zheng and Limin Wang},
      booktitle=NIPS,
      year={2024},
}

@inproceedings{Kirstain:2023Pick,
  title={Pick-a-pic: An open dataset of user preferences for text-to-image generation},
  author={Kirstain, Yuval and Polyak, Adam and Singer, Uriel and Matiana, Shahbuland and Penna, Joe and Levy, Omer},
  booktitle=NIPS,
  year={2023}
}

@inproceedings{hessel2021clipscore,
  title={Clipscore: A reference-free evaluation metric for image captioning},
  author={Hessel, Jack and Holtzman, Ari and Forbes, Maxwell and Bras, Ronan Le and Choi, Yejin},
  booktitle={EMNLP},
  year={2021}
}

@inproceedings{Zhou:2021pvd,
      title={3D Shape Generation and Completion through Point-Voxel Diffusion}, 
      author={Linqi Zhou and Yilun Du and Jiajun Wu},
      year={2021},
      booktitle=ICCV
}

@inproceedings{Guerreiro:2024layout,
      title={LayoutFlow: Flow Matching for Layout Generation}, 
      author={Julian Jorge Andrade Guerreiro and Naoto Inoue and Kento Masui and Mayu Otani and Hideki Nakayama},
      year={2024},
      booktitle=ECCV
}

@inproceedings{Zheng:2023layoutdiffusion,
      title={LayoutDiffusion: Controllable Diffusion Model for Layout-to-image Generation}, 
      author={Guangcong Zheng and Xianpan Zhou and Xuewei Li and Zhongang Qi and Ying Shan and Xi Li},
      year={2023},
      booktitle=CVPR
}

@inproceedings{Deka:2017rico,
author = {Deka, Biplab and Huang, Zifeng and Franzen, Chad and Hibschman, Joshua and Afergan, Daniel and Li, Yang and Nichols, Jeffrey and Kumar, Ranjitha},
title = {Rico: A Mobile App Dataset for Building Data-Driven Design Applications},
year = {2017},
booktitle = {Proceedings of the 30th Annual ACM Symposium on User Interface Software and Technology},
}

@inproceedings{Chang:2016shapenet,
      title={ShapeNet: An Information-Rich 3D Model Repository}, 
      author={Angel X. Chang and Thomas Funkhouser and Leonidas Guibas and Pat Hanrahan and Qixing Huang and Zimo Li and Silvio Savarese and Manolis Savva and Shuran Song and Hao Su and Jianxiong Xiao and Li Yi and Fisher Yu},
      year={2015},
      booktitle={3DV}
}

@inproceedings{Albergo:2024multimarginalsi,
      title={Multimarginal generative modeling with stochastic interpolants}, 
      author={Michael S. Albergo and Nicholas M. Boffi and Michael Lindsey and Eric Vanden-Eijnden},
      year={2024},
      booktitle=ICLR
}
\bibliographystyle{iclr2026_conference}
}


\clearpage
\newpage
\appendix



\section{Validity of Scheduler Reparameterization}
\label{sec:validity_of_scheduler_reparameterization}
We provide a detailed exposition of two key properties discussed in Sec.~\ref{subsec:transformed-path}: (i) the preservation of endpoint marginals, and (ii) the invaraince of the SI training objective with respect to the choice of schedule, provided that the SNR endpoints are identical.

\begin{proposition}
\label{prop:endpoint_marginal_preservation}
(Endpoint marginals equivalence)
Let $p_0$ and $p_1$ denote the endpoint marginals of $x_t$ governed by a scheduler $(\alpha_t, \sigma_t)$, $\bar{p}_0$ and $\bar{p}_1$ denote those of another stochastic interpolant $\bar{x}_s$. If $\bar{x}_s = c_s x_{t_s}$ with $c_s$ and $t_s$ defined in Eq.~\ref{eq:scale_transform}, then the endpoint marginals are preserved, i.e., $p_0 = \bar{p}_0$ and $p_1 = \bar{p}_1$. 
\end{proposition}


\begin{proof}
By the boundary conditions discussed in Sec.~\ref{sec:background}, any SI scheduler satisfies
\begin{align}
\alpha(0)=0,\, \alpha(1)=1,\, \sigma(0)=1,\, \sigma(1)=0.    
\end{align}
From Eq.~\ref{eq:scale_transform}, $c_s=1$ at both $s=0$ and $s=1$. By the definition of $t_s=\rho^{-1}(\bar\rho(s))$, we also have $t_{s=0}=0$ and $t_{s=1}=1$. Hence, $\bar{x}_{s=0}=x_{t=0}$ and $\bar{x}_{s=1}=x_{t=1}$. Since the pair of endpoints $(x_0, x_1)$ does not change under the sampling path transformation, the endpoint marginals also coincide: $p_0=\bar{p}_0$ and $p_1=\bar{p}_1$.
\end{proof}

\begin{proposition}
\label{prop:training_objective_invariance}
(Training objective invariance)
For any two schedulers $(\alpha_t,\sigma_t)$ and $(\bar\alpha_s,\bar\sigma_s)$ with matching SNR endpoints, training an SI model $S_\phi$ under either scheduler minimizes the same training objective, hence yields equivalent optima.
\end{proposition}

\begin{proof}
As noted in Sec.~\ref{sec:background}, different SI models (noise, velocity, or score predictors) learn different but interchangeable quantities, and can all be expressed in the denoiser form $\hat x_\phi$. For convenience, we therefore recall the continuous-time SI objective in the denoiser form (Eq. 18 from \citet{Kingma:2021VDM}): for a schedule $(\alpha_t,\sigma_t)$ with $\rho(t)=\alpha(t)/\sigma(t)$ strictly increasing, define $\nu=\rho^2$ and 
\begin{equation}
x_\nu=\alpha_\nu x_1+\sigma_\nu x_0,\qquad
\alpha_\nu:=\alpha\!\big(\rho^{-1}(\sqrt{\nu})\big),\ \ 
\sigma_\nu:=\sigma\!\big(\rho^{-1}(\sqrt{\nu})\big),
\end{equation}
with $\nu_{\min}=\rho(0)^2,\ \nu_{\max}=\rho(1)^2$. Then,
\begin{equation}
\mathcal L_{\nu}[\phi;(\alpha_t,\sigma_t)]
=\tfrac12\int_{\nu_{\min}}^{\nu_{\max}}
\mathbb{E}_{x_1\sim p_{\rm data},\,x_0\sim p_0}\!\left[
\|x_1-\hat x_\phi(x_\nu,\nu)\|_2^2\right]\,d\nu.
\end{equation}
Now consider another schedule $(\bar\alpha_s,\bar\sigma_s)$ with the same SNR endpoints. Since $\rho(t)=\frac{\alpha(t)}{\sigma(t)}$ and $\bar\rho(s)=\frac{\bar\alpha(s)}{\bar\sigma(s)}$ are strictly increasing, the maps $t\mapsto\nu=\rho(t)^2$ and $s\mapsto\nu=\bar\rho(s)^2$ are bijections onto the common interval $[\nu_{\text{min}},\nu_{\text{max}}]$. Now, for each fixed $\nu$, we have
\begin{equation}
    \nu=\frac{\alpha(t)^2}{\sigma(t)^2}=\frac{\bar\alpha(s)^2}{\bar\sigma(s)^2},
\end{equation}
which implies $\sigma_\nu=\frac{\alpha_\nu}{\sqrt\nu}$ and $\bar\sigma_\nu=\frac{\bar\alpha_\nu}{\sqrt\nu}$. Therefore, the interpolants satisfy 
\begin{equation}
    x_\nu=\alpha_\nu\!\Big(x_1+\tfrac{1}{\sqrt\nu}x_0\Big),\quad \bar x_\nu=\bar\alpha_\nu\!\Big(x_1+\tfrac{1}{\sqrt\nu}x_0\Big),
\end{equation}
so that $x_\nu=\tfrac{\alpha_\nu}{\bar\alpha_\nu}\,\bar x_\nu$.
This shows that the two interpolants differ only by a scalar rescaling factor, and since the integration limits are the same, we conclude
\begin{equation}
\mathcal L_{\nu}[\phi;(\alpha_t,\sigma_t)]
=
\mathcal L_{\nu}[\phi;(\bar\alpha_s,\bar\sigma_s)].
\end{equation}
\end{proof}

\section{Theoretical Analysis of Sampling Trajectory Spaces}
\label{sec:theoretical_analysis}

In this section, we formally show that the family of sampling trajectories realizable by BézierFlow is a super set of that of LD3~\citep{Tong:2025LD3}, offering better optimization advantages.


\begin{theorem}[Inclusion of LD3 in the BézierFlow Trajectory Space]
\label{thm:strict_inclusion}
Let $M$ and $D$ denote the number of sampling steps and the dimension of the state $x$, respectively. Let $\mathcal{X}_{\mathrm{LD3}}, \mathcal{X}_{\mathrm{BF}} \subseteq \mathbb{R}^{M+1 \times D}$ be the sets of sampling trajectories over discrete timesteps realizable by LD3 and \ours{} (parameterized by B\'ezier curves of degree $n \ge M$). Assuming the source sampling path defines a non-linear geometry, the trajectory space of LD3 is a subset of that of \ours{}:
\begin{equation}
    \mathcal{X}_{\mathrm{LD3}} \subsetneq\mathcal{X}_{\mathrm{BF}}.
\end{equation}
\end{theorem}

\begin{proof}
Let $\{t_k\}_{k=0}^M$ and $\{s_k\}_{k=0}^M$ denote the timesteps on the source scheduler and on the trajectory induced by the Bézier SI scheduler, respectively, with \[
    0 = t_0 < t_1 < \cdots < t_M = 1,
    \qquad
    0 = s_0 < s_1 < \cdots < s_M = 1.
\] 
Define
\begin{equation}
    \alpha_k := \alpha(t_k),\quad
    \sigma_k := \sigma(t_k),\quad
    x_{t_k} = \alpha_k x_1 + \sigma_k x_0,
    \qquad k=0,\dots,M.
\end{equation}
With this notation, any LD3 sampling trajectory over $M$ timesteps can be written as
\(
    \mathbf{x} = (x_{t_0},\dots,x_{t_M}) \in \mathcal{X}_{\mathrm{LD3}}.
\)

Since a B\'ezier curve of degree $n \ge M$ can interpolate any $M+1$ distinct values, there exists $\theta^\star$ such that
\begin{equation}
    \bar{\alpha}_{\theta^\star}(s_k) = \alpha_k,
    \qquad
    \bar{\sigma}_{\theta^\star}(s_k) = \sigma_k,
    \qquad \forall k.
\end{equation}
Hence
\begin{equation}
    \bar{\rho}_{\theta^\star}(s_k)
    :=
    \frac{\bar{\alpha}_{\theta^\star}(s_k)}{\bar{\sigma}_{\theta^\star}(s_k)}
    =
    \frac{\alpha_k}{\sigma_k}
    =
    \rho(t_k),
\end{equation}
and therefore
\begin{equation}
    t_{s_k} = \rho^{-1}\big(\bar{\rho}_{\theta^\star}(s_k)\big) = t_k,
    \qquad
    c_{s_k} = \frac{\bar{\sigma}_{\theta^\star}(s_k)}{\sigma(t_{s_k})}
            = \frac{\sigma_k}{\sigma_k} = 1.
\end{equation}
Using the sampling transformation in Eq.~\ref{eq:ubar},
\begin{equation}
    \bar{x}_{s_k} = c_{s_k} x_{t_{s_k}} = x_{t_k},
    \qquad \forall k,
\end{equation}
so every $\mathbf{x} \in \mathcal{X}_{\mathrm{LD3}}$ is also realizable by \ours{}, and thus
\begin{equation}
    \mathcal{X}_{\mathrm{LD3}} \subseteq \mathcal{X}_{\mathrm{BF}}.
\end{equation}

For strictness, fix $\{t_k\}_{k=0}^M$ and consider a target scheduler $\theta$ such that
\begin{equation}
    \bar{\rho}_{\theta}(s_k) = \rho(t_k),
    \qquad
    \bar{\sigma}_{\theta}(s_k) \neq \sigma_k
    \quad\text{for at least one } k.
\end{equation}
Then
\begin{equation}
    t_{s_k} = \rho^{-1}\big(\bar{\rho}_{\theta}(s_k)\big) = t_k,
    \qquad
    c_{s_k} = \frac{\bar{\sigma}_{\theta}(s_k)}{\sigma(t_{s_k})}
            = \frac{\bar{\sigma}_{\theta}(s_k)}{\sigma_k} \neq 1
    \quad\text{for at least one } k,
\end{equation}
and the resulting trajectory satisfies
\begin{equation}
    \bar{x}_{s_k} = c_{s_k} x_{t_k},
    \qquad c_{s_k} \neq 1\quad\text{for some } k.
\end{equation}
Since LD3 is constrained to the fixed source scheduler, which corresponds to sampling via Eq.~\ref{eq:ubar} with
\begin{equation}
    s_k = t_k,
    \qquad
    c_{s_k} \equiv 1,
\end{equation}
so any trajectory with $c_{s_k} \neq 1$ for some $k$ cannot be realized by LD3.
Thus
\begin{equation}
    \mathbf{x}_{\theta} \in \mathcal{X}_{\mathrm{BF}}
    \quad \text{and} \quad
    \mathbf{x}_{\theta} \notin \mathcal{X}_{\mathrm{LD3}},
\end{equation}
which implies
\begin{equation}
    \mathcal{X}_{\mathrm{LD3}} \subsetneq \mathcal{X}_{\mathrm{BF}}.
\end{equation}
\end{proof}

\begin{proposition}[Better Optima under Larger Trajectory Spaces]
\label{thm:optimization_advantage}
Let $\mathcal{X}_{\mathrm{LD3}}, \mathcal{X}_{\mathrm{BF}} \subseteq \mathbb{R}^{M+1 \times D}$ be the
trajectory spaces of LD3 and \ours{}, respectively, and suppose
\begin{equation}
    \mathcal{X}_{\mathrm{LD3}} \subsetneq \mathcal{X}_{\mathrm{BF}}
\end{equation}
as in Theorem~\ref{thm:strict_inclusion}.
Let
\(
    \mathcal{L} : \mathcal{X} \to \mathbb{R}
\)
be any real-valued objective functional (e.g., a distillation loss to a teacher).
Define the optimal objective values
\begin{equation}
    \mathcal{L}^\star_{\mathrm{LD3}}
    :=
    \inf_{\mathbf{x} \in \mathcal{X}_{\mathrm{LD3}}} \mathcal{L}(\mathbf{x}),
    \qquad
    \mathcal{L}^\star_{\mathrm{BF}}
    :=
    \inf_{\mathbf{x} \in \mathcal{X}_{\mathrm{BF}}} \mathcal{L}(\mathbf{x}).
\end{equation}
Then, the following inequality holds:
\begin{equation}
    \mathcal{L}^\star_{\mathrm{BF}}
    \;\le\;
    \mathcal{L}^\star_{\mathrm{LD3}}.
\end{equation}
\end{proposition}

\begin{proof}
Recall that for any two sets $\mathcal{A} \subseteq \mathcal{B}$ and an objective function $f$, the infimum over the superset is less than or equal to the infimum over the subset, i.e.,
\begin{equation}
    \inf_{x \in \mathcal{B}} f(x) \;\le\; \inf_{x \in \mathcal{A}} f(x).
\end{equation}
Since $\mathcal{X}_{\mathrm{LD3}} \subsetneq \mathcal{X}_{\mathrm{BF}}$, applying this property directly yields:
\begin{align}
    \mathcal{L}^\star_{\mathrm{BF}}
    &= \inf_{\mathbf{x} \in \mathcal{X}_{\mathrm{BF}}} \mathcal{L}(\mathbf{x})
      \;\le\; \inf_{\mathbf{x} \in \mathcal{X}_{\mathrm{LD3}}} \mathcal{L}(\mathbf{x})
      = \mathcal{L}^\star_{\mathrm{LD3}}.
\end{align}
Moreover, if there exists $\mathbf{x}' \in \mathcal{X}_{\mathrm{BF}} \setminus \mathcal{X}_{\mathrm{LD3}}$
such that $\mathcal{L}(\mathbf{x}') < \mathcal{L}^\star_{\mathrm{LD3}}$, then
\begin{align}
    \mathcal{L}^\star_{\mathrm{BF}} &< \mathcal{L}^\star_{\mathrm{LD3}}.
\end{align}

\end{proof}

\begin{table}[!t]
\centering
\scriptsize
\caption{\textbf{FID comparison of \ours{}, LD3 and their combination, denoted as Both.} The best results are highlighted in~\textbf{bold} and the second best results are \underline{underlined}. Gray cells indicate the base ODE solvers.} 
\setlength{\tabcolsep}{5pt}
\begin{tabularx}{\textwidth}
  {>{\raggedright\arraybackslash}p{0.1\textwidth}
   |*{4}{>{\centering\arraybackslash}X}
   |>{\raggedright\arraybackslash}p{0.1\textwidth}
   |*{4}{>{\centering\arraybackslash}X}}
\toprule
 Method & NFE=4 & NFE=6 & NFE=8 & NFE=10 & Method & NFE=4 & NFE=6 & NFE=8 & NFE=10 \\
 \midrule
\cellcolor{gray!20}UniPC & \multicolumn{4}{c|}{CIFAR-10 with EDM (Teacher FID: 2.08)} &
\cellcolor{gray!20}UniPC & \multicolumn{4}{c}{FFHQ with EDM (Teacher FID: 2.86)}\\
 \midrule
 + LD3 & 12.04 & 3.56 & \underline{2.43} & \underline{2.62} & + LD3 & 22.48 & \underline{6.16} & 4.25 & 2.92 \\
 + \ours{} & \underline{9.55} & \textbf{3.13} & \textbf{2.40} & \textbf{2.09} & + \ours{} & \textbf{17.05} & \textbf{7.43} & \textbf{3.82} & \underline{3.13} \\
 + Both & \textbf{9.32} & \underline{3.37} & 2.44 & 2.71 & + Both & \underline{20.77} & 6.24 & \underline{4.13} & \textbf{3.04} \\
 \midrule
\cellcolor{gray!20}RK2 & \multicolumn{4}{c|}{CIFAR-10 with ReFlow (Teacher FID: 2.70)} & \cellcolor{gray!20}RK2 &
\multicolumn{4}{c}{ImageNet with FlowDCN (Teacher FID: 15.89)}\\
 \midrule
 + LD3 & 29.45 & 13.82 & 6.26 & 3.86 & + LD3 & \textbf{7.59} & 10.17 & 12.75 & 14.04 \\
 + \ours{} & \underline{13.18} & \underline{6.00} & \underline{4.31} & \underline{3.74} & + \ours{} & 9.50 & \textbf{5.94} & \textbf{6.22} & \textbf{7.56} \\
 + Both & \textbf{12.23} & \textbf{5.50} & \textbf{3.74} & \textbf{3.17} & + Both & 9.11 & \underline{6.64} & \underline{9.38} & \underline{10.88} \\
 \bottomrule
\end{tabularx}
\label{tbl:ours_ld3}
\end{table}

\vspace{-1\baselineskip}
\section{Comparison of Scheduler Parameterizations: BézierFlow vs. Bespoke Solver}
\label{sec:comparison_with_bespoke}
\begin{table}[!t]
\centering
\scriptsize
\caption{\textbf{FID comparison of \ours{}, Bespoke Solver and Bespoke Solver trained with our training loss, denoted as Bespoke*.} Results for the base solvers are reported on each top rows. The best results are highlighted in~\textbf{bold} and the second best results are \underline{underlined}. Gray cells indicate the base ODE solvers.} 
\setlength{\tabcolsep}{5pt}
\begin{tabularx}{\textwidth}
  {>{\raggedright\arraybackslash}p{0.1\textwidth}
   |*{4}{>{\centering\arraybackslash}X}
   |>{\raggedright\arraybackslash}p{0.1\textwidth}
   |*{4}{>{\centering\arraybackslash}X}}
\toprule
 Method & NFE=4 & NFE=6 & NFE=8 & NFE=10 & Method & NFE=4 & NFE=6 & NFE=8 & NFE=10 \\
 \midrule
\multicolumn{10}{c}{CIFAR-10 $32\times 32$ with ReFlow~\citep{Liu:2023RF} (Teacher FID: 2.70)} \\
 \midrule
 \cellcolor{gray!20}RK1 & 52.78 & 26.30 & 17.40 & 13.30 & \cellcolor{gray!20}RK2 & 25.36 & \underline{12.12} & 9.17 & 7.89 \\
 + Bespoke & 45.31 & 18.08 & 11.88 & 9.25 & + Bespoke & 39.45 & 64.87 & 16.67 & 13.34 \\
 + Bespoke* & \underline{38.34} & \underline{17.28} & \underline{10.34} & \underline{7.65} & + Bespoke* & \underline{19.44} & 49.65 & \underline{4.40} & \textbf{3.70} \\
 + \ours{} & \textbf{20.64} & \textbf{9.67} & \textbf{7.30} & \textbf{5.51} & + \ours{} & \textbf{13.18} & \textbf{6.00} & \textbf{4.31} & \underline{3.74} \\
 \bottomrule
\end{tabularx}
\label{tbl:lpips_comparison}
\end{table}
As discussed in Sec.\ref{subsec:connection_to_prior_Work}, both Bespoke Solver~\citep{Shaul:2023Bespoke} and \ours{} aim to learn sampling trajectories, but differ in (i) parameterization and (ii) training objective. Bespoke Solver~\citep{Shaul:2023Bespoke} employs discrete per-step parameterization and minimizes step-wise $\ell_2$ errors against teacher outputs, whereas \ours{} adopts a B\'ezier-based continuous parameterization and is trained with a global truncation loss, computed along the full trajectory from $x_0$ to $x_1$, with LPIPS~\citep{Zhang:2018LPIPS}. 

To ablate the effect of different training objectives and focus solely on parameterization, we report additional quantitative results in Tab.~\ref{tbl:lpips_comparison}, where Bespoke* retains Bespoke Solver's parameterization but adopts the same training objective as ours. While Bespoke* improves over the original Bespoke Solver, \ours{} remains superior, with especially large gains at low NFEs, such as +17.7 FID improvement at NFE=4 with RK1, underscoring the advantage of our B\'ezier-based continuous parameterization.

\section{Implementation Details}
\label{sec:implementation_details}

We first describe the shared experimental setup for the methods based on teacher-forcing framework (Bespoke solver, LD3, and~\ours{}) and method-specific configurations for \ours{} and the baselines.

\subsection{Shared Setup}

\paragraph{Teacher Data Generation.} 
We generate teacher samples using the high-order adaptive solver RK45~\citep{BUTCHER:1996RK}, except for Stable Diffusion v3.5~\citep{Esser:2024SD3}, where we adopt RK2 with 30 NFEs. 
The same teacher samples are used for all baselines that rely on the teacher-forcing framework (e.g., Bespoke Solver, LD3).

\paragraph{Training.} 
We train for 8 epochs on CIFAR-10~\citep{Krizhevsky:2009CIFAR}, FFHQ~\citep{Karras:2019FFHQ}, and AFHQv2~\citep{Choi:2020AFHQv2}, and for 5 epochs on ImageNet~\citep{Deng2009ImageNet} and Stable Diffusion v3.5~\citep{Esser:2024SD3}. At the end of each epoch, we perform validation and select the checkpoint with the best validation score for final evaluation. We use LPIPS~\citep{Zhang:2018LPIPS} as the distance metric for LD3~\citep{Tong:2025LD3} and \ours{}, and RMSE for Bespoke Solver~\citep{Shaul:2023Bespoke}.

\paragraph{Evaluation.} 
We report Fréchet Inception Distance (FID)~\citep{Heusel:2017FID} scores computed against the reference set using 50K randomly generated samples. 
On ImageNet, generated samples are drawn to match the class distribution of the reference set. 
For SD3.5, both reference and generated samples are constructed from disjoint subsets of 30K text prompts from the MS-COCO validation set, following the setup of LD3~\citep{Tong:2025LD3}.

\subsection{BézierFlow Training Details}

\paragraph{Target Timesteps.} 
For diffusion models, the timesteps $\{s_i\}_{i=0}^{\mathrm{NFE}}$ are initialized to be uniformly spaced in terms of the signal-to-noise ratio (SNR). 
For flow models, they are initialized to be uniformly spaced in the time domain.

\paragraph{Initialization.} 
We initialize the Bézier scheduler with a linear SI scheduler, i.e., $\bar\alpha(s)=s$ and $\bar\sigma(s)=1-s$. 
Under the 1D Bézier parameterization, this corresponds to
\begin{equation}
    \theta^{(\alpha)}_i = 1,\quad \theta^{(\sigma)}_i = 1,\qquad i=0,1,\dots,n,
\end{equation}
which places the $n-1$ interior control points uniformly between the two endpoints. 
We use 32 control points in all experiments. 
For decoupled timesteps $s_i^c$ that are fed into the model, we set $s_i^c = s_i + \theta_i^{(c)}$, 
with $\theta_i^{(c)}$ initialized to zero, following LD3~\citep{Tong:2025LD3}.

\paragraph{Optimizer.}
We optimize the Bézier scheduler parameters $\theta^{(\alpha)}, \theta^{(\sigma)}$ using RMSprop, and the decoupled timesteps $\theta_i^{(c)}$ using SGD. 
For RMSprop, we set the momentum to $0.9$ and weight decay to $0$. 
The learning rate is $5\times 10^{-3}$ for CIFAR-10, FFHQ, and AFHQv2, and $1\times 10^{-3}$ for ImageNet and Stable Diffusion v3.5. 
For the decoupled timesteps, we use SGD with a learning rate of $3\times 10^{-2}$ for CIFAR-10, $1\times 10^{-1}$ for ImageNet, $1\times 10^{-2}$ for FFHQ and AFHQv2, and $5\times 10^{-4}$ for Stable Diffusion v3.5, each further scaled by $1/\text{NFE}$. 
We apply gradient clipping with a global norm threshold of $1.0$ to all parameters.

\subsection{Baselines}

\paragraph{GITS~\citep{Chen:2024GITS}.}
We adopt the official implementation code and follow the default number of sampling trajectories, which is 256.

\paragraph{Bespoke Solver~\citep{Shaul:2023Bespoke}.}
Since no official implementation code is publicly available, we re-implemented the method based on the descriptions in the original paper. 
We employ Adam optimizer with a learning rate of $1\times 10^{-4}$, as we observed that the learning rate reported in the paper ($2\times 10^{-3}$) caused divergence and very high FID scores when training on relatively small datasets.

\paragraph{LD3~\citep{Tong:2025LD3}.} 
We adopt the official implementation code and follow the default training configurations. 
For timestep parameters, we use the same optimizer and match their learning rate to that of our scheduler. 
For the decoupled timesteps, we follow the original parameterization and use SGD with a learning rate of $\frac{0.1}{\text{NFE}}$ except for ImageNet and Stable Diffusion v3.5, where we use $\frac{0.001}{\text{NFE}}$ following original paper.

\section{More Quantitative Results}
\label{sec:more_quanti}

\begin{table}[!t]
\centering
\scriptsize
\caption{\textbf{FID comparison of few-step generation with diffusion models at extremely low NFEs.}
Results of the base ODE solvers are reported on each top rows. \textbf{Bold} indicates the best results, and \underline{underline} marks the second best. Gray cells indicate the base ODE solvers.}
\setlength{\tabcolsep}{2pt}
\begin{tabularx}{\textwidth}{
  >{\raggedright\arraybackslash}p{0.1\textwidth}|
  *{3}{>{\centering\arraybackslash}X}|
  *{3}{>{\centering\arraybackslash}X}|
  *{3}{>{\centering\arraybackslash}X}}
\toprule
& \multicolumn{3}{c|}{CIFAR-10 $32\times 32$ with EDM} 
& \multicolumn{3}{c|}{FFHQ $64\times 64$ with EDM} 
& \multicolumn{3}{c}{AFHQv2 $64\times 64$ with EDM} \\
\cmidrule(lr){2-4}\cmidrule(lr){5-7}\cmidrule(lr){8-10}
Method &
NFE=1 & NFE=2 & NFE=3 &
NFE=1 & NFE=2 & NFE=3 & 
NFE=1 & NFE=2 & NFE=3 \\
\midrule
\cellcolor{gray!20}
UniPC
& \underline{377.15} & 168.35 & 57.45  
& \underline{280.61} & \underline{104.57} & 59.54  
& \underline{312.37} & \underline{64.62}     &  44.52    
\\
+ DMN
& - & \underline{160.65} & 66.03
& - & 142.57 & 64.99
& - & 141.99     & 70.01
\\
+ GITS
& - & 168.29 & \underline{53.21}
& - & 107.71 & \underline{42.38}
& - & 73.95     & \textbf{25.13}
\\
+ LD3
& - & 187.42  & \textbf{39.56}
& - & 120.87 & 48.30
& - & 107.77     & 30.53
\\
+ \ours{}
& \textbf{125.03} & \textbf{50.41}  & 55.07
& \textbf{121.94} & \textbf{72.03} & \textbf{33.72}
& \textbf{159.58} & \textbf{39.86}  & \underline{26.31} 
\\
\midrule
\cellcolor{gray!20}
iPNDM
& \underline{377.15} & 153.31 & 47.68
& \underline{280.61} & 102.50 & 45.70
& \underline{312.37} & 79.32     & 38.16
\\
+ DMN
& - & 146.40 & 58.98
& - & 112.55 & 61.54
& - & 128.36  & 76.28
\\
+ GITS
& - & 153.33 & 43.71
& - & 105.78 & \textbf{32.33}
& - & 95.47    & \underline{26.40}
\\
+ LD3
& - & \underline{145.03} & \underline{32.19}
& - & \underline{97.62}  & 38.14
& - & \underline{91.10}   & \textbf{23.85}
\\
+ \ours{}
& \textbf{125.03} & \textbf{41.58} & \textbf{22.20}
& \textbf{121.94} & \textbf{60.45} & \underline{35.10}
& \textbf{159.58} & \textbf{34.70}  & 36.26
\\
\bottomrule
\end{tabularx}
\label{tbl:verylow_fid_diffusion}
\end{table}

\begin{table}[!t]
\centering
\scriptsize
\caption{\textbf{FID comparison of few-step generation with flow models at extremely low NFEs.}
Results of the base ODE solvers are reported on each top rows. \textbf{Bold} indicates the best results, and \underline{underline} marks the second best. Gray cells indicate the base ODE solvers.}
\setlength{\tabcolsep}{2pt}
\begin{tabularx}{\textwidth}{
  >{\raggedright\arraybackslash}p{0.1\textwidth}|
  *{3}{>{\centering\arraybackslash}X}|
  *{3}{>{\centering\arraybackslash}X}|
  *{3}{>{\centering\arraybackslash}X}}
\toprule
& \multicolumn{3}{c|}{CIFAR-10 $32\times 32$ with ReFlow} 
& \multicolumn{3}{c|}{ImageNet $256\times 256$ with FlowDCN} 
& \multicolumn{3}{c}{MS-COCO $512\times 512$ with SDv3.5} \\
\cmidrule(lr){2-4}\cmidrule(lr){5-7}\cmidrule(lr){8-10}
Method &
NFE=1 & NFE=2 & NFE=3 &
NFE=1 & NFE=2 & NFE=3 & 
NFE=1 & NFE=2 & NFE=3 \\
\midrule
\cellcolor{gray!20}
RK1
& \underline{379.22} & 171.48 & 89.00  
& \underline{263.54} & \underline{113.27} & \textbf{27.69}  
& 328.02 & 214.03  &  103.97    
\\
+ DMN
& - & \underline{170.53} & \underline{79.54}
& - & 115.89 & 42.78
& - & 218.20  & \textbf{82.68}
\\
+ GITS
& - & 183.40 & 81.46
& - & 130.81 & 31.70
& - & 166.06  & 94.94
\\
+ Bespoke
& 471.18 & 405.94 & 265.77
& 264.81 & 114.66 & \underline{28.27}
& \underline{324.94} & 212.16 & 98.91
\\
+ LD3
& - & 182.40  & 81.35
& - & 126.29 & 85.16
& - & \textbf{150.21}  & 85.16
\\
+ \ours{}
& \textbf{314.52} & \textbf{67.63}  & \textbf{30.40}
& \textbf{261.79} & \textbf{94.64} & 44.60
& \textbf{320.67} & \underline{156.20}  & \underline{83.55} 
\\
\midrule
\cellcolor{gray!20}
RK2
& - & \underline{128.80} & -
& - & 90.26 & -
& - & 163.35 & -
\\
+ DMN
& -& - & -
& -& - & -
& -& -  & -
\\
+ GITS
& -& - & -
& -& - & -
& -& -  & -
\\
+ Bespoke
& -& 309.60 & -
& -& \underline{86.58} & -
& -& \underline{162.40}    & -
\\
+ LD3
& -& - & -
& -& - & -
& -& -  & -
\\
+ \ours{}
& -& \textbf{70.87} & -
& -& \textbf{83.97} & -
& -& \textbf{146.20}  & -
\\
\bottomrule
\end{tabularx}
\label{tbl:verylow_fid_flow}
\end{table}
\subsection{Probing BézierFlow at Extremely Low NFEs}
\label{subsec:low_nfe}
To stress-test BézierFlow in the extreme low-NFE regime and identify where quality collapse begins, we conduct additional experiments in the very low-NFE range ($\text{NFE}\le3$), which is even lower than the NFEs used in Sec.~\ref{sec:experiment}. Except for the NFEs, all other experiment setups follow those used in Sec.~\ref{sec:experiment}.

As summarized in Tab.~\ref{tbl:verylow_fid_diffusion} and Tab.~\ref{tbl:verylow_fid_flow}, BézierFlow remains effective even at extremely low NFEs for both diffusion and flow models, improving over the base solvers by a substantial margin. Note that blank entries for RK2 simply reflect that RK2 only supports even numbers of function evaluations and has no timestep to learn in NFE=2. Furthermore, for NFE=1, timestep-learning methods cannot be applied, whereas scheduler-learning approaches such as Bespoke Solver~\citep{Shaul:2023Bespoke} and \ours{} remain applicable.

\begin{table}[!t]
\centering
\scriptsize
\caption{\textbf{Quantitative comparison of few-step generation on text--image alignment with Stable Diffusion~\citep{Esser:2024SD3}.}
Results for the base solvers are reported on each top rows. \textbf{Bold} indicates the best results, and \underline{underline} marks the second best. Gray cells indicate the
base ODE solvers.}
\setlength{\tabcolsep}{4pt}

\begin{tabularx}{\textwidth}
  {>{\raggedright\arraybackslash}p{0.12\textwidth}
   |*{8}{>{\centering\arraybackslash}X}}
\toprule
Method 
  & \multicolumn{2}{c}{NFE=4}
  & \multicolumn{2}{c}{NFE=6}
  & \multicolumn{2}{c}{NFE=8}
  & \multicolumn{2}{c}{NFE=10} \\
 & CLIP $\uparrow$ & PickScore $\uparrow$
 & CLIP $\uparrow$ & PickScore $\uparrow$
 & CLIP $\uparrow$ & PickScore $\uparrow$
 & CLIP $\uparrow$ & PickScore $\uparrow$ \\
\midrule
\multicolumn{9}{c}{MS-COCO $512\times 512$ with Stable Diffusion~\citep{Esser:2024SD3}} \\
\midrule
\cellcolor{gray!20}RK1 
  & 0.240 & 0.206 
  & \underline{0.252} & \underline{0.212} 
  & \underline{0.257} & \underline{0.215} 
  & \textbf{0.260} & \textbf{0.217} \\
+ DMN 
  & 0.225 & 0.199 
  & 0.246 & 0.209 
  & 0.253 & 0.213 
  & 0.256 & 0.215 \\
+ Bespoke 
  & 0.241 & 0.206 
  & 0.243 & \underline{0.212} 
  & 0.251 & 0.214 
  & 0.252 & \underline{0.216} \\
+ GITS 
  & 0.234 & 0.204 
  & 0.247 & 0.210 
  & 0.252 & 0.213 
  & 0.255 & 0.214 \\
+ LD3 
  & \underline{0.244} & \underline{0.208} 
  & 0.249 & \underline{0.212} 
  & \textbf{0.258} & \textbf{0.217} 
  & \underline{0.258} & \textbf{0.217} \\
+ \ours{}
  & \textbf{0.245} & \textbf{0.209} 
  & \textbf{0.253} & \textbf{0.214} 
  & 0.256 & \textbf{0.217} 
  & \underline{0.258} & \textbf{0.217} \\
\midrule
\cellcolor{gray!20}RK2 
  & 0.244 & 0.208 
  & 0.255 & 0.214 
  & \underline{0.259} & 0.216 
  & 0.260 & 0.217 \\
+ DMN 
  & 0.243 & 0.208 
  & \underline{0.257} & \textbf{0.216} 
  & 0.252 & 0.213 
  & 0.259 & 0.217 \\
+ Bespoke 
  & 0.244 & 0.208 
  & 0.225 & 0.200 
  & 0.253 & 0.215 
  & 0.257 & 0.217 \\
+ GITS 
  & \textbf{0.251} & \textbf{0.211} 
  & 0.255 & 0.214 
  & 0.257 & 0.216 
  & 0.258 & 0.216 \\
+ LD3 
  & 0.241 & 0.208 
  & 0.255 & \underline{0.215} 
  & \textbf{0.260} & \textbf{0.218} 
  & \underline{0.261} & \underline{0.218} \\
+ \ours{}
  & \underline{0.248} & \underline{0.210} 
  & \textbf{0.258} & \underline{0.215} 
  & \textbf{0.260} & \underline{0.217} 
  & \textbf{0.263} & \textbf{0.219} \\
\bottomrule
\end{tabularx}

\label{tbl:sd_clip_pick}
\end{table}

\subsection{Text--Image Alignment for Foundational Model}
\label{subsec:sd_clip_pickscore}
To complement the zero-shot MS-COCO FID results of Stable Diffusion v3.5~\citep{Esser:2024SD3} in Tab.~\ref{tbl:main_fid_flow}, we provide additional evaluation results for a more comprehensive assessment. We report CLIP score~\citep{hessel2021clipscore} and PickScore~\citep{Kirstain:2023Pick}, both of which measure the alignment between the given text prompt and the generated image.

As shown in Tab.~\ref{tbl:sd_clip_pick}, BézierFlow achieves the best or second-best performance across various NFEs, solvers, and evaluation metrics except for the CLIP Score at NFE=8 with the RK1 solver. These additional results further corroborate the superiority of BézierFlow even with the large-scale 2.5B pretrained stochastic interpolant model~\citep{Esser:2024SD3}.

\begin{table}[!t]
\centering
\scriptsize
\caption{\textbf{Quantitative comparison on training efficiency in few-step generation for diffusion and flow models on CIFAR-10.} All experiments are conducted on A6000 GPUs, except for the last row of distillation methods, which reports the performance of pretrained model from their official implementations~\citep{Song:2023CM, Liu:2023RF}. ``Time'' denotes wall-clock training time, where s/m/d denote seconds/minutes/days, respectively.}
\setlength{\tabcolsep}{4pt}

\begin{tabularx}{\textwidth}
  {>{\raggedright\arraybackslash}p{0.1\textwidth}
   |*{4}{>{\centering\arraybackslash}X}
   |>{\raggedright\arraybackslash}p{0.1\textwidth}
   |*{4}{>{\centering\arraybackslash}X}}
\toprule
Method 
  & \multicolumn{2}{c}{NFE=6}
  & \multicolumn{2}{c|}{NFE=8}
& Method 
  & \multicolumn{2}{c}{NFE=6}
  & \multicolumn{2}{c}{NFE=8} \\
 & FID $\downarrow$ & Time $\downarrow$
 & FID $\downarrow$ & Time $\downarrow$
 &
 & FID $\downarrow$ & Time $\downarrow$
 & FID $\downarrow$ & Time $\downarrow$ \\
\midrule
\multicolumn{10}{c}{(1) Non-distillation Methods}\\
\midrule
\cellcolor{gray!20}iPNDM & \multicolumn{4}{c|}{CIFAR-10 with EDM (Teacher FID: $2.08$)} 
& \cellcolor{gray!20}RK2 & \multicolumn{4}{c}{CIFAR-10 with ReFlow (Teacher FID: $2.70$)}  \\
\midrule
+ DMN 
  & 9.33 & 5s 
  & 4.82 & 5s 
& + DMN 
  & 51.99 & 5s 
  & 21.43 & 5s \\
+ GITS 
  & 6.80 & 30s 
  & 4.07 & 30s 
& + GITS 
  & 11.84 & 30s 
  & 8.77 & 30s \\
+ Bespoke 
  & - & - 
  & - & - 
& + Bespoke 
  & 64.87 & 30m 
  & 16.67 & 30m \\
+ LD3 
  & 4.42 & 10m 
  & 2.93 & 13m 
& + LD3 
  & 13.82 & 10m 
  & 6.26 & 13m \\
+ \ours{}
  & 3.35 & 10m 
  & 2.81 & 13m 
& + \ours{}
  & 6.00 & 10m 
  & 4.31 & 13m \\
\midrule
\multicolumn{10}{c}{(2) Distillation Methods}\\
\midrule
CD
  & 359.59 & 15m 
  & 343.59 & 15m 
& + 2-RF
  & 12.12 & 15m 
  & 9.17 & 15m \\
CD
  & 4.24 & 6d 
  & 3.95 & 6d 
& + 2-RF
  & 5.69 & 2d 
  & 5.45 & 2d \\
CD
  & 2.82 & 8d (A100) 
  & 2.79 & 8d (A100) 
& + 2-RF
  & 3.74 & 8d (A100)
  & 3.68 & 8d (A100) \\
\bottomrule
\end{tabularx}

\label{tbl:comparison_with_cd_rf}
\end{table}

\subsection{Comparison on Training Efficiency with Few-Step Generation Methods}
\label{subsec:comparison_with_few_step}
For a more comprehensive and fair comparison of training efficiency beyond the Tab.~\ref{tab:cifar10_ours_vs_baselines}, we report additional results at matched NFEs with varying training budgets in Tab.~\ref{tbl:comparison_with_cd_rf}. As shown, under the same NFEs, distillation-based approaches (Consistency Distillation (CD)~\citep{Song:2023CM} and 2-Rectified Flow (2-RF)~\citep{Liu:2023RF}) yield notably worse FID under the same lightweight training budget (15 minutes) and require \emph{substantially longer} training time (2-6 days) to achieve FID comparable to BézierFlow, corresponding to roughly \textbf{200-600$\times$ more training time}. These results underscore BézierFlow’s highly training-efficient acceleration, achieving in just a few minutes the performance that prior distillation-based approaches require several days of training to reach. Note that the 15-minute performance of 2-RF is identical to that of the base pretrained model as this budget is fully spent on the data creation stage for ReFlow.

We also include training time comparisons against non-distillation baselines that accelerate generation with lightweight training, including DMN, GITS, Bespoke Solver and LD3~\citep{Xue:2024DMN, Chen:2024GITS, Shaul:2023Bespoke, Tong:2025LD3}. Among these lightweight acceleration methods, BézierFlow achieves the best FID, even outperforming LD3 under the same training budget. This demonstrates that BézierFlow offers a more favorable trade-off between training efficiency and sample quality.

\begin{figure}[t!]
\centering
{\scriptsize
\setlength{\tabcolsep}{4pt} 
\begin{minipage}[t]{0.49\textwidth}
  \centering
  \begin{tabularx}{\linewidth}{@{} >{\centering\arraybackslash}m{1.3em} @{\hspace{6pt}}| *{5}{>{\centering\arraybackslash}m{4.18em}} @{}}
    \toprule
    NFE & \hspace{0.3em}UniPC & DMN & \hspace{-0.3em}GITS & \hspace{-0.7em}LD3 & \hspace{-1.8em}BF \\
    \midrule
    \multicolumn{6}{@{}c@{}}{FFHQ $64\times64$ with EDM~\citep{Karras:2022EDM}}\\
    \midrule
    \multirow{2}{*}{6}
    & \multicolumn{5}{@{}c@{}}{%
         \includegraphics[height=1.12cm]{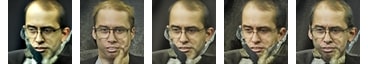}%
       }\\
    & \multicolumn{5}{@{}c@{}}{%
         \includegraphics[height=1.12cm]{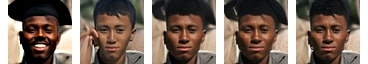}%
       }\\
    \midrule
    \multirow{2}{*}{8}
    & \multicolumn{5}{@{}c@{}}{%
         \includegraphics[height=1.12cm]{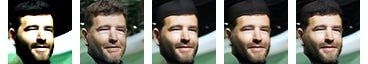}
       }\\
    & \multicolumn{5}{@{}c@{}}{%
         \includegraphics[height=1.12cm]{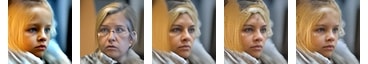}
       }\\
    \bottomrule
  \end{tabularx}
\end{minipage}\hfill
\begin{minipage}[t]{0.49\textwidth}
  \centering
  \begin{tabularx}{\linewidth}{@{} >{\centering\arraybackslash}m{1.3em} @{\hspace{6pt}}| *{5}{>{\centering\arraybackslash}m{4.18em}} @{}}
    \toprule
    NFE & \hspace{0.3em}UniPC & DMN & \hspace{-0.3em}GITS & \hspace{-0.7em}LD3 & \hspace{-1.8em}BF \\
    \midrule
    \multicolumn{6}{@{}c@{}}{AFHQv2 $64\times64$ with EDM~\citep{Karras:2022EDM}}\\
    \midrule
    \multirow{2}{*}{6}
    & \multicolumn{5}{@{}c@{}}{%
         \includegraphics[height=1.12cm]{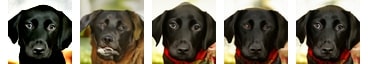}
       }\\
    & \multicolumn{5}{@{}c@{}}{%
         \includegraphics[height=1.12cm]{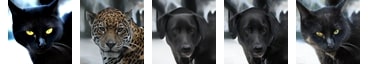}
       }\\
    \midrule
    \multirow{2}{*}{8}
    & \multicolumn{5}{@{}c@{}}{%
         \includegraphics[height=1.12cm]{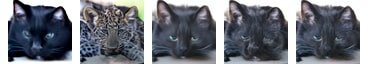}
       }\\
    & \multicolumn{5}{@{}c@{}}{%
         \includegraphics[height=1.12cm]{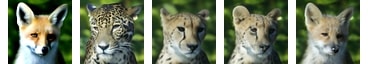}
       }\\
    \bottomrule
  \end{tabularx}
\end{minipage}
}
\caption{\textbf{Qualitative comparisons of samples generated using NFEs 6 and 8 on FFHQ and
AFHQv2 datasets.} We use UniPC solver as the base solver for both cases.}
\label{fig:app_qualitative_diffusion}
\end{figure}
\begin{figure}[t!]
\centering
{\scriptsize
\setlength{\tabcolsep}{4pt} 
\begin{minipage}[t]{0.485\textwidth}
  \centering
  \begin{tabularx}{\linewidth}{@{} >{\centering\arraybackslash}m{1.3em} @{\hspace{6pt}}| *{6}{>{\centering\arraybackslash}m{3.3em}} @{}}
    \toprule
    NFE & RK2 & \hspace{-0.2em}DMN & \hspace{-0.65em}GITS & \hspace{-1em}Bespoke & \hspace{-1.3em}LD3 & \hspace{-2.1em}BF \\
    \midrule
    \multicolumn{7}{@{}c@{}}{CIFAR-10 $32\times 32$ with ReFlow~\citep{Liu:2023RF}} \\
    \midrule
    \multirow{2}{*}{6}
    & \multicolumn{6}{@{}c@{}}{%
         \includegraphics[height=1.01cm]{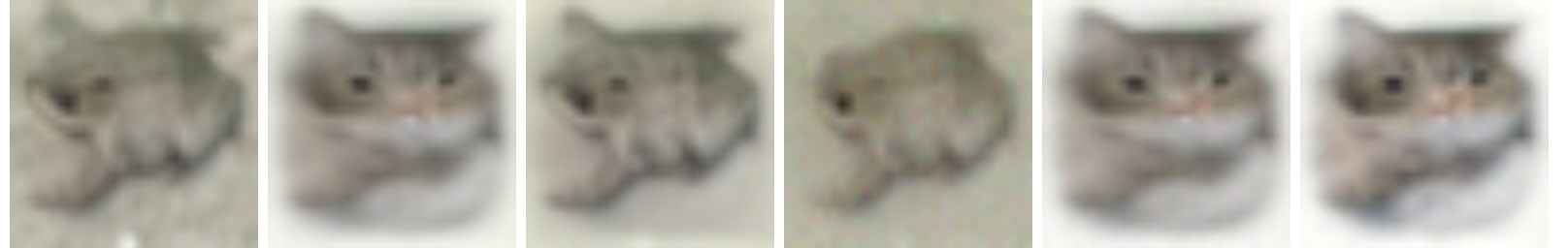}%
       }\\
    & \multicolumn{6}{@{}c@{}}{%
         \includegraphics[height=1.01cm]{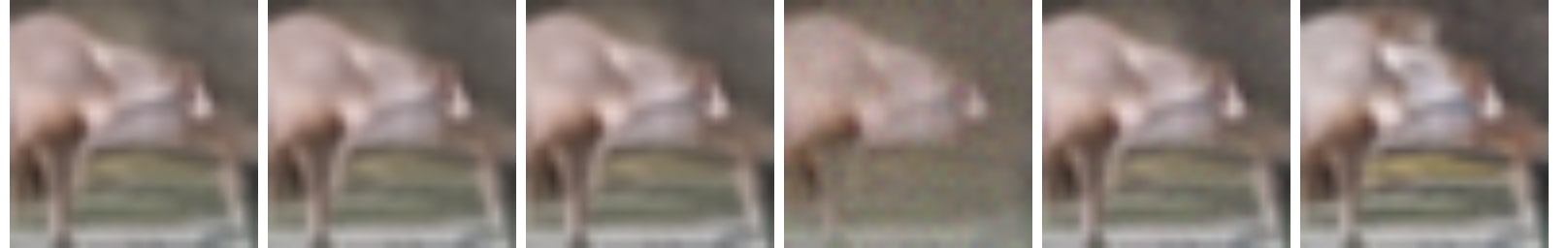}%
       }\\
    \midrule
    \multirow{2}{*}{8}
    & \multicolumn{6}{@{}c@{}}{%
         \includegraphics[height=1.01cm]{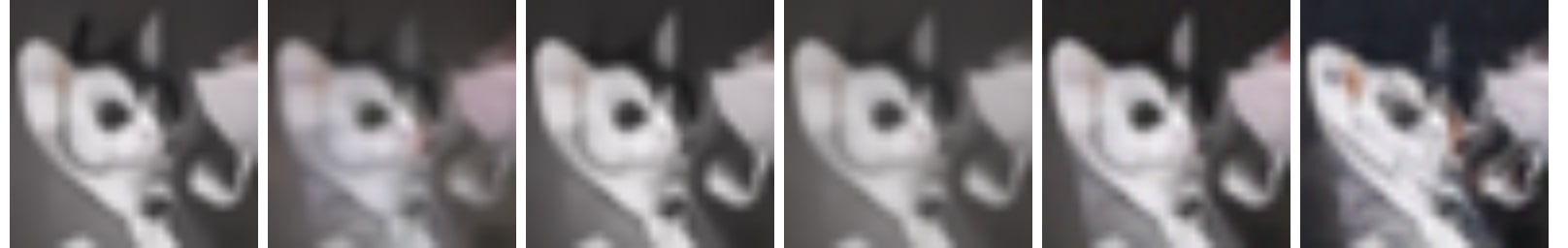}
       }\\
    & \multicolumn{6}{@{}c@{}}{%
         \includegraphics[height=1.01cm]{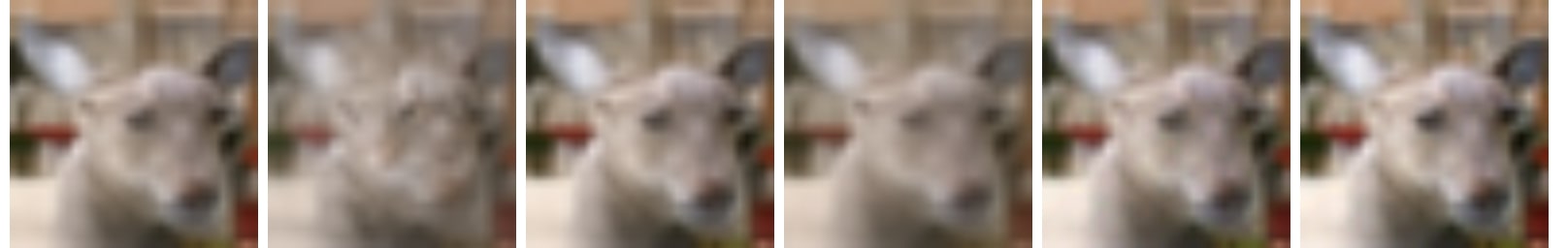}
       }\\
    \bottomrule
  \end{tabularx}
\end{minipage}\hfill
\begin{minipage}[t]{0.485\textwidth}
  \centering
  \begin{tabularx}{\linewidth}{@{} >{\centering\arraybackslash}m{1.3em} @{\hspace{6pt}}| *{6}{>{\centering\arraybackslash}m{3.3em}} @{}}
    \toprule
    NFE & RK2 & \hspace{-0.2em}DMN & \hspace{-0.65em}GITS & \hspace{-1em}Bespoke & \hspace{-1.3em}LD3 & \hspace{-2.1em}BF \\
    \midrule
    \multicolumn{7}{@{}c@{}}{ImageNet $256\times 256$ with FlowDCN~\citep{Wang:2024FlowDCN}}\\
    \midrule
    \multirow{2}{*}{6}
    & \multicolumn{6}{@{}c@{}}{%
         \includegraphics[height=1.01cm]{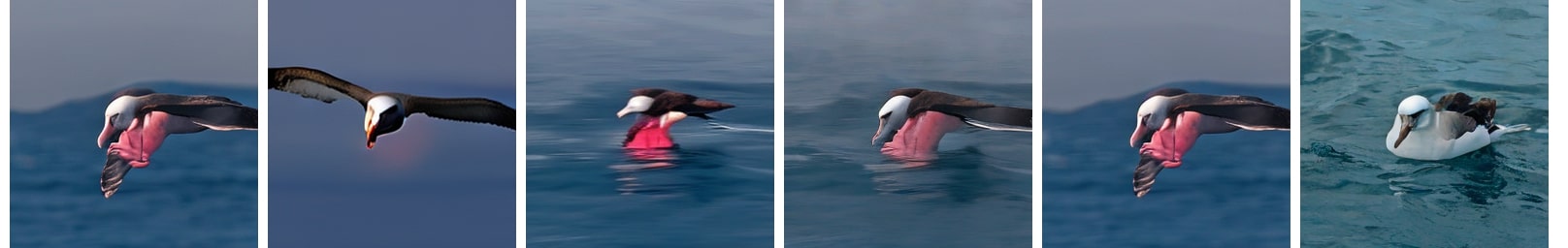}
       }\\
    & \multicolumn{6}{@{}c@{}}{%
         \includegraphics[height=1.01cm]{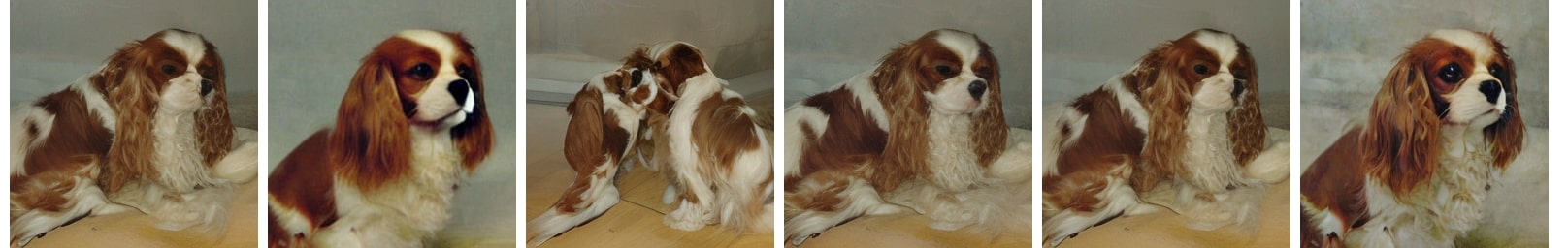}
       }\\
    \midrule
    \multirow{2}{*}{8}
     & \multicolumn{6}{@{}c@{}}{%
         \includegraphics[height=1.01cm]{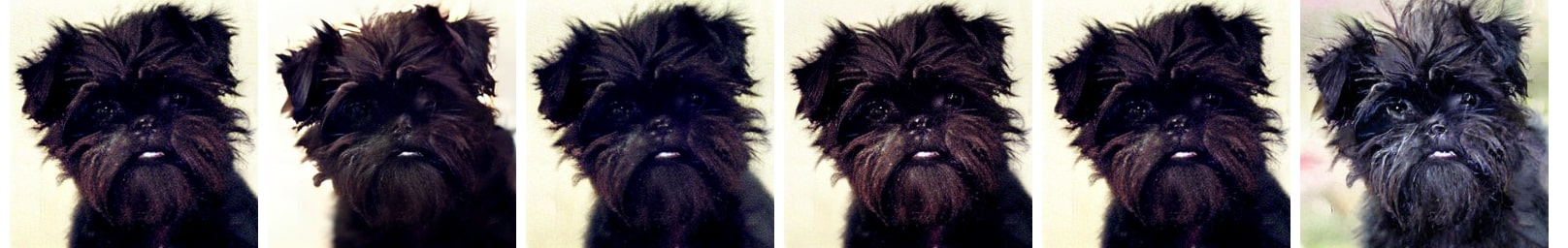}
       }\\
     & \multicolumn{6}{@{}c@{}}{%
         \includegraphics[height=1.01cm]{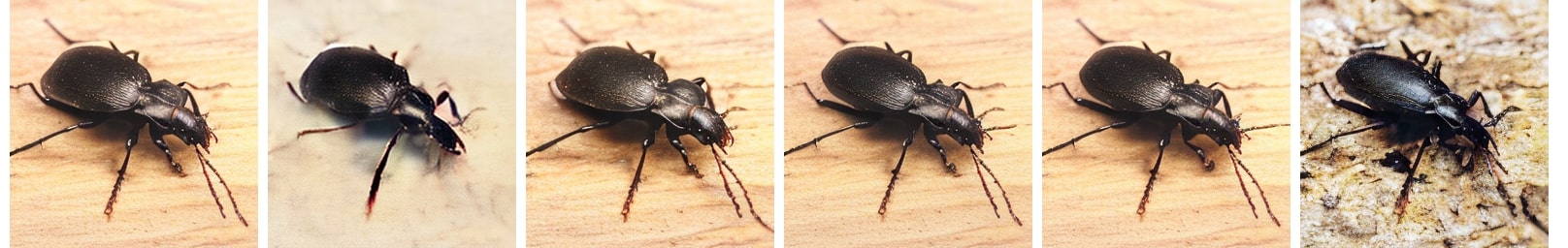}
       }\\
    \bottomrule
  \end{tabularx}
\end{minipage}
}
\caption{\textbf{Qualitative comparisons of samples generated using NFEs 6 and 8 on CIFAR-10 and ImageNet datasets.} We use RK2 solver as the base solver for both cases.}
\vspace{-\baselineskip}
\label{fig:app_qualitative_flow}
\end{figure}

\begin{figure}[t!]
\centering
{\scriptsize
\setlength{\tabcolsep}{1pt}
\begin{tabularx}{0.953\textwidth}{@{}
  >{\centering\arraybackslash}m{3em}
  *{6}{>{\centering\arraybackslash}X}
  >{\centering\arraybackslash}X
@{}}
  \toprule
  \multicolumn{1}{c|}{NFE} & RK1 & DMN & GITS & Bespoke & LD3 & BézierFlow \\
  \midrule
    \multicolumn{7}{@{}c@{}}{MS-COCO $512\times 512$ with Stable Diffusion~\citep{Esser:2024SD3}} \\
\midrule
  \multirow{2}{*}{6}
  & \includegraphics[width=\linewidth]{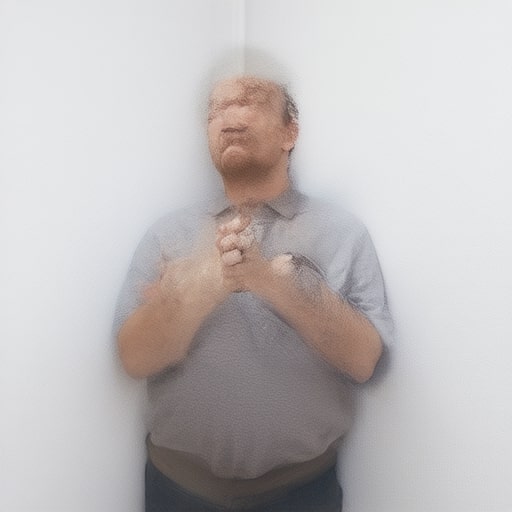}
  & \includegraphics[width=\linewidth]{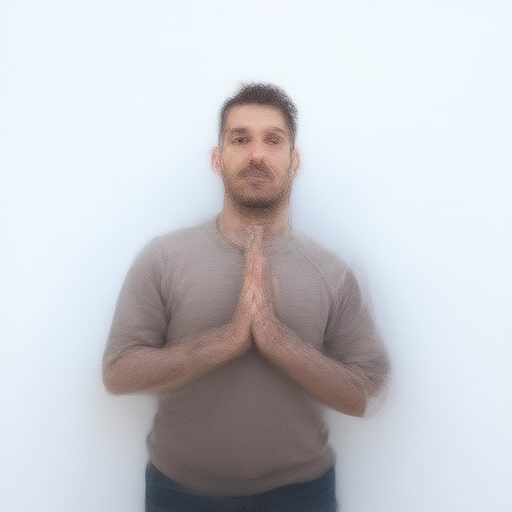}
  & \includegraphics[width=\linewidth]{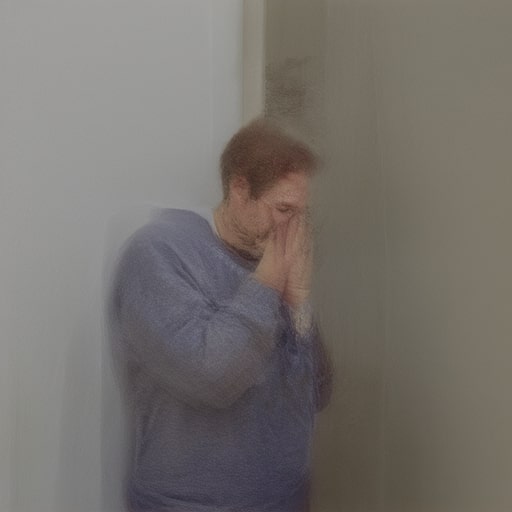}
  & \includegraphics[width=\linewidth]{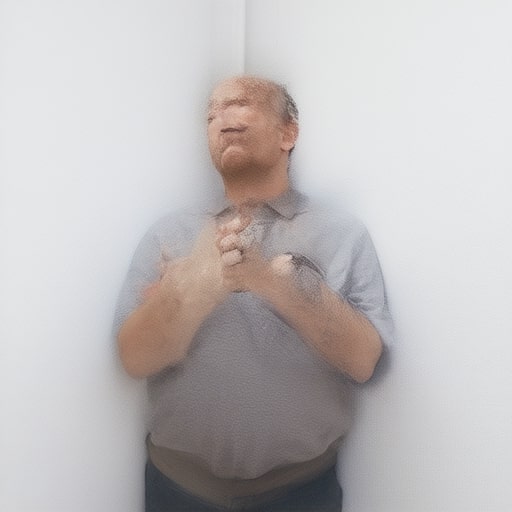}
  & \includegraphics[width=\linewidth]{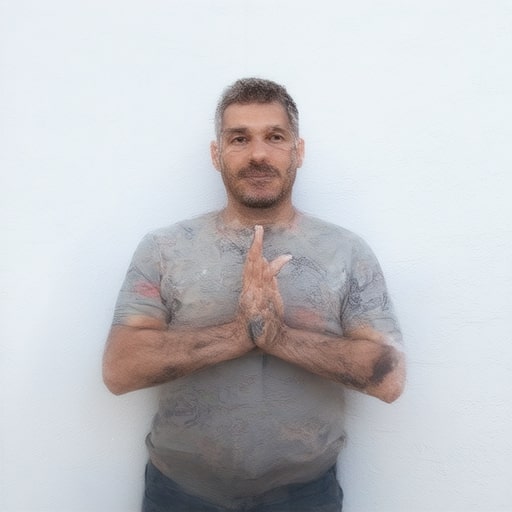}
  & \includegraphics[width=\linewidth]{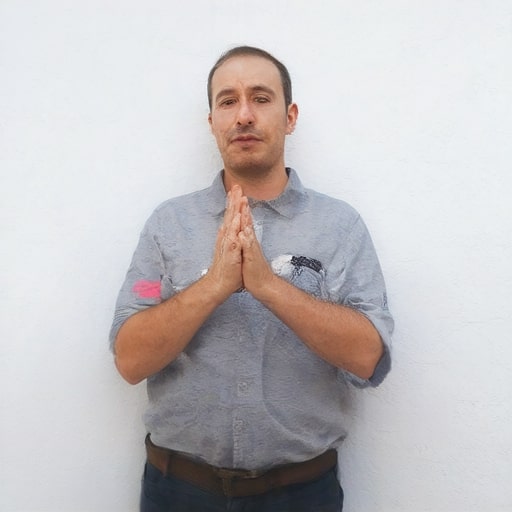}
  \\[0.01em]
  & \multicolumn{6}{c}{\scriptsize{\textit{``A man standing up against a wall with his hands clasped together.''}}} \\[0.02em]
  &
  \includegraphics[width=\linewidth]{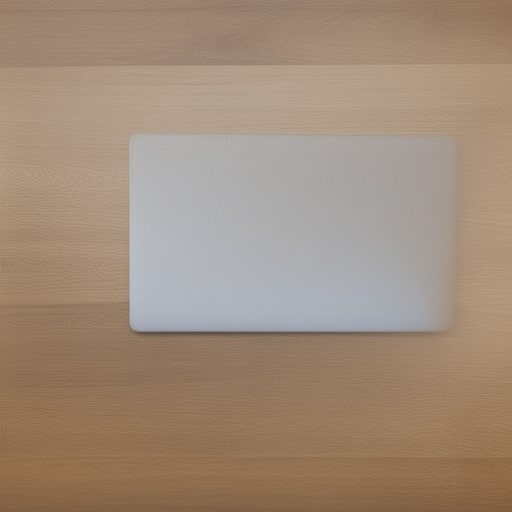}
  & \includegraphics[width=\linewidth]{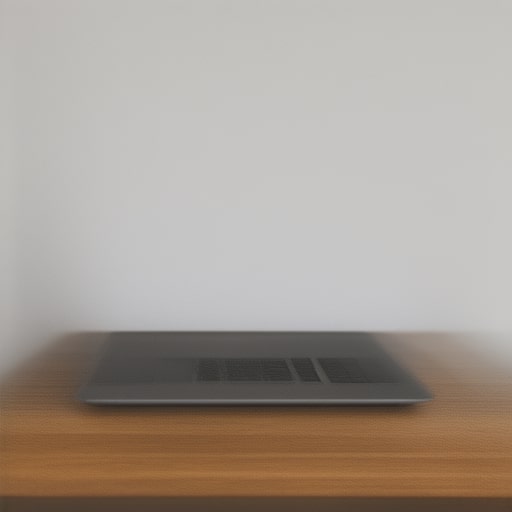}
  & \includegraphics[width=\linewidth]{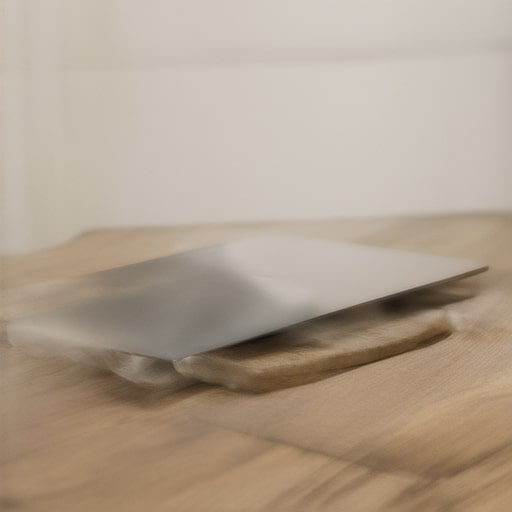}
  & \includegraphics[width=\linewidth]{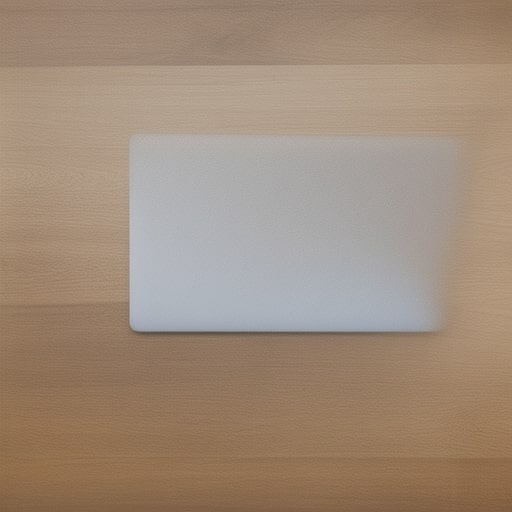}
  & \includegraphics[width=\linewidth]{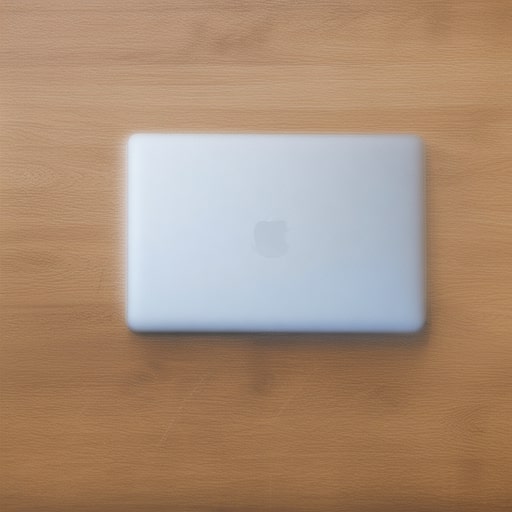}
  & \includegraphics[width=\linewidth]{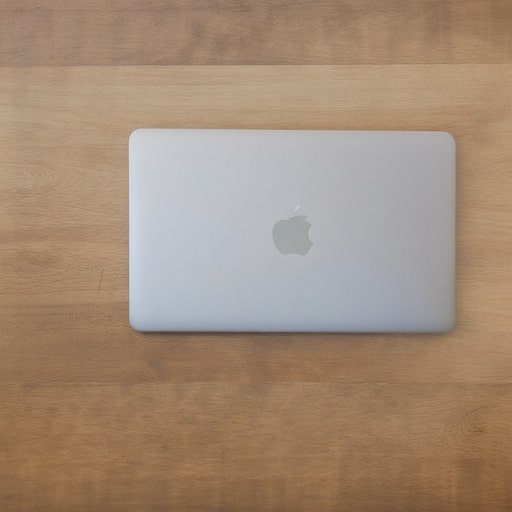}
  \\[0.01em]
  & \multicolumn{6}{c}{\scriptsize{\textit{``A laptop computer sitting on top of a wooden table.''}}} \\[0.01em]
  \midrule
  \multirow{2}{*}{8}
  & \includegraphics[width=\linewidth]{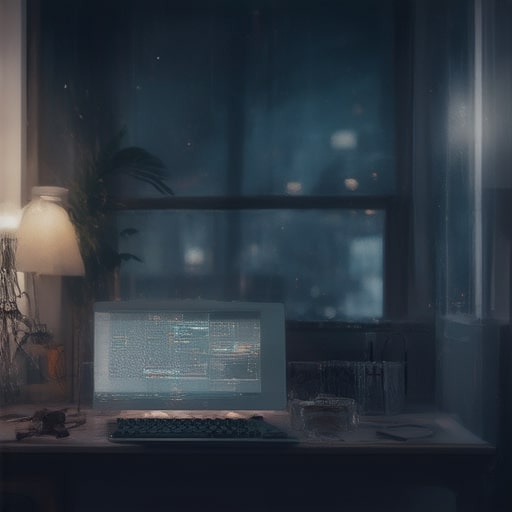}
  & \includegraphics[width=\linewidth]{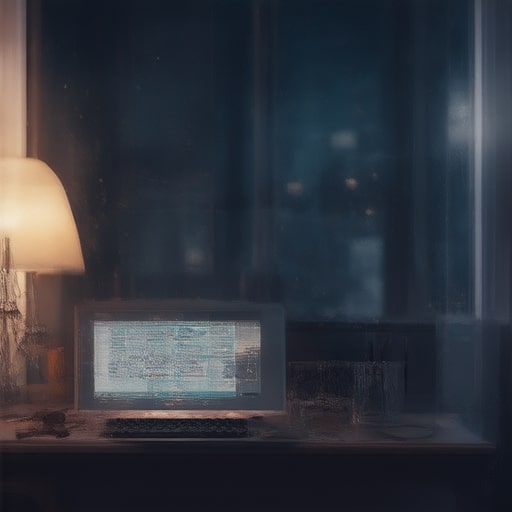}
  & \includegraphics[width=\linewidth]{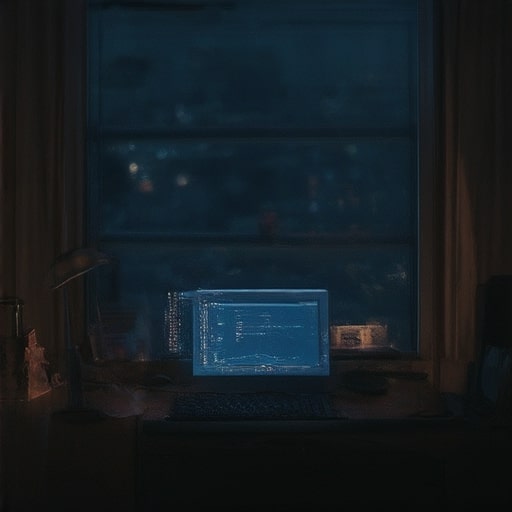}
  & \includegraphics[width=\linewidth]{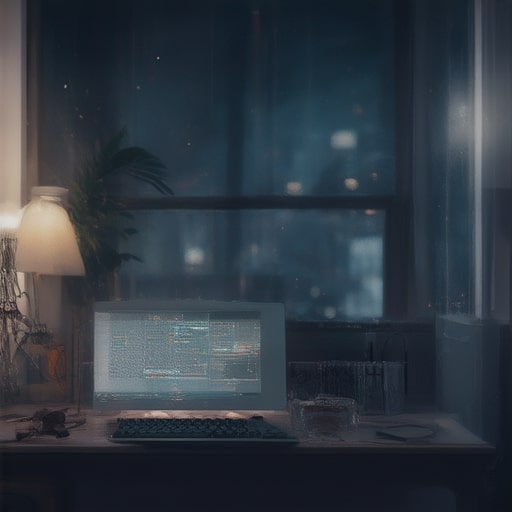}
  & \includegraphics[width=\linewidth]{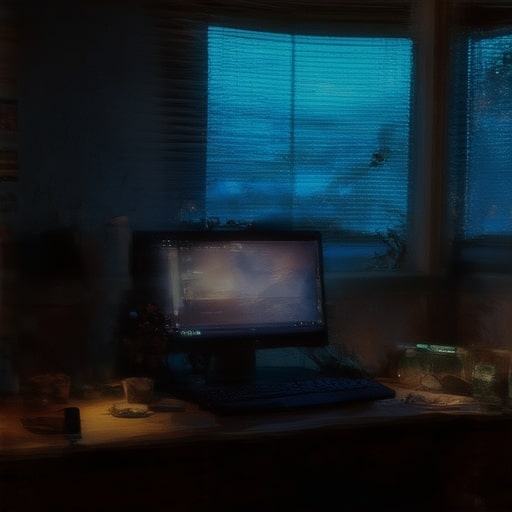}
  & \includegraphics[width=\linewidth]{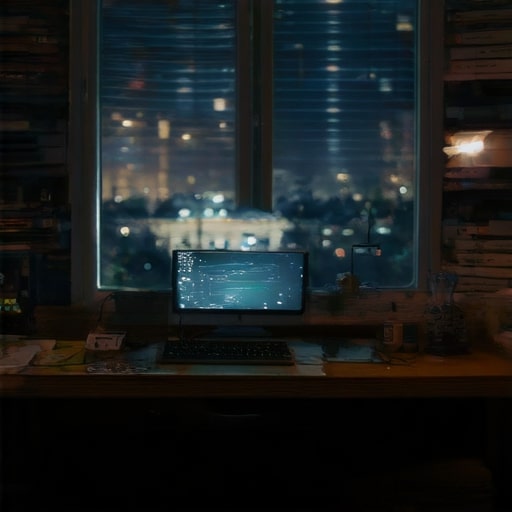}
  \\[0.01em]
  & \multicolumn{6}{c}{\scriptsize{\textit{``Computer on the desk at nighttime in front of a window.''}}} \\[0.02em]
  &
  \includegraphics[width=\linewidth]{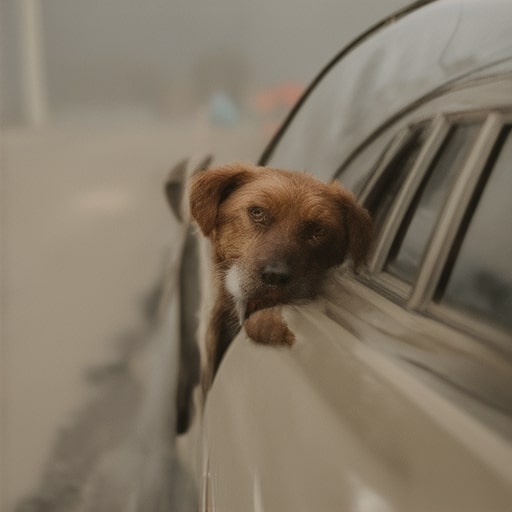}
  & \includegraphics[width=\linewidth]{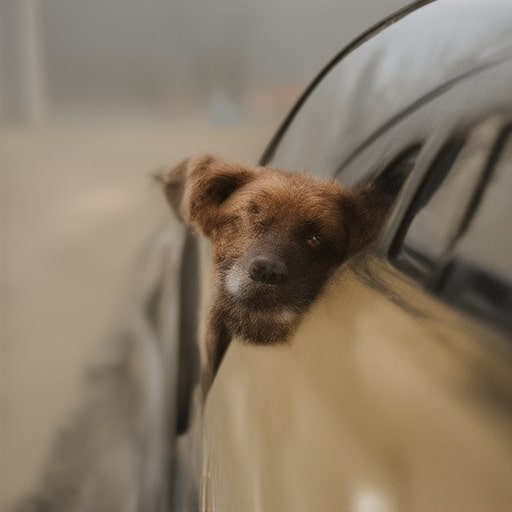}
  & \includegraphics[width=\linewidth]{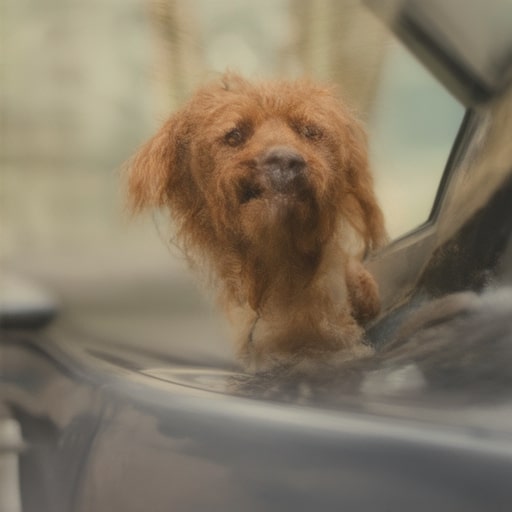}
  & \includegraphics[width=\linewidth]{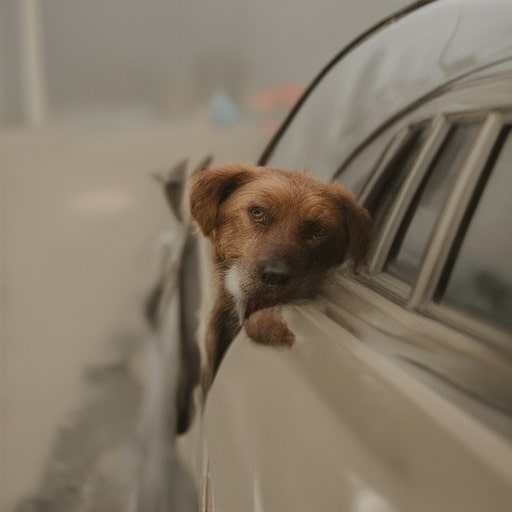}
  & \includegraphics[width=\linewidth]{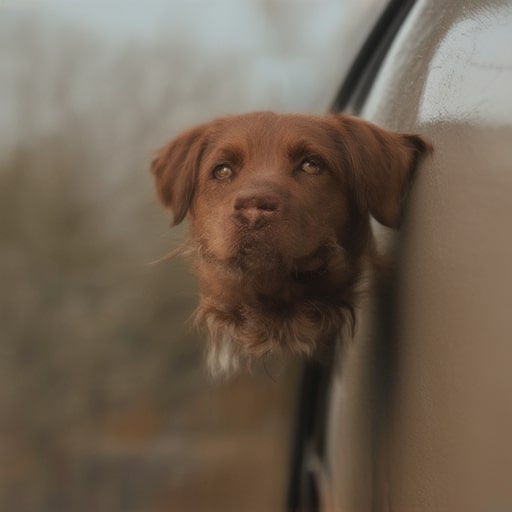}
  & \includegraphics[width=\linewidth]{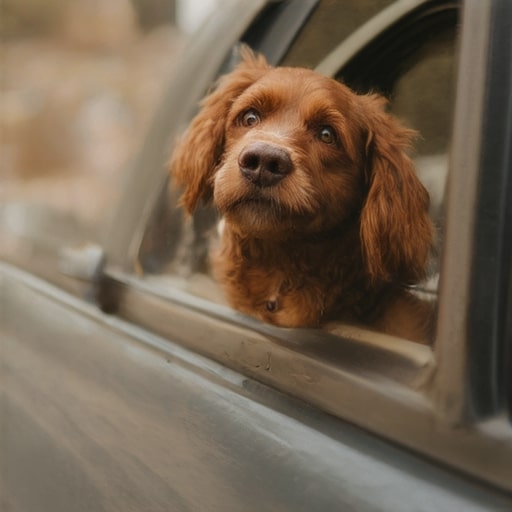}
  \\[0.01em]
  & \multicolumn{6}{c}{\scriptsize{\textit{``A brown dog hanging it's head out of a car window.''}}} \\[0.01em]
  \midrule
  NFE & RK2 & DMN & GITS & Bespoke & LD3 & BézierFlow \\
  \midrule
  \multirow{2}{*}{6}
  & \includegraphics[width=\linewidth]{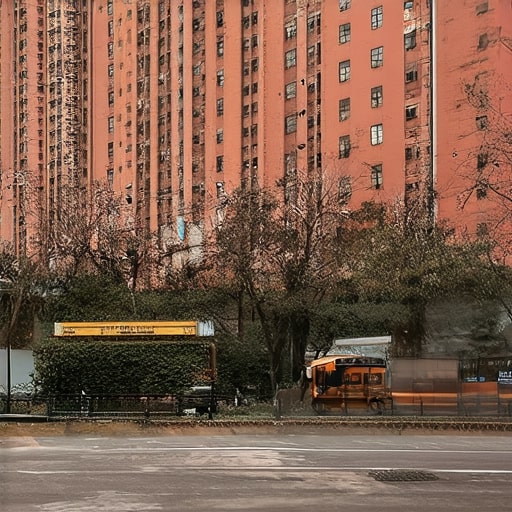}
  & \includegraphics[width=\linewidth]{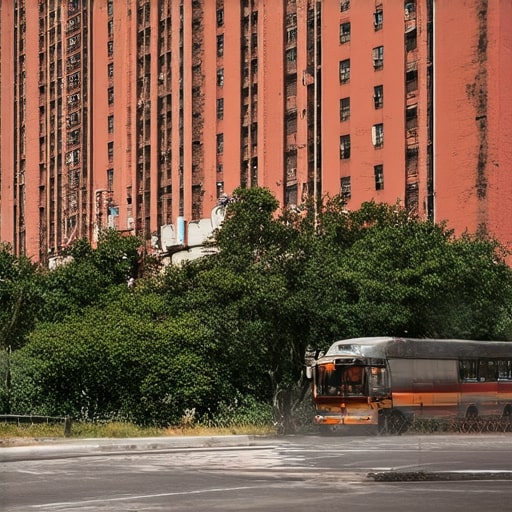}
  & \includegraphics[width=\linewidth]{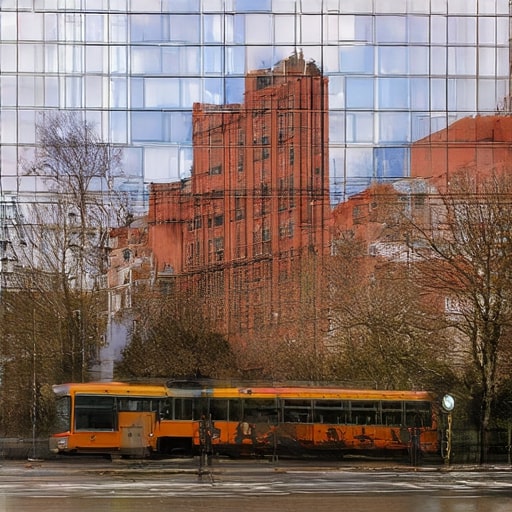}
  & \includegraphics[width=\linewidth]{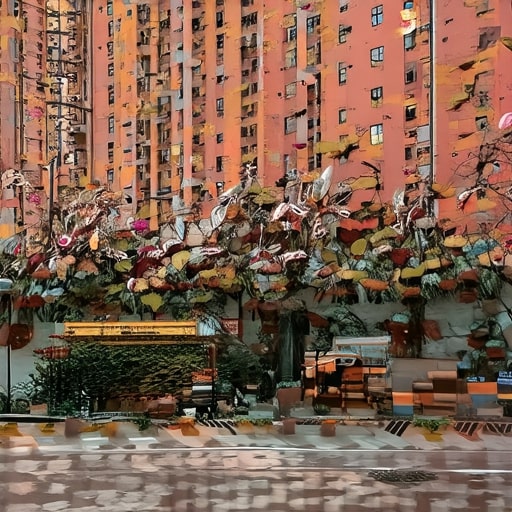}
  & \includegraphics[width=\linewidth]{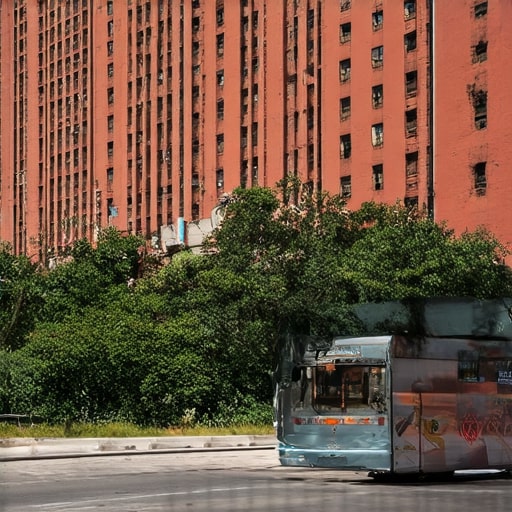}
  & \includegraphics[width=\linewidth]{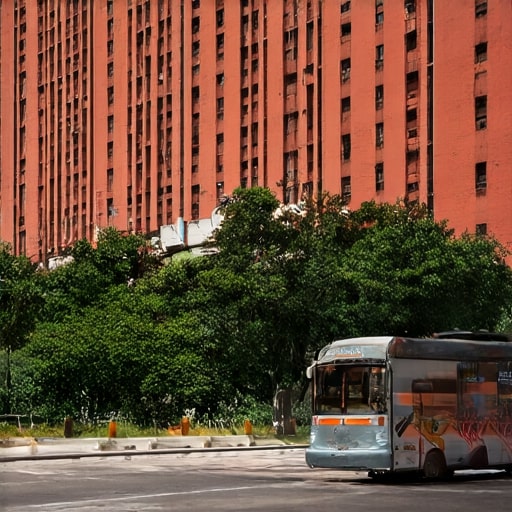}
  \\[0.01em]
  & \multicolumn{6}{c}{\scriptsize{\textit{``A bus stopped in front of a tall red building.''}}} \\[0.02em]
  &
  \includegraphics[width=\linewidth]{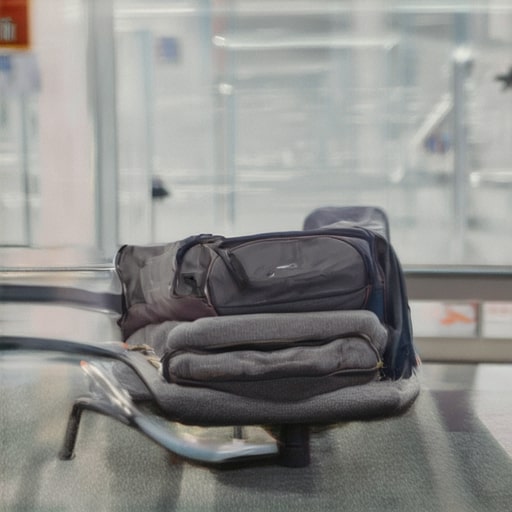}
  & \includegraphics[width=\linewidth]{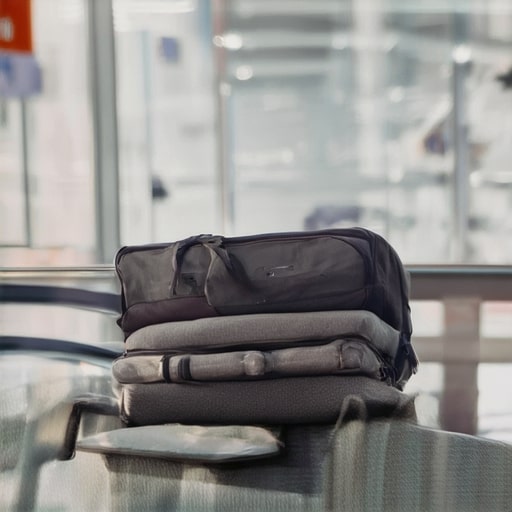}
  & \includegraphics[width=\linewidth]{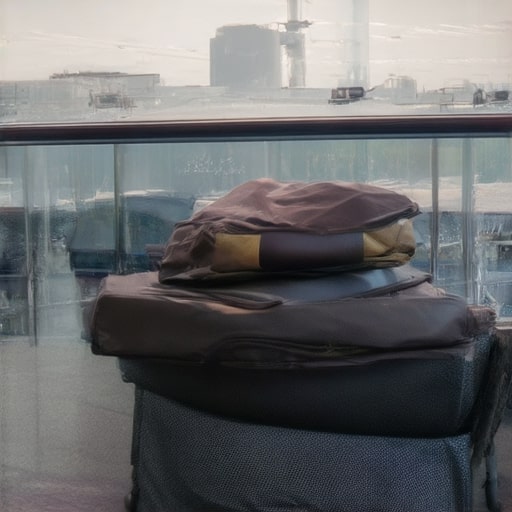}
  & \includegraphics[width=\linewidth]{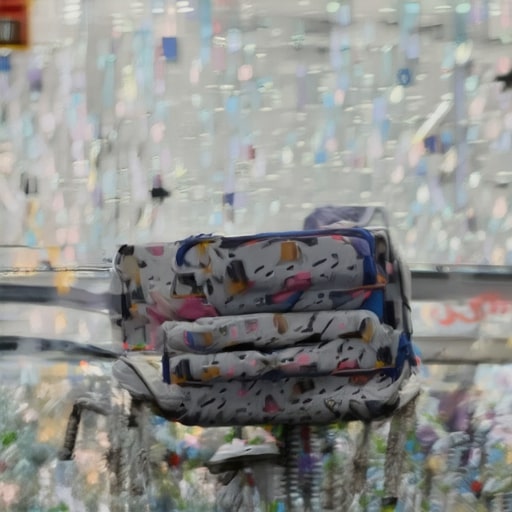}
  & \includegraphics[width=\linewidth]{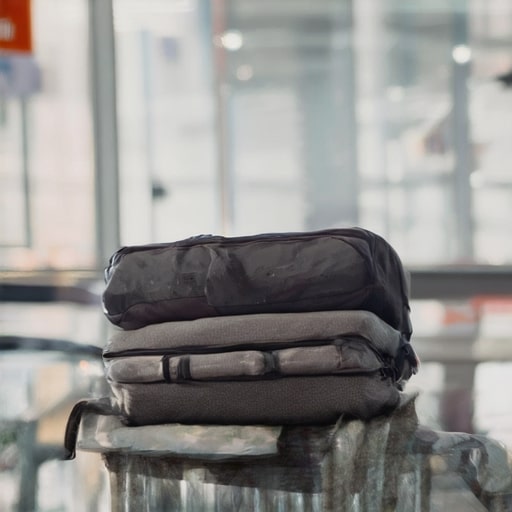}
  & \includegraphics[width=\linewidth]{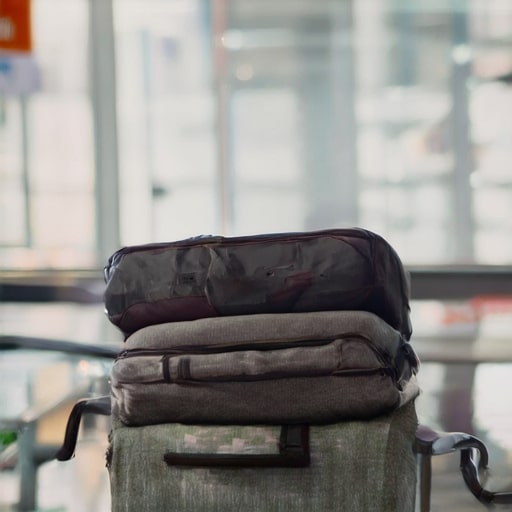}
  \\[0.01em]
  & \multicolumn{6}{c}{\scriptsize{\textit{``A few pieces of luggage sitting on top of a chair in an airport.''}}} \\[0.01em]
  \midrule
  \multirow{2}{*}{8}
  & \includegraphics[width=\linewidth]{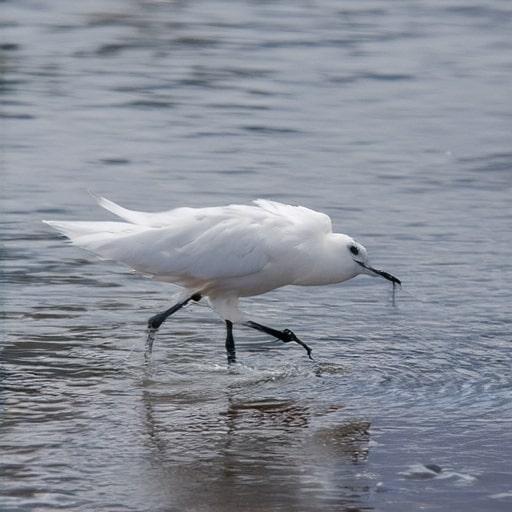}
  & \includegraphics[width=\linewidth]{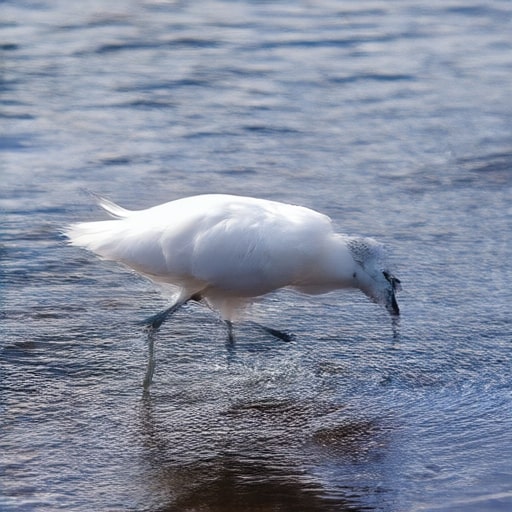}
  & \includegraphics[width=\linewidth]{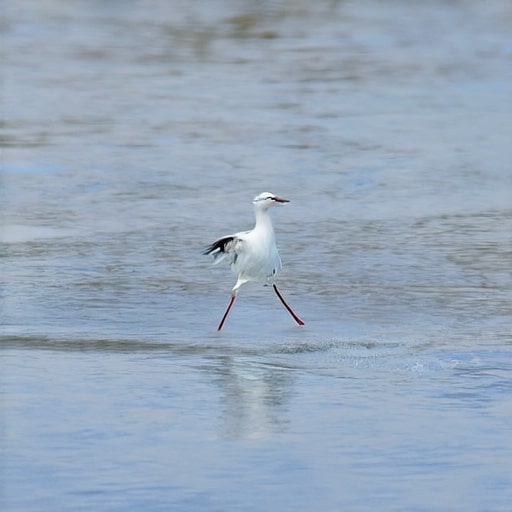}
  & \includegraphics[width=\linewidth]{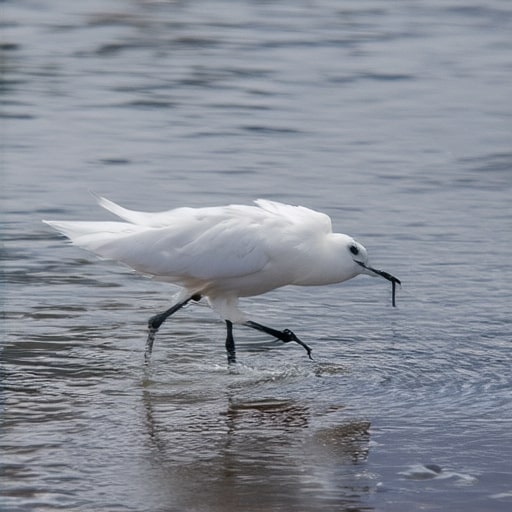}
  & \includegraphics[width=\linewidth]{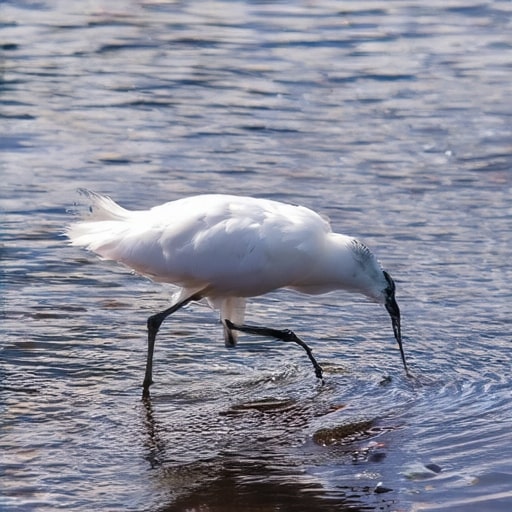}
  & \includegraphics[width=\linewidth]{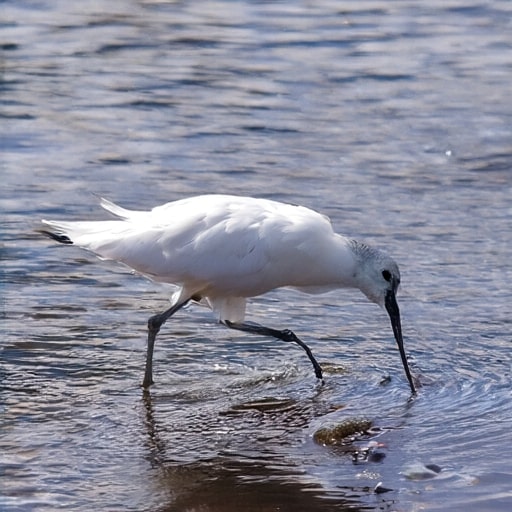}
  \\[0.01em]
  & \multicolumn{6}{c}{\scriptsize{\textit{``A white bird walking through a shallow area of water.''}}} \\[0.02em]
  &
  \includegraphics[width=\linewidth]{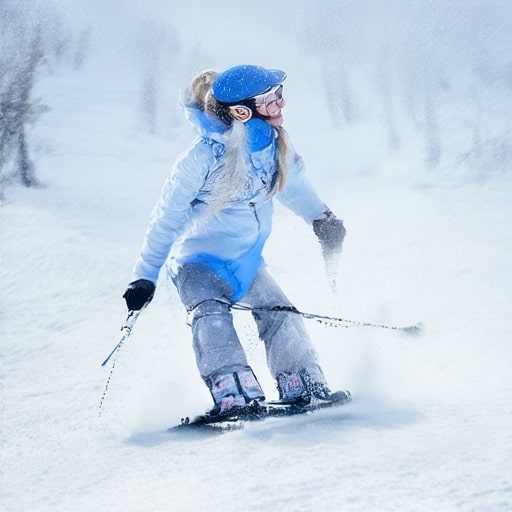}
  & \includegraphics[width=\linewidth]{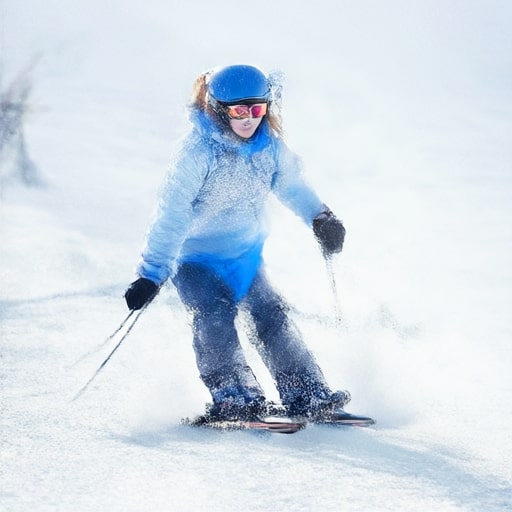}
  & \includegraphics[width=\linewidth]{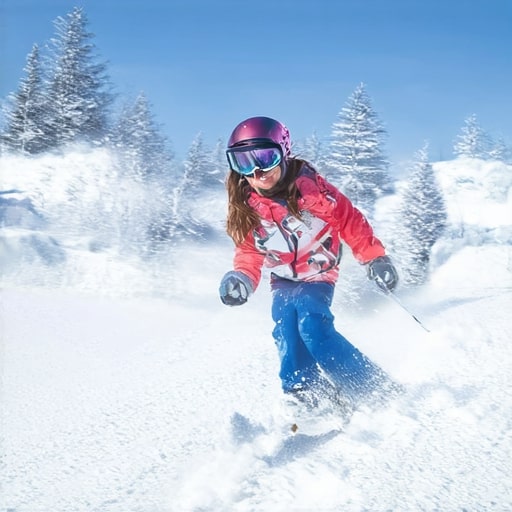}
  & \includegraphics[width=\linewidth]{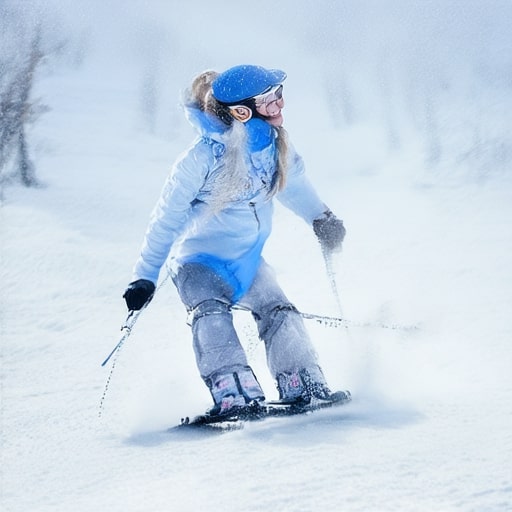}
  & \includegraphics[width=\linewidth]{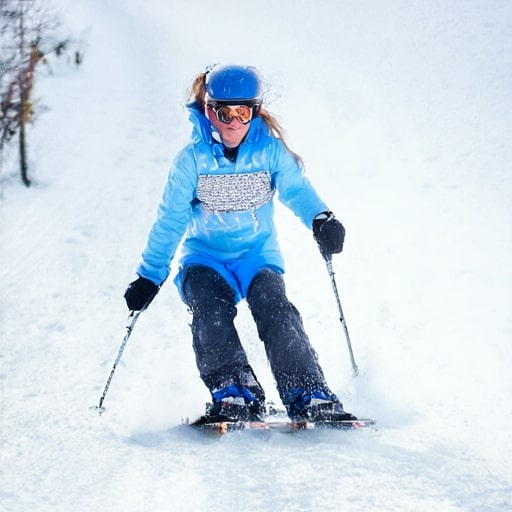}
  & \includegraphics[width=\linewidth]{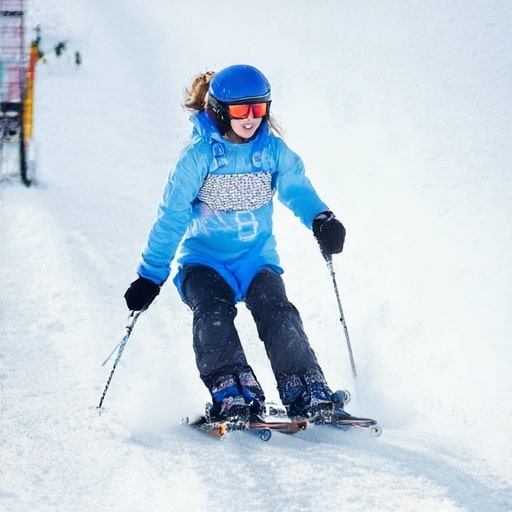}
  \\[0.01em]
  & \multicolumn{6}{c}{\scriptsize{\textit{``A young lady riding skis on a snow covered slope.''}}} \\[0.01em]
  \bottomrule
\end{tabularx}
}
\vspace{-0.6\baselineskip}
\caption{\textbf{Qualitative comparisons of samples generated using NFEs 6 and 8 with Stable Diffusion v3.5~\cite{Esser:2024SD3}.} We use RK1 and RK2 as the base solver.}
\label{fig:app_qualitative_sd}
\end{figure}

\section{More Qualitative Results}
\label{sec:more_quali}

We provide more qualitative results for accelerated sampling of diffusion models in Fig.~\ref{fig:app_qualitative_diffusion} and flow models in Fig.~\ref{fig:app_qualitative_flow} and Fig.~\ref{fig:app_qualitative_sd}. Across both model classes, \ours{} (BF) consistently yields clearer structures and more faithful details compared to baselines under low NFEs.

\section{Extension to Other Domains}
\label{sec:other_modalities}
BézierFlow is a generic framework applicable not only to image synthesis but also to various generative tasks within the stochastic interpolant framework. To demonstrate the versatility of our method and its robustness to different distance metrics beyond LPIPS, we conduct additional experiments on two distinct domains: 3D point cloud generation and layout generation.

\subsection{Unconditional 3D Point Cloud Generation}
\label{subsec:point_cloud}
3D Point cloud generation involves creating 3D representations of objects using discrete points, a task essential for applications in robotics, autonomous driving, and 3D modeling. We evaluate BézierFlow using the Point-Voxel Diffusion (PVD) model~\citep{Zhou:2021pvd}, trained on the \textit{airplane} category of the ShapeNet dataset~\citep{Chang:2016shapenet}.

\paragraph{Experiment Setup.} We adopt a simple mean squared error (MSE) loss for both training and validation. We generate 32 noise–data pairs for both the training and validation sets using DPM-Solver \citep{Lu:2022DPMSolver} with 64 NFEs, and train the model for 5 epochs. We compare our method against the same set of baselines reported in Tab.~\ref{tbl:main_fid_diffusion}.

\paragraph{Evaluation Metrics.} Following the evaluation protocol of PVD~\citep{Zhou:2021pvd}, we assess the quality of generated samples using three metrics based on the Chamfer Distance (CD): Minimum Matching Distance (CD-MMD), Coverage Score (CD-COV), and Jensen-Shannon Divergence (JSD).

\paragraph{Results.} Tab.~\ref{tab:point_cloud_results} presents the quantitative results. BézierFlow consistently achieves the best or second-best performance on CD-MMD, CD-COV across all NFEs, substantially improving over the base solvers and timestep-learning baselines~\citep{Xue:2024DMN, Chen:2024GITS, Tong:2025LD3}. Fig.~\ref{fig:pc_generation} provides qualitative comparisons of generated 3D point clouds, where BézierFlow better preserves both the global shape and coverage of the target distribution.

\begin{table}[t!]
    \centering
    \tiny
    \caption{\textbf{Quantitative comparison on unconditional 3D point cloud generation with Point Voxel Diffusion (PVD)~\citep{Zhou:2021pvd}.} Lower is better for CD-MMD (denoted as MMD) and JSD and higher is better for CD-COV (denoted as COV). CD-MMD is multiplied by $10^3$. Results for the base solvers are reported on each top rows. \textbf{Bold} indicates the best results, and \underline{underline} marks the second best. Gray cells indicate the
base ODE solvers.}
    \label{tab:point_cloud_results}
    \setlength{\tabcolsep}{2pt}
    \begin{tabularx}{\textwidth}
    {>{\raggedright\arraybackslash}p{0.11\textwidth}
    |*{12}{>{\centering\arraybackslash}X}}
        \toprule
        \multirow{2}{*}{Method} & \multicolumn{3}{c}{NFE=4} & \multicolumn{3}{c}{NFE=6} & \multicolumn{3}{c}{NFE=8} & \multicolumn{3}{c}{NFE=10} \\
        \cmidrule(lr){2-4} \cmidrule(lr){5-7} \cmidrule(lr){8-10} \cmidrule(lr){11-13}
        & MMD $\downarrow$ & COV $\uparrow$ & JSD $\downarrow$ & MMD $\downarrow$ & COV $\uparrow$ & JSD $\downarrow$ & MMD $\downarrow$ & COV $\uparrow$ & JSD $\downarrow$ & MMD $\downarrow$ & COV $\uparrow$ & JSD $\downarrow$ \\
        \midrule
        \cellcolor{gray!20}UniPC & 2.50 & 3.21 & 0.46 & 1.25 & 8.89 & 0.30 & 0.95 & 17.03 & 0.25 & 0.79 & 20.25 & \underline{0.22} \\
        + DMN & \underline{1.10} & 16.79 & \underline{0.27} & \underline{0.68} & \textbf{26.42} & \textbf{0.23} & 1.50 & 15.06 & 0.32 & \underline{0.67} & 19.51 & 0.27 \\
        + GITS & 5.32 & 9.52 & 0.56 & 9.14 & 0.74 & 0.63 & 1.20 & 20.25 & 0.31 & 0.90 & 16.30 & 0.23 \\
        + LD3 & 1.20 & \textbf{21.23} & \textbf{0.24} & 1.16 & 20.74 & \underline{0.24} & \underline{0.80} & \underline{21.48} & \underline{0.25} & 0.91 & \underline{21.23} & 0.23 \\
        + BézierFlow & \textbf{0.88} & \underline{18.77} & 0.29 & \textbf{0.59} & \underline{22.72} & \textbf{0.23} & \textbf{0.58} & \textbf{23.45} & \textbf{0.24} & \textbf{0.53} & \textbf{23.70} & \textbf{0.21} \\
        \midrule
        \cellcolor{gray!20}iPNDM & \underline{1.17} & 14.81 & \textbf{0.27} & 0.91 & 16.54 & \underline{0.23} & 0.78 & \underline{23.46} & \textbf{0.21} & 0.67 & \textbf{26.67} & \textbf{0.20} \\
        + DMN & 1.18 & \textbf{19.26} & \underline{0.29} & \underline{0.63} & \textbf{24.69} & \textbf{0.22} & 1.74 & 6.17 & 0.35 & \underline{0.65} & 20.49 & 0.22 \\
        + GITS & 3.22 & 7.65 & 0.44 & 3.59 & 3.21 & 0.48 & 3.99 & 1.73 & 0.49 & 2.73 & 3.95 & 0.41 \\
        + LD3 & 2.40 & 13.33 & 0.34 & 0.89 & 18.52 & 0.25 & \underline{0.77} & 19.01 & 0.25 & 0.70 & 22.72 & 0.22 \\
        + BézierFlow & \textbf{0.85} & \underline{18.52} & \underline{0.29} & \textbf{0.58} & \underline{22.73} & \underline{0.23} & \textbf{0.57} & \textbf{24.44} & \underline{0.23} & \textbf{0.56} & \underline{24.52} & \underline{0.21} \\
        \bottomrule
    \end{tabularx}
\end{table}
\begin{figure}[t!]
\centering
{\scriptsize
\setlength{\tabcolsep}{1pt}
\begin{tabularx}{\textwidth}{@{}
  >{\centering\arraybackslash}m{3em}
  *{4}{>{\centering\arraybackslash}X}
  >{\centering\arraybackslash}X
@{}}
  \toprule
  \multicolumn{1}{c|}{NFE} 
  & iPNDM & DMN & GITS & LD3 
  & BézierFlow \\
  \midrule
      \multicolumn{6}{@{}c@{}}{ShapeNet \emph{airplane} with PVD~\citep{Zhou:2021pvd}} \\
\midrule
  \multirow{2}{*}{6}
  & \includegraphics[width=\linewidth]{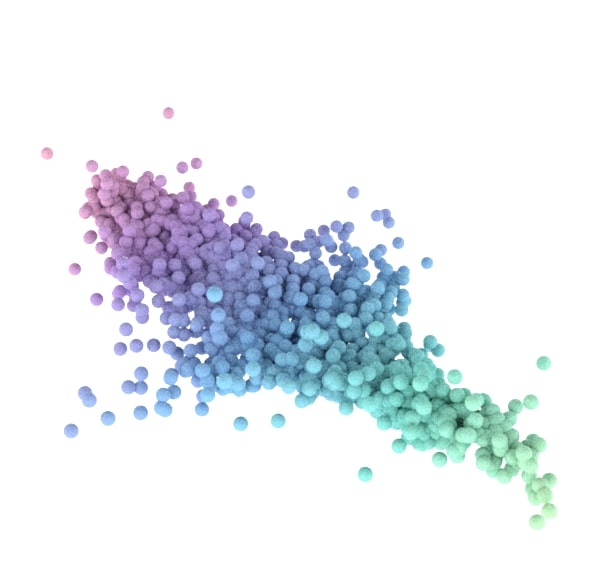}
  & \includegraphics[width=\linewidth]{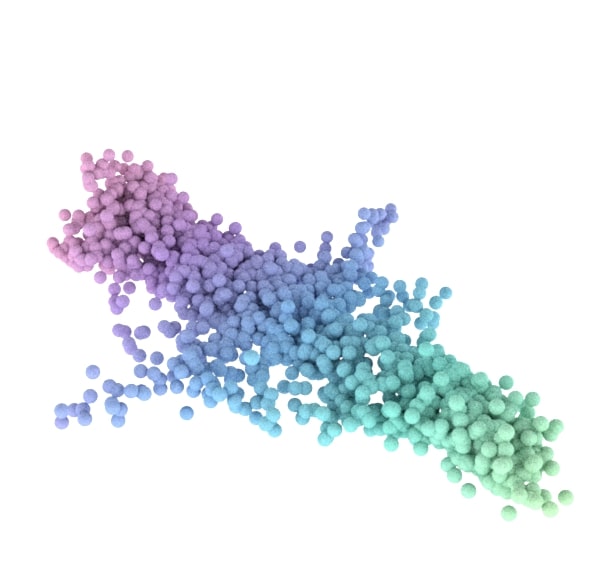}
  & \includegraphics[width=\linewidth]{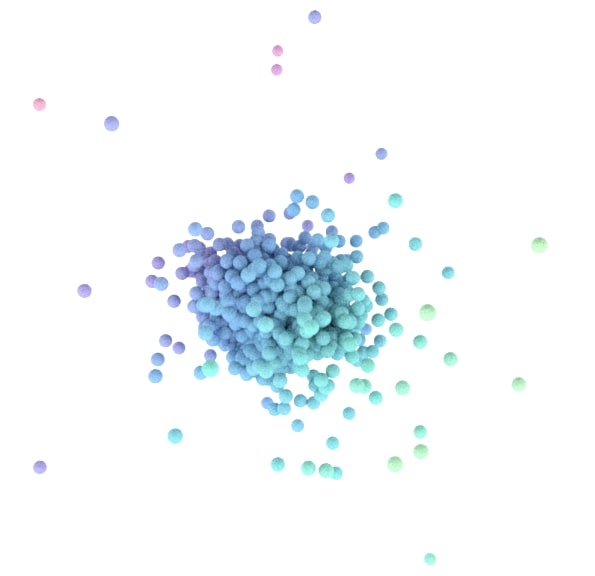}
  & \includegraphics[width=\linewidth]{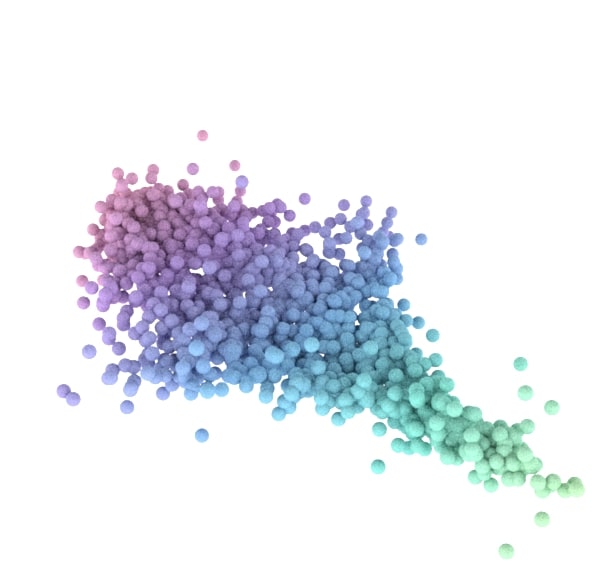}
  & \includegraphics[width=\linewidth]{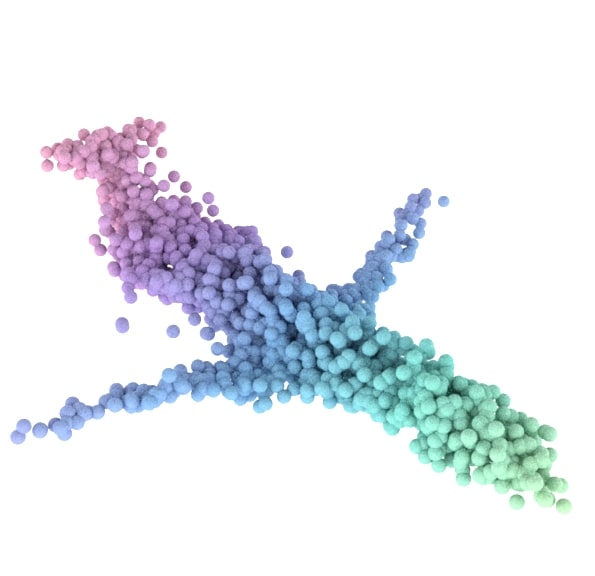}
  \\
  &
  \includegraphics[width=\linewidth]{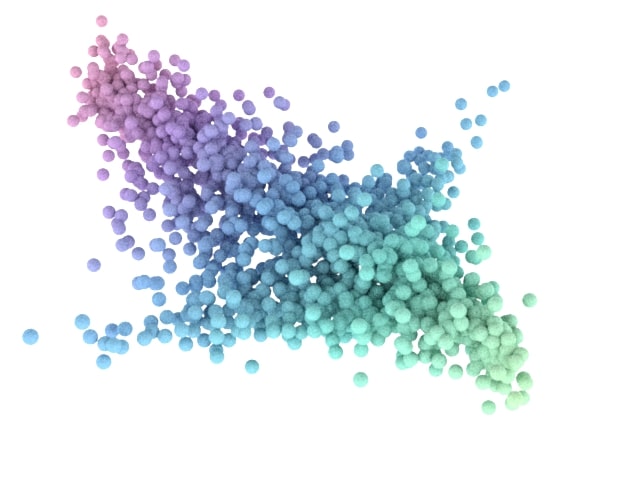}
  & \includegraphics[width=\linewidth]{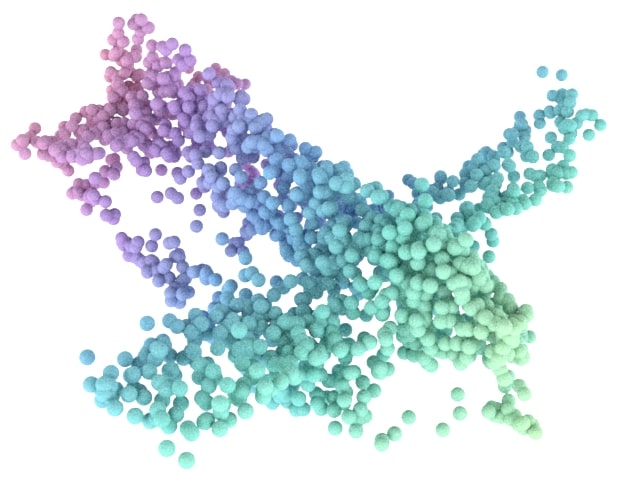}
  & \includegraphics[width=\linewidth]{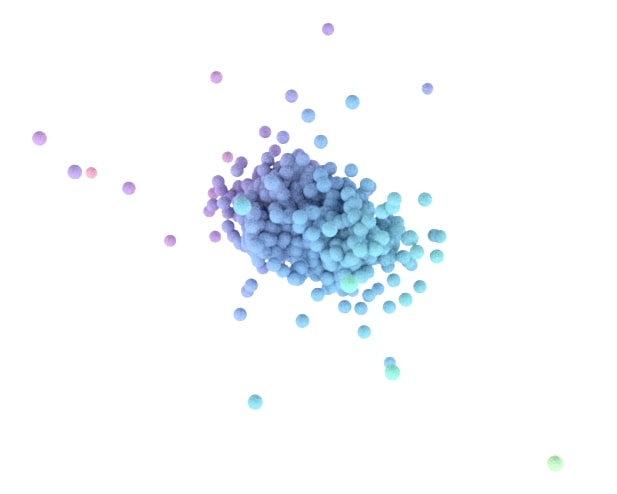}
  & \includegraphics[width=\linewidth]{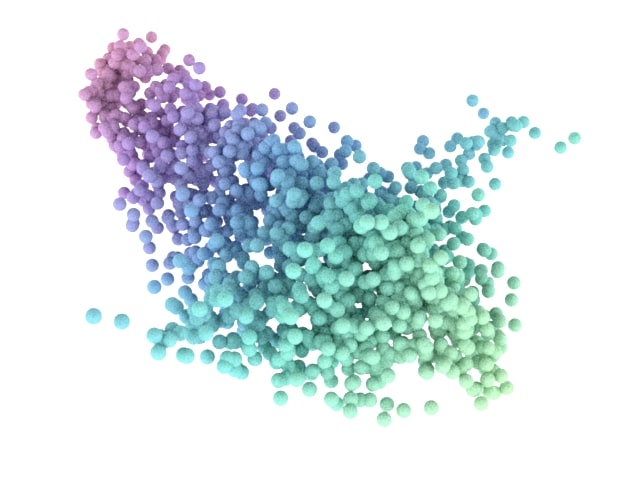}
  & \includegraphics[width=\linewidth]{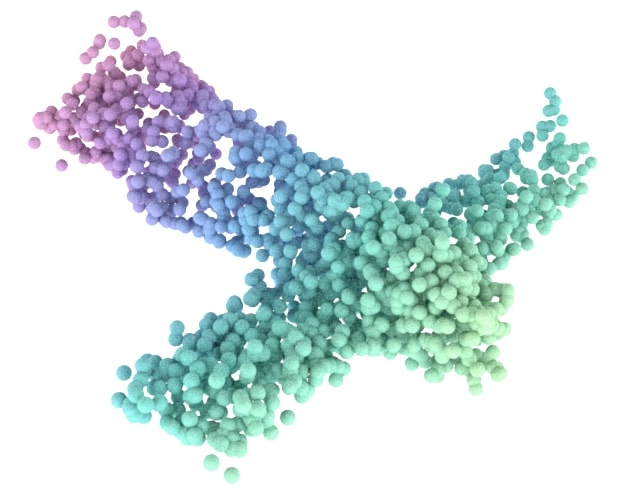}
  \\
  \midrule
  \multirow{2}{*}{8}
  & \includegraphics[width=\linewidth]{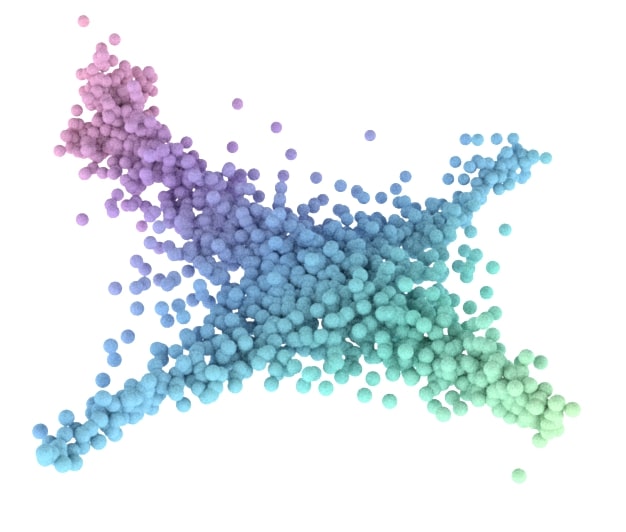}
  & \includegraphics[width=\linewidth]{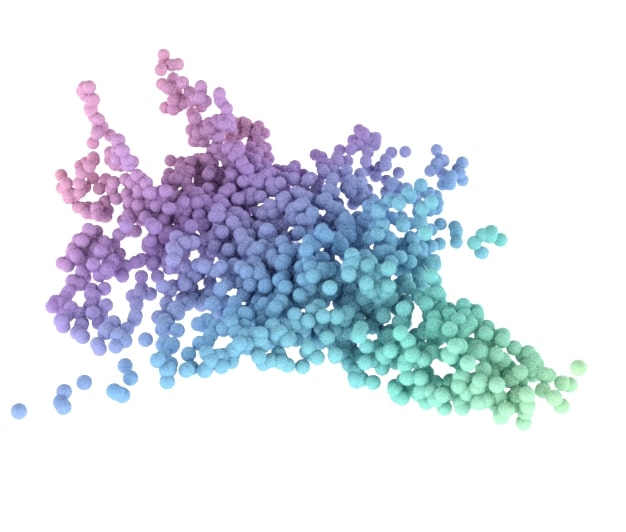}
  & \includegraphics[width=\linewidth]{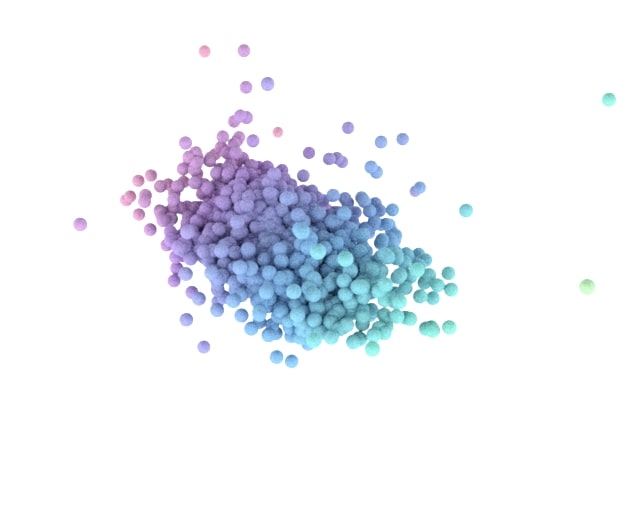}
  & \includegraphics[width=\linewidth]{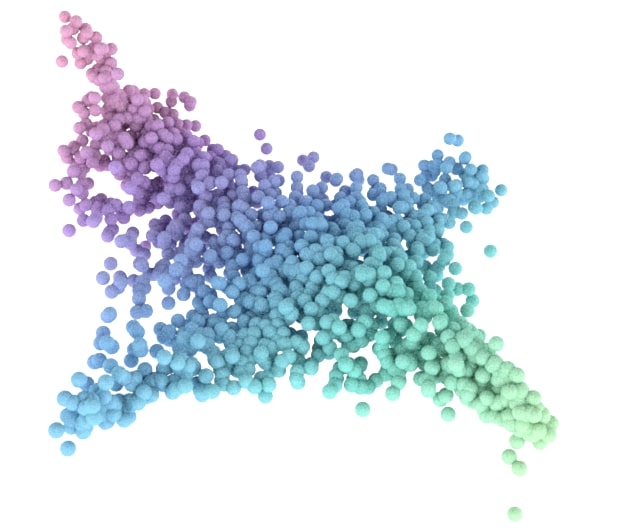}
  & \includegraphics[width=\linewidth]{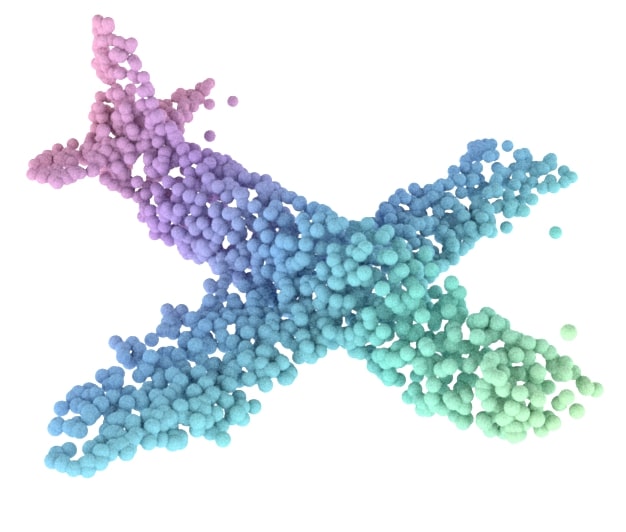}
  \\
  &
  \includegraphics[width=\linewidth]{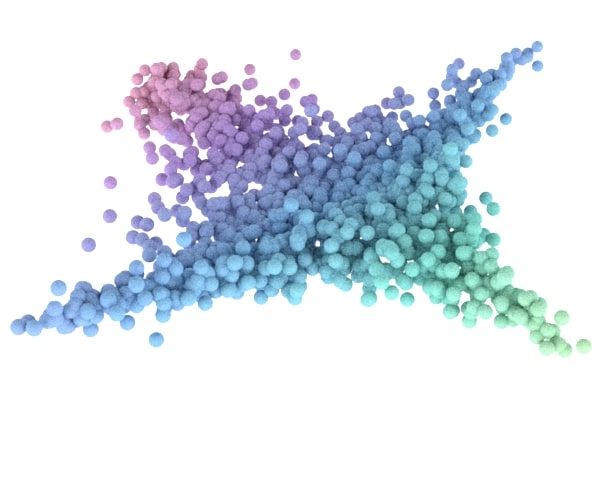}
  & \includegraphics[width=\linewidth]{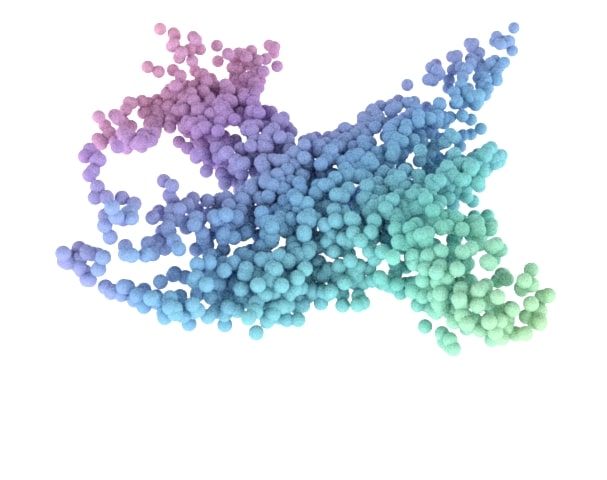}
  & \includegraphics[width=\linewidth]{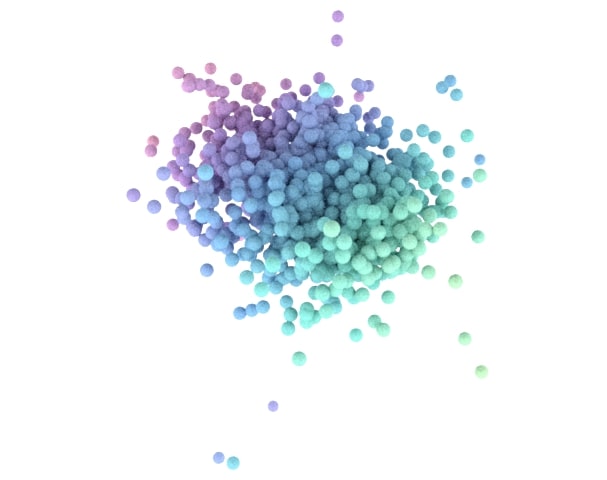}
  & \includegraphics[width=\linewidth]{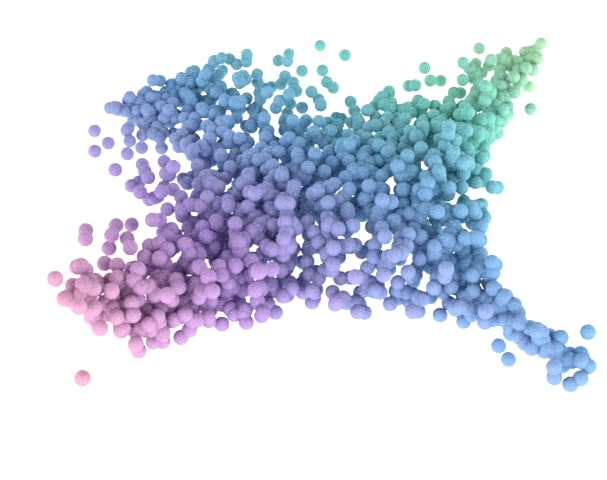}
  & \includegraphics[width=\linewidth]{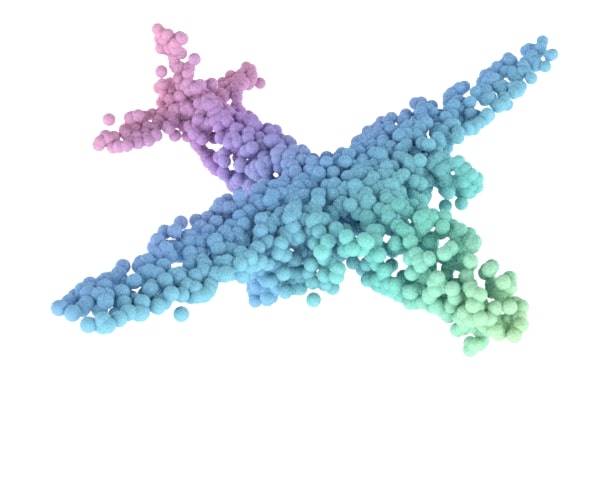}
  \\
  \bottomrule
\end{tabularx}
}
\caption{\textbf{Qualitative comparisons of 3D point cloud samples generated using NFEs 6 and 8 with PVD~\citep{Zhou:2021pvd}.} We use iPNDM as the base solver.}
\label{fig:pc_generation}
\end{figure}
\begin{table}[t!]
    \centering
    \tiny
    \caption{\textbf{Quantitative comparison on unconditional layout generation with LayoutFlow~\citep{Guerreiro:2024layout}.} Lower is better for FID, Alignment (denoted as Align.), Overlap. Results for the base solvers are reported on each top rows. \textbf{Bold} indicates the best results, and \underline{underline} marks the second best. Gray cells indicate the
base ODE solvers.}
    \label{tab:layout_results}
    \setlength{\tabcolsep}{2pt}
    \begin{tabularx}{\textwidth}
  {>{\raggedright\arraybackslash}p{0.11\textwidth}
   |*{12}{>{\centering\arraybackslash}X}}
        \toprule
        \multirow{2}{*}{Method} & \multicolumn{3}{c}{NFE=4} & \multicolumn{3}{c}{NFE=6} & \multicolumn{3}{c}{NFE=8} & \multicolumn{3}{c}{NFE=10} \\
        \cmidrule(lr){2-4} \cmidrule(lr){5-7} \cmidrule(lr){8-10} \cmidrule(lr){11-13}
        & FID $\downarrow$ & Align. $\downarrow$ & Overlap $\downarrow$ & FID $\downarrow$ & Align. $\downarrow$ & Overlap $\downarrow$ & FID $\downarrow$ & Align. $\downarrow$ & Overlap $\downarrow$ & FID $\downarrow$ & Align. $\downarrow$ & Overlap $\downarrow$ \\
        \midrule
        \cellcolor{gray!20}RK1 & 55.88 & 0.40 & 0.60 & 22.75 & 0.35 & 0.56 & 11.66 & 0.30 & 0.54 & 7.93 & 0.27 & 0.52 \\
        + DMN & 178.35 & 0.55 & 1.08 & 88.40 & 0.69 & 0.70 & 26.27 & 0.37 & \underline{0.46} & 10.96 & 0.33 & \textbf{0.46} \\
        + GITS & 41.08 & 0.37 & 0.57 & 12.84 & 0.35 & \textbf{0.47} & 7.32 & 0.29 & \textbf{0.45} & 5.90 & 0.27 & \textbf{0.46} \\
        + Bespoke & 213.61 & 0.92 & 1.01 & 201.20 & 0.88 & 0.67 & 168.49 & 0.63 & 0.59 & 171.11 & 0.63 & 0.56 \\
        + LD3 & \textbf{19.51} & \textbf{0.32} & \underline{0.54} & \underline{8.36} & \underline{0.28} & \underline{0.51} & \underline{5.03} & \textbf{0.23} & 0.48 & \underline{3.70} & \underline{0.23} & \underline{0.47} \\
        + BézierFlow & \underline{32.78} & \underline{0.35} & \textbf{0.53} & \textbf{7.10} & \textbf{0.26} & \textbf{0.47} & \textbf{3.86} & \underline{0.25} & 0.49 & \textbf{2.96} & \textbf{0.22} & 0.50 \\
        \midrule
        \cellcolor{gray!20}RK2 & 143.90 & 0.67 & 0.65 & 73.91 & 0.47 & 0.49 & 35.84 & 0.38 & 0.51 & 20.80 & 0.34 & 0.51 \\
        + DMN & 142.40 & 0.66 & \textbf{0.35} & 88.15 & 0.49 & \underline{0.46} & 63.57 & 0.37 & \textbf{0.42} & 56.23 & 0.35 & \textbf{0.43} \\
        + GITS & \textbf{102.11} & \textbf{0.49} & \underline{0.42} & 51.62 & 0.37 & 0.48 & 27.84 & \underline{0.32} & 0.50 & \underline{8.25} & \textbf{0.22} & \underline{0.47} \\
        + Bespoke & \underline{126.80} & \underline{0.61} & 0.47 & 187.62 & 0.86 & \textbf{0.37} & 32.99 & 0.38 & 0.54 & 21.54 & 0.36 & 0.50 \\
        + LD3 & 162.98 & 0.62 & 0.47 & \underline{42.82} & \underline{0.37} & 0.48 & \textbf{12.57} & \textbf{0.26} & \underline{0.48} & 8.39 & 0.27 & 0.48 \\
        + BézierFlow & 142.34 & 0.63 & 0.57 & \textbf{39.17} & \textbf{0.35} & 0.52 & \underline{25.51} & 0.37 & 0.50 & \textbf{7.18} & \underline{0.26} & 0.49 \\
        \bottomrule
    \end{tabularx}
\end{table}
\begin{figure}[t!]
\centering
{\scriptsize
\setlength{\tabcolsep}{1pt}
\begin{tabularx}{\textwidth}{@{}
  >{\centering\arraybackslash}m{3em}
  *{6}{>{\centering\arraybackslash}X}
  >{\centering\arraybackslash}X
@{}}
  \toprule
  \multicolumn{1}{c|}{NFE} 
  & RK1 & DMN & GITS & Bespoke & LD3 
  & BézierFlow 
  & \multicolumn{1}{|c}{Teacher} \\
  \midrule
      \multicolumn{8}{@{}c@{}}{RICO with LayoutFlow~\citep{Guerreiro:2024layout}} \\
\midrule
  \multirow{2}{*}{6}
  & \includegraphics[width=\linewidth]{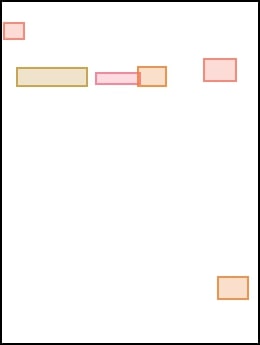}
  & \includegraphics[width=\linewidth]{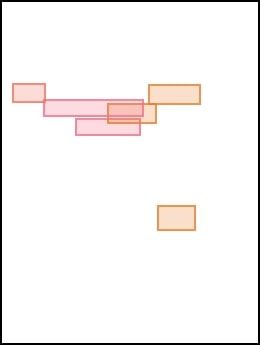}
  & \includegraphics[width=\linewidth]{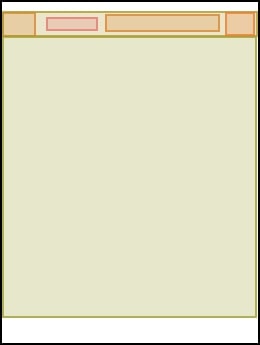}
  & \includegraphics[width=\linewidth]{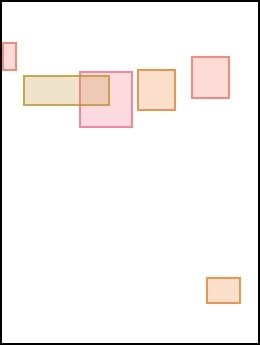}
  & \includegraphics[width=\linewidth]{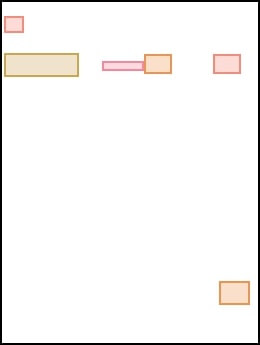}
  & \includegraphics[width=\linewidth]{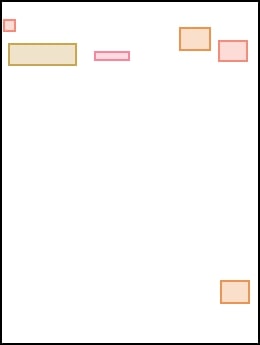}
  & \includegraphics[width=\linewidth]{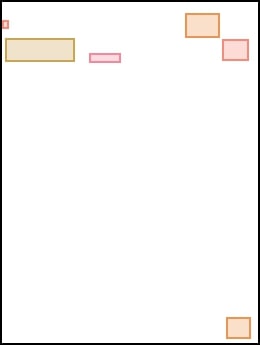}
  \\
  &
  \includegraphics[width=\linewidth]{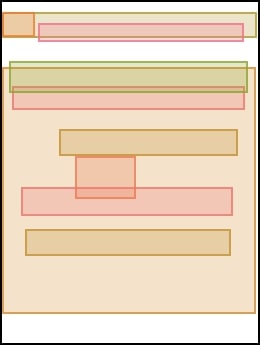}
  & \includegraphics[width=\linewidth]{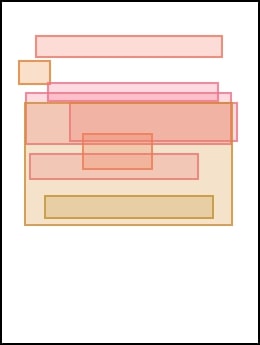}
  & \includegraphics[width=\linewidth]{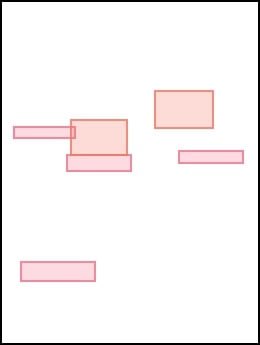}
  & \includegraphics[width=\linewidth]{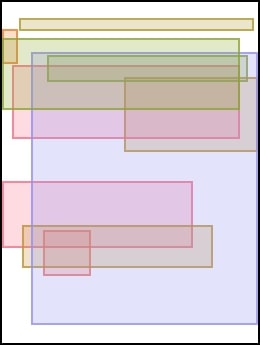}
  & \includegraphics[width=\linewidth]{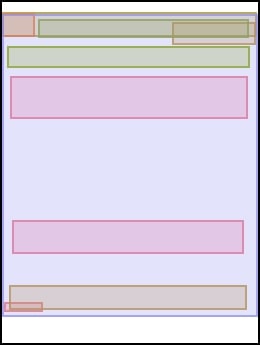}
  & \includegraphics[width=\linewidth]{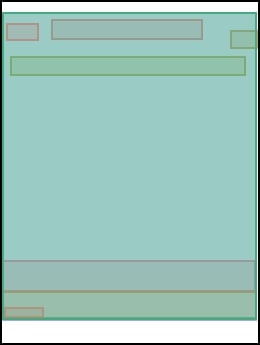}
  & \includegraphics[width=\linewidth]{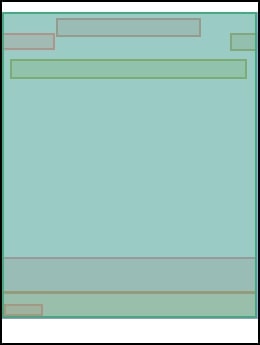}
  \\
  \midrule
  \multirow{2}{*}{8}
  & \includegraphics[width=\linewidth]{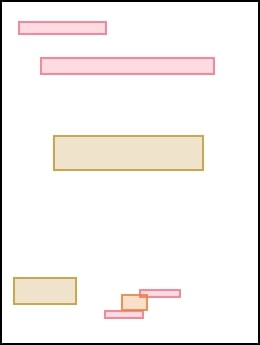}
  & \includegraphics[width=\linewidth]{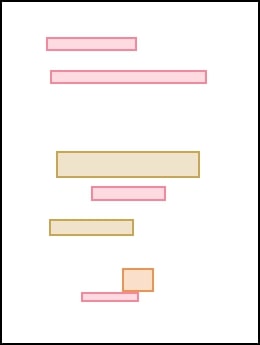}
  & \includegraphics[width=\linewidth]{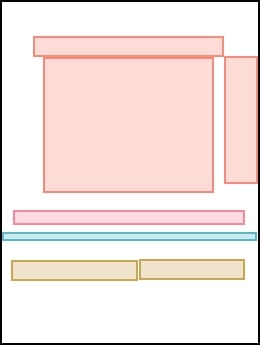}
  & \includegraphics[width=\linewidth]{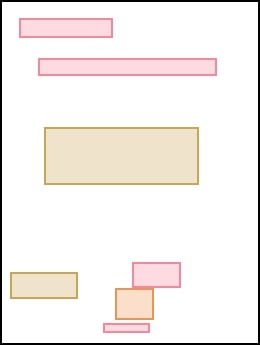}
  & \includegraphics[width=\linewidth]{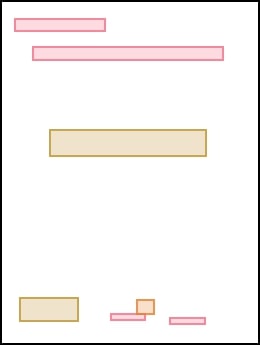}
  & \includegraphics[width=\linewidth]{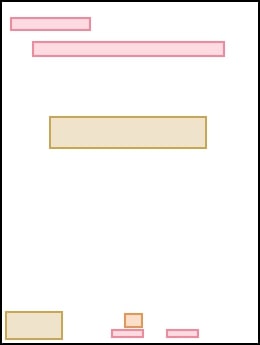}
  & \includegraphics[width=\linewidth]{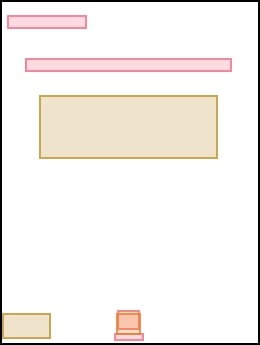}
  \\
  &
  \includegraphics[width=\linewidth]{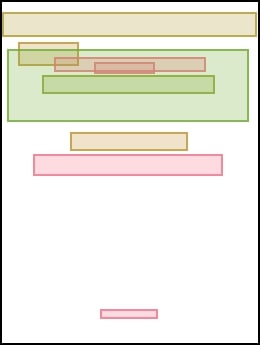}
  & \includegraphics[width=\linewidth]{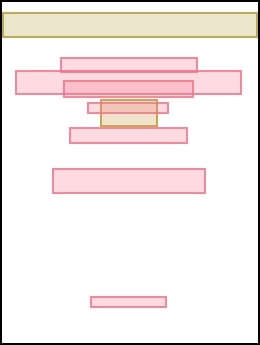}
  & \includegraphics[width=\linewidth]{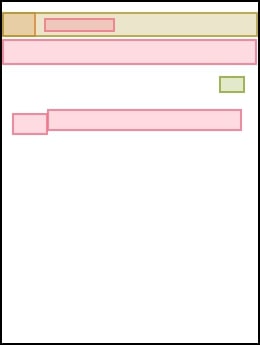}
  & \includegraphics[width=\linewidth]{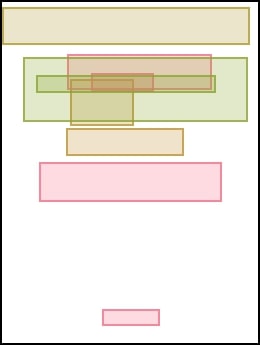}
  & \includegraphics[width=\linewidth]{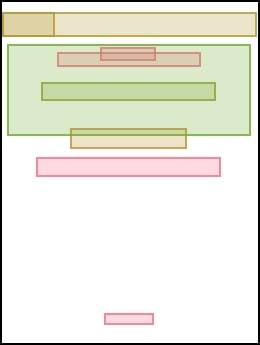}
  & \includegraphics[width=\linewidth]{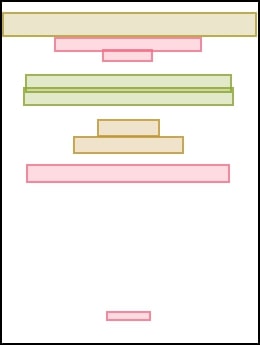}
  & \includegraphics[width=\linewidth]{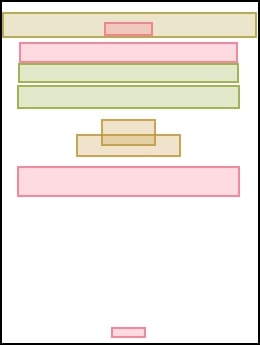}
  \\
  \bottomrule
\end{tabularx}
}
\caption{\textbf{Qualitative comparisons of layout samples generated using NFEs 6 and 8 with LayoutFlow~\citep{Guerreiro:2024layout}.} We use RK1 as the base solver. The rightmost column shows teacher samples from RK45 solver.}
\label{fig:layout_generation}
\end{figure}

\subsection{Unconditional Layout Generation}
\label{subsec:layout_gen}
Layout generation aims to synthesize structural arrangements of elements (e.g., UI components, document blocks), which is a critical step in graphic design automation. We evaluate our method on unconditional layout generation using LayoutFlow~\citep{Guerreiro:2024layout}, pretrained on the RICO dataset~\citep{Deka:2017rico}.

\paragraph{Experiment Setup.} We adopt negative mean Intersection over Union (mIoU) between the teacher and student layouts as the objective for both training and validation. We generate 50 noise–data pairs for both the training and validation sets using an RK45~\citep{BUTCHER:1996RK} solver, and train the model for 5 epochs. We compare our method against the same set of baselines reported in Tab.~\ref{tbl:main_fid_flow}.

\paragraph{Evaluation Metrics.} Following the evaluation protocol of LayoutFlow~\citep{Guerreiro:2024layout}, we assess generation quality using Fréchet Inception Distance (FID) adapted for layouts, alongside Alignment and Overlap scores. For FID calculation, we employ the feature extractor from LayoutDiffusion~\citep{Zheng:2023layoutdiffusion}.

\paragraph{Results.} As shown in Table \ref{tab:layout_results}, BézierFlow shows better performance than other baselines in terms of FID and Alignment. Interestingly, it consistently outperforms the base solvers (RK1 and RK2) on all metrics at the same NFE settings. We provide a qualitative comparison of the generated layouts in Fig.~\ref{fig:layout_generation}. BézierFlow produces layouts that most closely follow the teacher trajectory, preserving the spatial arrangement and aspect ratios of objects.

\begin{table}[!t]
\centering
\scriptsize
\caption{\textbf{FID comparison of VDM, Multi-marginal SI (denoted as MMSI), and BézierFlow on CIFAR-10.} Results for the base solvers are reported on each top rows. \textbf{Bold} indicates the best results, and \underline{underline} marks the second best. Gray cells indicate the base ODE solvers.}
\setlength{\tabcolsep}{5pt}
\begin{tabularx}{\textwidth}
  {>{\raggedright\arraybackslash}p{0.1\textwidth}
   |*{4}{>{\centering\arraybackslash}X}
   |>{\raggedright\arraybackslash}p{0.1\textwidth}
   |*{4}{>{\centering\arraybackslash}X}}
\toprule
 Method & NFE=4 & NFE=6 & NFE=8 & NFE=10 & Method & NFE=4 & NFE=6 & NFE=8 & NFE=10 \\
 \midrule
\multicolumn{10}{c}{CIFAR-10 $32\times 32$ with ReFlow~\citep{Liu:2023RF} (Teacher FID: 2.70)} \\
\midrule
 \cellcolor{gray!20}RK1 & 52.78 & 26.30 & 17.40 & 13.30 & \cellcolor{gray!20}RK2 & 25.36 & 12.12 & 9.17 & 7.89 \\
 + VDM & 54.72 & 22.06 & 19.10 & 19.00 & + VDM & 36.24 & 25.74 & 16.37 & 12.39 \\
 + MMSI & \underline{22.89}&\underline{12.06}&\underline{7.59}&\underline{5.86}
 & + MMSI & \underline{20.82}&\underline{9.03}&\underline{7.57}&\underline{7.79} \\
 + \ours{} & \textbf{20.64} & \textbf{9.67} & \textbf{7.30} & \textbf{5.51} & + \ours{} & \textbf{13.18}&\textbf{6.00}&\textbf{4.31}&\textbf{3.74} \\

 \bottomrule
\end{tabularx}
\label{tbl:vdm_mmsi}
\end{table}

\section{Comparison with Other Scheduler Parameterizations}
In this section, we discuss prior work~\citep{Kingma:2021VDM, Albergo:2024multimarginalsi} that also learns optimal SI schedulers and compare them against BézierFlow. We first clarify how these methods differ in their scheduler parameterizations, and then experimentally show that BézierFlow achieves superior performance due to its compact parameterization that explicitly satisfies the core SI scheduler requirements: boundary conditions, monotonicity, and differentiability.

\paragraph{Varitional Diffusion Models~\citep{Kingma:2021VDM}.} Variational Diffusion Models (VDMs) model the signal-to-noise ratio function $\text{SNR}(t)$ with a monotone neural network to satisfy monotonicity, aiming primarily to improve generative performance rather than sampling acceleration.
However, this neural network contains more than 1024 parameters, and thus is not parameter-efficient.
In contrast, BézierFlow uses a much more compact parameterization with only $n=32$ control points in our experiments by leveraging 1-D Bézier functions, reducing the number of scheduler parameters by roughly an order of magnitude.

\paragraph{Multi-Marginal Stochastic Interpolant~\citep{Albergo:2024multimarginalsi}.} Multi-Marginal Stochastic Interpolant (Multi-Marginal SI) also learns a stochastic interpolant scheduler to improve generative performance.
In the 2-marginal case, the (unnormalized) scheduler is parameterized as




\begin{equation}
    \tilde\alpha(s) = 1 - s + \left(\sum_{k=1}^K a_k \sin\Big(\frac{\pi}{2}t\Big)\right)^2,\qquad \tilde\sigma(s) = s + \left(\sum_{k=1}^K b_k \sin\Big(\frac{\pi}{2}t\Big)\right)^2,
\end{equation}

with learnable coefficients $a_k$ and $b_k$, which are then normalized via

\begin{equation}
    \bar\alpha(s) = \frac{\tilde\alpha(s)}{\tilde\alpha(s) + \tilde\sigma(s)},\qquad \bar\sigma(s)=\frac{\tilde\sigma(s)}{\tilde\alpha(s) + \tilde\sigma(s)}.
\end{equation}

While this parameterization enforces the boundary conditions $\bar\alpha(0)=0$, $\bar\sigma(1)=0$, $\bar\alpha(1)=1$, and $\bar\sigma(0)=1$, the induced SNR schedule $\bar\rho(s) = \bar\alpha(s) / \bar\sigma(s)$ is not guaranteed to be monotonically increasing. In contrast, our Bézier-based parameterization explicitly satisfies the three core requirements of an SI scheduler: (1) boundary conditions, (2) monotonicity, and (3) differentiability. This advantage is reflected in the quantitative comparison reported below.


\paragraph{Results.} We report few-step generation FIDs on CIFAR-10 with ReFlow~\citep{Liu:2023RF} for VDM~\citep{Kingma:2021VDM}, Multi-Marginal SI~\citep{Albergo:2024multimarginalsi}, and \ours{}. For VDM, we parameterize the SNR neural network as a 3-layer MLP with hidden size $1024$, following the original configuration in the paper, and set the trigonometric order of Multi-Marginal SI to $K=32$ so that its number of scheduler degrees of freedom matches our $n=32$ Bézier parameterization. As shown in Tab.~\ref{tbl:vdm_mmsi}, BézierFlow consistently achieves the best FID across all NFEs and base solvers, outperforming VDM and Multi-Marginal SI under the same training setup. This highlights that our Bézier-based parameterization, which satisfies the key requirements of an SI scheduler, provides a more effective and stable way to learn SI schedulers for few-step generation than existing neural or trigonometric alternatives.

\begin{table}[!t]
\centering
\scriptsize
\caption{\textbf{Cross-dataset transfer of Bézier stochastic interpolant schedulers.} Results for the base solvers are reported on each top rows. \textbf{Bold} indicates the best results, and \underline{underline} marks the second best. Gray cells indicate the base ODE solvers.}
\setlength{\tabcolsep}{5pt}
\begin{tabularx}{\textwidth}
  {>{\raggedright\arraybackslash}p{0.1\textwidth}
   |*{4}{>{\centering\arraybackslash}X}
   |>{\raggedright\arraybackslash}p{0.1\textwidth}
   |*{4}{>{\centering\arraybackslash}X}}
\toprule
 Method & NFE=4 & NFE=6 & NFE=8 & NFE=10 & Method & NFE=4 & NFE=6 & NFE=8 & NFE=10 \\
 \midrule
\multicolumn{10}{c}{CIFAR-10 $32\times 32$ with EDM $\rightarrow$ FFHQ $64\times 64$ with EDM} \\
\midrule
 \cellcolor{gray!20}UniPC & 47.62 & 14.96 & 7.76 & 8.93 & \cellcolor{gray!20}iPNDM & 28.75 & 11.15 & 6.68 & 4.80 \\
 + \ours{} & \textbf{17.05} & \textbf{7.43} & \textbf{3.82} & \textbf{3.13} & + \ours{} & \textbf{15.39} & \textbf{7.84} & \underline{5.56} & \textbf{3.75} \\
 + Transferred & \underline{22.35} & \underline{9.05} & \underline{4.93} & \underline{4.50} & + Transferred & \underline{23.41} & \underline{9.87} & \textbf{5.47} & \underline{3.79} \\
 \midrule
\multicolumn{10}{c}{CIFAR-10 $32\times 32$ with EDM $\rightarrow$ AFHQv2 $64\times 64$ with EDM} \\
\midrule
 \cellcolor{gray!20}UniPC & 23.59 & 10.15 & 7.76 & 6.38 & \cellcolor{gray!20}iPNDM & 15.14 & 6.12 & 3.80 &  3.01 \\
 + \ours{} & \underline{12.27} & \textbf{4.46} & \textbf{2.75} & \textbf{2.67} & + \ours{} & \underline{14.44} & \underline{4.69} & \textbf{2.63} & \textbf{2.16} \\
 + Transferred & \textbf{11.58} & \underline{4.47} & \underline{2.98} & \underline{2.71} & + Transferred & \textbf{9.52} & \textbf{3.95} & \underline{2.96} & \underline{2.30} \\

 \bottomrule
\end{tabularx}
\label{tbl:scheduler_transfer}
\end{table}

\section{Cross-Dataset Transfer of Bézier Scheduler}
We investigate whether a BézierFlow trained on one dataset can be reused on other datasets without retraining. Specifically, we train BézierFlow on CIFAR-10 with pretrained diffusion models~\citep{Karras:2022EDM} and then evaluate on two different datasets, FFHQ and AFHQv2. In Tab.~\ref{tbl:scheduler_transfer}, we report FIDs for the base ODE solvers, the dataset-specific scheduler (denoted as ``\ours{}'') and the CIFAR-10–trained scheduler reused on the target datasets (denoted as ``Transferred'').

As shown, although the scheduler is trained only on CIFAR-10, its performance on out-of-domain datasets still outperforms the base solvers and remains competitive with a scheduler trained directly on the target dataset. This demonstrates that BézierFlow provides a generally effective acceleration scheme even under domain shift.

\begin{table}[t]
\centering
\scriptsize
\caption{\textbf{Training time and peak GPU memory usage of BézierFlow for diffusion and flow models at NFEs 4 and 10 on a single A6000 GPU.}}
\setlength{\tabcolsep}{6pt}
\begin{tabularx}{\textwidth}
  {>{\raggedright\arraybackslash}p{0.445\textwidth}
   |*{4}{>{\centering\arraybackslash}X}}
\toprule
\multirow{2}{*}{Dataset / Model} & \multicolumn{2}{c}{NFE=4} & \multicolumn{2}{c}{NFE=10} \\
 & Training Time & VRAM & Training Time & VRAM \\
\midrule
\multicolumn{5}{c}{(1) Diffusion Models} \\
\midrule
CIFAR-10 32$\times$32 with EDM~\citep{Karras:2022EDM}        & 8 minutes  & 4 GB  & 15 minutes  & 8 GB  \\
FFHQ 64$\times$64 with EDM~\citep{Karras:2022EDM}           & 11 minutes & 3 GB  & 18 minutes  & 7 GB  \\
AFHQv2 64$\times$64 with EDM~\citep{Karras:2022EDM}         & 11 minutes & 3 GB  & 18 minutes  & 7 GB  \\
\midrule
\multicolumn{5}{c}{(2) Flow Models} \\
\midrule
CIFAR-10 32$\times$32 with ReFlow~\citep{Liu:2023RF}       & 8 minutes  & 3 GB  & 15 minutes  & 7 GB  \\
ImageNet 256$\times$256 with FlowDCN~\citep{Wang:2024FlowDCN}     & 25 minutes & 5 GB  & 45 minutes  & 8 GB  \\
MS-COCO 512$\times$512 with Stable Diffusion~\citep{Esser:2024SD3} & 60 minutes  & 21 GB & 100 minutes & 22 GB \\
\bottomrule
\end{tabularx}
\label{tbl:comp_costs}
\end{table}

\section{Computational Costs}
\label{sec:comp_costs}
We report wall-clock training time and peak GPU memory for BézierFlow across all datasets and both diffusion and flow models, evaluated at NFE=4 and NFE=10 on a single A6000 GPU. As shown in Tab.~\ref{tbl:comp_costs}, BézierFlow trains in at most 1 hour even for a 2.5B large-scale pretrained model~\citep{Esser:2024SD3} at $512\times 512$ resolution, while requiring only 22~GB of GPU memory. This makes the method practical even on a single commodity GPU commonly available in research labs. Despite this low training and memory cost, BézierFlow improves FID by large margins over the base model, e.g., \textbf{from 50.30 to 9.55 ($\approx$81\% relative improvement)} at NFE=4 on CIFAR-10 and \textbf{from 8.93 to 3.13 ($\approx$64\% relative improvement)} at NFE=10 on FFHQ, as shown in Tab.~\ref{tbl:main_fid_diffusion}.

The increase in training time from NFE=4 to NFE=10 is always less than $2\times$, and peak GPU memory grows only mildly with NFE. This is because we apply gradient checkpointing over the student trajectory, so activation memory scales only weakly with the number of steps, even for Stable Diffusion v3.5. These results indicate that BézierFlow scales to high-resolution, large models with modest computational overhead, making it practical as a plug-and-play scheduler even for large-scale generative models.


\end{document}